\documentclass{article}
\usepackage{fullpage}
\usepackage{natbib}

\usepackage{algorithm}
\usepackage{algorithmic}

\usepackage{amsfonts}
\usepackage{amsmath}
\usepackage{amsthm}
\usepackage{amssymb}
\usepackage{dsfont}

\usepackage{xcolor}
\usepackage{color}
\usepackage{graphicx}
\usepackage{verbatim}
\usepackage{xspace} % for algorithm names

\newcommand{\algopt}{\textsc{Choco-SGD}\xspace} 
\newcommand{\algcons}{\textsc{Choco-Gossip}\xspace} 
\newcommand{\eqncons}{\textsc{Choco-G}\xspace}

\providecommand{\lin}[1]{\ensuremath{\left\langle #1 \right\rangle}}
\providecommand{\abs}[1]{\left\lvert#1\right\rvert}
\providecommand{\norm}[1]{\left\lVert#1\right\rVert}

  \providecommand{\R}{\mathbb{R}} % Reals
  \providecommand{\N}{\mathbb{N}} % Naturals
  
  \providecommand{\E}[1]{{\mathbb E}\left.#1\right. }        %expectation
  \providecommand{\EE}[2]{{\mathbb E}_{#1}\left.#2\right. }  %expectation  
  \providecommand{\0}{\mathbf{0}}
  \providecommand{\1}{\mathbf{1}}
  \renewcommand{\aa}{\mathbf{a}}
  \providecommand{\bb}{\mathbf{b}}

  \renewcommand{\gg}{\mathbf{g}}

  \providecommand{\qq}{\mathbf{q}}

  \providecommand{\xx}{\mathbf{x}}
  \providecommand{\yy}{\mathbf{y}}

  \providecommand{\cD}{\mathcal{D}}

  \providecommand{\cO}{\mathcal{O}}

  \usepackage[textwidth=5cm]{todonotes}

\newtheorem{lemma}{Lemma}

\newtheorem{definition}{Definition}
\newtheorem{remark}[lemma]{Remark}
\newtheorem{assumption}{Assumption}
\newtheorem{theorem}[lemma]{Theorem}

\usepackage[colorlinks=true,linkcolor=blue]{hyperref} 
\usepackage[capitalize,noabbrev]{cleveref}

\usepackage{url}

\usepackage{subfigure}
\usepackage{tikz}
\usetikzlibrary{arrows.meta}
\renewcommand{\cite}[1]{\citep{#1}}
\definecolor{mydarkblue}{rgb}{0,0.08,0.45}
\hypersetup{ %
    colorlinks=true,
    linkcolor=mydarkblue,
    citecolor=mydarkblue,
    filecolor=mydarkblue,
    urlcolor=mydarkblue,
    pdfview=FitH}
\date{}%no date, use arxiv timestamp
\title{\bf Decentralized Stochastic Optimization and Gossip Algorithms with Compressed Communication}

\author{
		Anastasia Koloskova \\
	EPFL
	\and 
	Sebastian U. Stich \\
	EPFL
	\and 
	 Martin Jaggi\\
	 EPFL
	 \and 
	\texttt{\normalsize\{anastasia.koloskova, sebastian.stich, martin.jaggi\}@epfl.ch}}
\begin{document}
	\maketitle

% !TEX root = main.tex

%

\begin{abstract}
We consider decentralized stochastic optimization with the objective function (e.g. data samples for machine learning task) 
being distributed over $n$ machines that can only communicate to their neighbors on a fixed communication graph. 
To reduce the communication bottleneck,  the nodes compress (e.g.\ quantize or sparsify) their model updates. We cover both unbiased and biased compression operators with quality denoted by $\omega \leq 1$ ($\omega=1$ meaning no compression).
\\
We (i) propose a novel gossip-based stochastic gradient descent algorithm, \algopt, 
that converges at rate $\cO\left(1/(nT) + 1/(T \delta^2 \omega)^2\right)$ for strongly convex objectives, where $T$ denotes the number of iterations %
and $\delta$ the eigengap of the connectivity matrix. 
Despite compression quality and network connectivity affecting the higher order terms, the first term in the rate, $\cO(1/(nT))$, is the same as for the centralized baseline with exact communication.
We (ii) present a novel gossip algorithm, \algcons, for the average consensus problem that converges in time $\cO(1/(\delta^2\omega) \log (1/\epsilon))$ 
for accuracy $\epsilon > 0$. This is (up to our knowledge) the first gossip algorithm that supports arbitrary compressed messages for $\omega > 0$ and still exhibits linear convergence. We (iii) show in experiments that both of our algorithms do outperform the respective state-of-the-art baselines and \algopt can reduce communication by at least two orders of magnitudes.
\end{abstract}

\section{Introduction}

\emph{Decentralized} machine learning methods are becoming core aspects of many important applications, both in view of scalability to larger datasets and systems, but also from the perspective of data locality, ownership and privacy.
In this work we address the general \emph{data-parallel} setting where the data is distributed across different compute devices, and consider decentralized optimization methods that do not rely on a central coordinator (e.g. parameter server) but instead only require on-device computation and local communication with neighboring devices.
This covers for instance the classic setting of training machine learning models in large data-centers, but also emerging applications were the computations are executed directly on the consumer devices, which keep their part of the data private at all times.\footnote{Note the optimization process itself (as for instance the computed result) might leak information about the data of other nodes. We do not focus on quantifying notions of privacy in this work.}

Formally, we consider optimization problems distributed across $n$ devices or nodes of the form
\begin{align}
f^\star := \min_{\xx \in \R^d} \bigg[ f(\xx) := \frac{1}{n} \sum_{i=1}^n f_i(\xx) \bigg]\,, \label{eq:prob}
\end{align}
where $f_i \colon \R^d \to \R$ for $i\in [n] := \{1,\dots,n\}$ are the objectives defined by the local data available on each node.
We also allow each local objective $f_i$ to have stochastic optimization (or sum) structure, covering the important case of empirical risk minimization in distributed machine learning and deep learning applications.

\paragraph{Decentralized Communication.}
We model the network \emph{topology} as a graph $G=([n],E)$ with edges $\{i,j\}\in E$ if and only if nodes $i$ and $j$ are connected by a communication link, meaning that these nodes directly can exchange messages (for instance computed model updates).
The decentralized setting is motivated by centralized topologies (corresponding to a star graph) often not being possible, and otherwise often posing a significant bottleneck on the central node in terms of communication latency, bandwidth and fault tolerance.
Decentralized topologies avoid these bottlenecks and thereby offer hugely improved potential in scalability. For example, while the master node in the centralized setting receives (and sends) in each round messages from all workers, $\Theta(n)$  in total\footnote{For better connected topologies sometimes more efficient all-reduce and broadcast implementations are available.}, 
in decentralized topologies the maximal degree of the network is often constant (e.g. ring or torus) or a slowly growing function in $n$ (e.g. scale-free networks).

\paragraph{Decentralized Optimization.}
For the case of deterministic (full-gradient) optimization, recent seminal theoretical advances show that the network topology only affects higher-order terms of the convergence rate of decentralized optimization algorithms on convex problems~\cite{Scaman2017:optimal,Scaman2018:non-smooth}.
We prove the first analogue result for the important case of decentralized stochastic gradient descent (SGD), proving convergence at rate $\cO ( 1/(nT) )$ (ignoring for now higher order terms) on strongly convex functions where $T$ denotes the number of iterations.

This result is significant since stochastic methods are highly preferred for their efficiency over deterministic gradient methods in machine learning applications.
Our algorithm, \algopt, is as efficient in terms of iterations as centralized mini-batch SGD (and consequently also achieves a speedup of factor $n$ compared to the serial setting on a single node) but avoids the communication bottleneck that centralized algorithms suffer from.

\paragraph{Communication Compression.}
In distributed training, model updates (or gradient vectors) have to be exchanged between the worker nodes. To reduce the amount of data that has to be send, gradient \emph{compression} has become a popular strategy. For instance by quantization~\cite{Alistarh2017:qsgd,Wen2017:terngrad,Lin2018:deep}
or sparsification~\cite{Wangni2018:sparsification,Stich2018:sparsifiedSGD}. %

These ideas have recently been introduced also to the decentralized setting by~\citet{Tang2018:decentralized}. However, their analysis only covers unbiased compression operators with very (unreasonably) high accuracy constraints. 
Here we propose the first method that supports arbitrary low accuracy and even biased compression operators, such as in~\cite{Alistarh2018:topk,Lin2018:deep,Stich2018:sparsifiedSGD}.

\paragraph{Contributions.}
Our contributions can be summarized as follows: \vspace{-2mm}
\begin{itemize}%
 \item We show that the proposed \algopt converges at rate $\cO(1/(nT) + 1/(T \delta^2 \omega)^2)$, where $T$ denotes the number of iterations, $n$ the number of workers, $\delta$ the eigengap of the gossip (connectivity) matrix %
and $\omega \leq 1$ the compression quality factor ($\omega=1$ meaning no compression). We show that the decentralized method achieves the same speedup as centralized mini-batch SGD when the number~$n$ of workers grows. The network topology and the compression only mildly affect the convergence rate. This is verified experimentally on the ring topology and by reducing the communication by a factor of 100 ($\omega = \frac{1}{100}$).
 \item We present the first provably-converging gossip algorithm with communication compression, for the distributed average consensus problem. Our algorithm, \algcons, converges linearly at rate $\cO(1/(\delta^2\omega) \log (1/\epsilon))$ for accuracy $\epsilon > 0$, and allows arbitrary communication compression operators (including biased and unbiased ones).
  In contrast, previous work required very high-precision quantization $\omega \approx 1$ and could only show convergence towards a neighborhood of the optimal solution.
 \item \algopt significantly outperforms state-of-the-art methods for decentralized optimization with gradient compression, such as ECD-SGD and DCD-SGD introduced in~\cite{Tang2018:decentralized}, in all our experiments.
\end{itemize}

\section{Related Work}
Stochastic gradient descent (SGD)~\cite{Robbins:1951sgd,Bottou2010:sgd} 
and variants thereof are the standard algorithms for machine learning problems of the form~\eqref{eq:prob}, though it is an inherit serial algorithm that does not take the distributed setting into account.
Mini-batch SGD~\cite{Dekel2012:minibatch} is the natural parallelization of SGD for~\eqref{eq:prob} in the centralized setting, i.e. when a master node collects the updates from all worker nodes, and serves a baseline here. 

\paragraph{Decentralized Optimization.}
The study of decentralized optimization algorithms can be tracked back at least to the 1980s~\cite{Tsitsiklis1985:gossip}.
Decentralized algorithms are sometimes referred to as \emph{gossip} algorithms~\cite{Kempe2003:gossip,Xiao2014:averaging,Boyd2006:randgossip}
 as the information is not broadcasted by a central entity, but spreads---similar as gossip---along the edges specified by the communication graph.
The most popular algorithms are based on
(sub)gradient descent 
\cite{Nedic2009:distributedsubgrad,Johansson2010:distributedsubgrad},
alternating direction method of multipliers (ADMM)~\cite{Wei2012:distributedadmm,Iutzeler2013:randomizedadmm} or dual averaging~\cite{Duchi2012:distributeddualaveragig,Nedic2015:dualavg}.
\citet{cola2018nips} address the more specific problem class of generalized linear models.
\\
For the deterministic (non-stochastic) convex version of~\eqref{eq:prob} a recent line of work developed optimal algorithms based on acceleration~ \cite{Jakovetic2014:fast,Scaman2017:optimal,Scaman2018:non-smooth,Uribe:2018uk}. Rates for the stochastic setting are derived in~\cite{Shamir2014:distributedSO,Rabbat2015:mirrordescent}, under the assumption that the distributions %
on all nodes are equal.  %
This is a strong restriction which prohibits most distributed machine learning applications. Our algorithm \algopt avoids any such assumption.
Also, \cite{Rabbat2015:mirrordescent} requires multiple communication rounds per stochastic gradient computation and so is not suited for sparse communication, as the required number of communication rounds would increase proportionally to the sparsity.
\citet{Lan2018:decentralized} applied gradient sliding techniques allowing to skip some of the communication rounds.
\\
\citet{Lian2017:decentralizedSGD%
,Tang2018:d2,Tang2018:decentralized, Assran:2018sdggradpush} consider the non-convex setting with \citet{Tang2018:decentralized} also applying gradient quantization techniques to reduce the communication cost. However, their algorithms require very high precision quantization, a constraint we can overcome here.%

\paragraph{Gradient Compression.}
Instead of transmitting a full dimensional (gradient) vector $\gg \in \R^d$, methods with gradient compression transmit a compressed vector $Q(\gg)$ instead, where $Q \colon \R^d \to \R^d$ is a (random) operator chosen such that $Q(\gg)$ can be more efficiently represented, for instance by using limited bit representation (\emph{quantization}) or enforcing \emph{sparsity}. A class of very common quantization operators is based on random dithering~\cite{Goodall1951:randdithering,Roberts1962:randdithering} that is in addition also unbiased, $\EE{\xi}{Q(\xx)}=\xx$, $\forall \xx \in \R^d$, see \cite{Alistarh2017:qsgd,Wen2017:terngrad,Zhang2017:zipml}.
Much sparser vectors can be obtained by random sparsification techniques that randomly mask the input vectors and only preserve a constant number of coordinates~\cite{Wangni2018:sparsification,Konecny:2018sparseMean,Stich2018:sparsifiedSGD}.
Techniques that do not directly quantize gradients, but instead maintain additional states are known to perform better in theory and practice~\cite{Seide2015:1bit,Lin2018:deep,Stich2018:sparsifiedSGD}, an approach that we pick up here. Our analysis also covers deterministic  and biased compression operators, such as in~\cite{Alistarh2018:topk,Stich2018:sparsifiedSGD}. 
We will not further distinguish between sparsification and quantization approaches, and refer to both of them as \emph{compression} operators in the following.

\paragraph{Distributed Average Consensus.}
In the decentralized setting, the average consensus problem consists in finding the average vector of $n$ local vectors (see\ \eqref{def:consensus} below for a formal definition).
The problem is an important sub-routine of many decentralized algorithms.
It is well known that gossip-type algorithms converge linearly for average consensus~ \cite{Kempe2003:gossip,Xiao2014:averaging,Olfati2004:consensus,Boyd2006:randgossip}. 
However, for consensus algorithms with compressed communication it has been remarked that the standard gossip algorithm does not converge to the correct solution~\cite{Xiao2005:drift}. 
The proposed schemes in~\cite{Carli2007:noise,Nedic2008:quantizationeffects,%
Aysal2008:dithering,Carli2010:quantizedconsensus,Yuan2012:distributedquant} do only converge to a neighborhood (whose size depends on the compression accuracy) of the solution. %
\\
In order to converge, adaptive schemes (with varying compression accuracy) have been proposed~\cite{Carli2010:codingschemes,Fang2010:onebit,%
Li2011:quantizedconsensus,Thanou2013:quantizationrefinement}.
However, these approaches fall back to full (uncompressed) communication to reach high accuracy. %
In contrast, our method converges linearly to the true solution, even for arbitrary compressed %
 communication, without requiring adaptive accuracy. We are not aware of a method in the literature with similar guarantees.

\section{Average Consensus with Communication Compression}

In this section we present \algcons, a novel gossip algorithm for distributed average consensus with compressed communication.
As mentioned, the average consensus problem is an important special case of type~\eqref{eq:prob}, and formalized as\vspace{-2mm}
\begin{align}
 \overline{\xx} := \frac{1}{n} \sum_{i=1}^n \xx_i\,, \label{def:consensus}
\end{align}
for vectors $\xx_i \in \R^d$ distributed on $n$ nodes
(consider $f_i(\xx)=\frac{1}{2}\norm{\xx-\xx_i}^2$ in~\eqref{eq:prob}). 
Our proposed algorithm will later serve as a crucial primitive in our optimization algorithm for the general optimization problem~\eqref{eq:prob}, but is of independent interest for any average consensus problem with communication constraints. 

In Sections~\ref{sec:gossip1}--\ref{sec:gossip3} below we first review existing schemes that we later consider as baselines for the numerical comparison. The novel algorithm follows in Section~\ref{sec:ours}.

\subsection{Gossip algorithms}
\label{sec:gossip1}
The classic decentralized algorithms for the average consensus problem are \emph{gossip} type algorithms (see e.g.~\cite{Xiao2014:averaging}) that generate sequences $\bigl\{\xx_i^{(t)}\bigr\}_{t\geq 0}$ on every node $i \in [n]$ by iterations of the form 
\begin{align}
\xx^{(t + 1)}_i := \xx_i^{(t)} + \gamma \sum_{j = 1}^{n} w_{ij} \Delta_{ij}^{(t)}\,. \label{eq:iter_vector}
\end{align}
Here $\gamma \in (0,1]$ denotes a stepsize parameter, $w_{ij} \in [0,1]$ averaging weights and $\Delta_{ij}^{(t)} \in \R^d$ denotes a vector that is sent from node $j$ to node $i$ in iteration $t$. Note that no communication is required if $w_{ij}=0$. If we assume symmetry, $w_{ij}=w_{ji}$, the weights naturally define the communication graph $G=([n],E)$ with edges $\{i,j\} \in E$ if $w_{ij} > 0$ and self-loops $\{i\} \in E$ for $i \in [n]$. 
The convergence rate of scheme~\eqref{eq:iter_vector} crucially depends on the connectivity matrix $W \in \R^{n \times n}$ of the network defined as $(W)_{ij} = w_{ij}$, also called the interaction or gossip matrix.

\begin{definition}[Gossip matrix]\label{def:W}
	We assume that $W \in [0,1]^{n \times n}$ is a symmetric ($W=W^\top$) doubly stochastic ($W\1=\1$,$\1^\top W = \1^\top$) matrix with eigenvalues  $1 = |\lambda_1(W)| > |\lambda_2(W)| \geq \dots \geq |\lambda_n(W)|$ and spectral gap
	\begin{align}
	\delta := 1 - |\lambda_2(W)| \in (0,1]\,. \label{def:spectral_gap}
	\end{align}
	It will also be convenient to define 
	\begin{align}
	\rho &:=  1 - \delta\,, & &\text{and} & %
	\beta &:= \norm{I-W}_2 \in [0,2]\,. %
	\label{def:betarho}
	\end{align}
\end{definition}
Table~\ref{tab:rho} gives a few values of the spectral gap for commonly used network topologies (with uniform averaging between the nodes). It is well known that simple matrices $W$ with $\delta>0$ do exist for every connected graph.  %

\begin{table}[t]
\centering
\begin{tabular}{l|l|l}
graph/topology & $\delta^{-1}$ & node degree \\ \hline
ring & $\cO (n^2)$ & 2\\
2d-torus & $\cO (n)$ & 4 \\
fully connected & $\cO(1)$ & $n-1$
\end{tabular}
\caption{Spectral gap $\delta$ for some important network topologies  on $n$ nodes (see e.g. \citep[p. 169]{Aldous:2014EigengapBook}) for uniformly averaging $W$, i.e. $w_{ij} = \frac{1}{deg(i)} = \frac{1}{deg(j)}$ for $\{i,j\} \in E$.}%
\label{tab:rho}
\end{table}

\subsection{Gossip with Exact Communication}
\label{sec:gossip2}
For a fixed gossip matrix $W$, the classical algorithm analyzed in~\cite{Xiao2014:averaging} %
corresponds to the choice
\begin{align}
\gamma &:= 1, & \Delta_{ij}^{(t)} &:= \xx^{(t)}_j - \xx^{(t)}_i , \tag{E-G} \label{eq:boyd}
\end{align}
in~\eqref{eq:iter_vector}, with~\eqref{eq:boyd} standing for \emph{exact gossip}. This scheme can also conveniently be written in matrix notation as
\begin{align}
X^{(t+1)} := X^{(t)} + \gamma X^{(t)} (W-I)\,,  \label{eq:boyd2}
\end{align}
for iterates $X^{(t)}:=[\xx_{1}^{(t)},\dots,\xx_{n}^{(t)}] \in \R^{d \times n}$.
\begin{theorem}\label{th:boyd2_rate}
Let $\gamma \in (0,1]$ and $\delta$ be the spectral gap of $W$. Then the iterates of~\eqref{eq:boyd} converge linearly to the average $\overline{\xx} = \frac{1}{n} \sum_{i = 1}^{n}\xx_i^{(0)}$ with the rate
\begin{align*}
	\sum_{i = 1}^n\norm{\xx_i^{(t)} - \overline{\xx}}^2 \leq (1-\gamma\delta)^{2t} \sum_{i = 1}^n\norm{\xx_i^{(0)} - \overline{\xx}}^2\,.
	\end{align*}
\end{theorem}
For $\gamma = 1$ this corresponds to the classic result in e.g.~\cite{Xiao2014:averaging}, here we slightly extend the analysis for arbitrary stepsizes. The short proof shows the elegance of the matrix notation (that we will later also adapt for the proofs that will follow).
\begin{proof}[Proof for $\gamma = 1$.]
Let $\overline{X} := [\overline{\xx},\dots,\overline{\xx}] \in \R^{d \times n}$. Then for $\gamma = 1$ the theorem  follows from the observation
\begin{align*}
 \norm{X^{(t+1)}-\overline{X}}_F^2  &\stackrel{\eqref{eq:boyd2}}{=} \norm{(X^{(t)}-\overline{X})W}_F^2 \\
 &= \norm{(X^{(t)}-\overline{X})(W- \tfrac{1}{n}\1 \1^\top)}_F^2 \\
 &\leq \norm{W- \tfrac{1}{n}\1 \1^\top}_2^2  \norm{X^{(t)}-\overline{X}}_F^2 \\
 &= \rho^2 \norm{X^{(t)}-\overline{X}}_F^2\,.
\end{align*}
Here on the second line we used the crucial identity $X^{(t)}(\tfrac{1}{n} \1 \1^\top) = \overline{X}$, i.e. the algorithm preserves the average over all iterations. This can be seen from~\eqref{eq:boyd2}: %
\begin{align*}
X^{(t+1)}(\tfrac{1}{n} \1 \1^\top) = X^{(t)}  W(\tfrac{1}{n} \1 \1^\top)  = X^{(t)}  (\tfrac{1}{n} \1 \1^\top)  = \overline{X}\,,
\end{align*}
by Definition~\ref{def:W}.
The proof for arbitrary $\gamma$ follows the same lines and is given in the appendix.
\end{proof}

\subsection{Gossip with Quantized Communication}
\label{sec:gossip3}
In every round of scheme~\eqref{eq:boyd} a full dimensional vector $\gg \in \R^d$ is exchanged between two neighboring nodes for every link on the communication graph (node $j$ sends $\gg = \xx_j^{(t)}$ to all its neighbors $i$, $\{i,j\} \in E$). %
A natural way to reduce the communication is to compress $\gg$ before sending it, denoted as
$Q(\gg)$, for  a (potentially random) compression $Q \colon \R^d \to \R^d$. Informally, we can think of $Q$ as either a sparsification operator (that enforces sparsity of $Q(\gg)$) or a quantization operator that reduces the number of bits required to represent $Q(\gg)$. For instance random rounding to less precise floating point numbers or to integers. 

\citet{Aysal2008:dithering} propose the quantized gossip~\eqref{eq:consensus_first},
\begin{align}\label{eq:consensus_first}
\gamma &:= 1, & \Delta_{ij}^{(t)} &:= Q(\xx^{(t)}_j) - \xx^{(t)}_i\,, \tag{Q1-G}
\end{align}
in scheme~\eqref{eq:iter_vector}, i.e. to apply the compression operator directly on the message that is send out from node $j$ to node $i$.
However, this algorithm does not preserve the average of the iterates over the iterations, $\tfrac{1}{n}\sum_{i=1}^n\xx_i^{(0)} \neq \tfrac{1}{n} \sum_{i=1}^n\xx_i^{(t)}$ for $t \geq 1$, and as a consequence does not converge to the optimal solution $\overline{\xx}$ of~\eqref{def:consensus} (though in practice often to a close neighborhood).

An alternative proposal by~\citet{Carli2007:noise} alleviates this drawback. The scheme
\begin{align}\label{eq:consensus_second}
\gamma &:= 1, & \Delta_{ij}^{(t)} &:= Q(\xx^{(t)}_j) - Q(\xx^{(t)}_i)\,, \tag{Q2-G}
\end{align}
preserves the average of the iterates over the iterations. However, the scheme also fails to converge for arbitrary precision. If $\overline{\xx} \neq \0$, the noise introduced by the compression, $\bigl\|Q(\xx_j^{(t)})\bigr\|$, does not vanish for $t \to \infty$. As a consequence, the iterates oscillate around $\overline{\xx}$ when compression error becomes larger than the suboptimality $\bigl\|\xx_i^{(t)} - \overline{\xx}\bigr\|$.

Both these schemes have been theoretically studied in~\cite{Carli2010:quantizedconsensus} under assumption of unbiasendness, i.e. assuming $\EE{Q}{Q(\xx)}=\xx$ for all $\xx \in \R^d$ (and we will later also adopt this theoretically understood setting in our experiments).

\subsection{Proposed Method for Compressed Communication}
\label{sec:ours}
We propose the novel compressed gossip scheme\ \algcons that supports a much larger class of compression operators, beyond unbiased quantization as for the schemes above. The algorithm can be summarized as 
\begin{align}
\begin{split}
	\hat{\xx}_j^{(t + 1)} &:= \hat{\xx}_j^{(t)} + Q(\xx^{(t)}_j - \hat{\xx}_j^{(t)})\,,\\
	\Delta_{ij}^{(t)} &:= \hat{\xx}_j^{(t + 1)} - \hat{\xx}_i^{(t + 1)}\,,%
\end{split} \tag{\eqncons}%
\label{eq:ours}
\end{align}
for a stepsize $\gamma < 1$ depending on the specific compression operator $Q$ (this will be detailed below).
Here $\hat{\xx}_i^{(t)} \in \R^d$ denote additional variables that are stored\footnote{A closer look reveals that actually only 2 additional vectors have to be stored per node (refer to Appendix~\ref{sect:efficient_implementation}).
} by all neighbors $j$ of node $i$, $\{i,j\} \in E$, as well as on node $i$ itself. 

We will show in Theorem~\ref{th:consensus} below that this scheme (i) preserves the averages of the iterates $\xx_i^{(t)}$, $i \in [n]$ over the iterations $t \geq 0$. Moreover, (ii) the noise introduced by the compression operator vanishes as $t \to 0$. Precisely, we will show that  $(\xx_i^{(t)},\hat{\xx}_i^{(t)}) \to (\overline{\xx},\overline{\xx})$ for $t \to \infty$ for every $i \in [n]$. Consequently, the argument for $Q$ in~\eqref{eq:ours} goes to zero, and the noise introduced by $Q$ can be controlled.

The proposed scheme is summarized in Algorithm~\ref{alg:consensus}. Every worker $i \in [n]$ stores and updates its own local variable $\xx_i$ as well as the variables $\hat{\xx}_j$ for all neighbors (including itself) $j : \{i, j\}\in E$.

\begin{algorithm}[t]
\renewcommand{\algorithmiccomment}[1]{#1}
	\caption{\algcons %
	}\label{alg:consensus}
	\begin{algorithmic}[1]
		\INPUT{:  Initial values $\xx_i^{(0)} \in \R^d$ on each node $i \in [n]$, stepsize $\gamma$, communication graph $G = ([n], E)$ and mixing matrix $W$, initialize $\hat{\xx}_i^{(0)} := \0$ $\forall i$}\\
		\FOR[{{\it in parallel for all workers $i \in [n]$}}]{$t$ \textbf{in} $0\dots T-1$}
		\STATE $\qq_i^{(t)} := Q(\xx_i^{(t)} - \hat{\xx}_i^{(t)})$
		\FOR{neighbors $j \colon \{i,j\} \in E$ (including $\{i\} \in E$)}
		\STATE Send $\qq_i^{(t)}$ and receive $\qq_j^{(t)}$ 
		\STATE $\hat{\xx}_j^{(t + 1)} := \hat{\xx}_j^{(t)} + \qq_j^{(t)}$%
		\ENDFOR
		\vspace{1mm}
		\STATE $\xx_i^{(t + 1)} := \xx_i^{(t)} + \gamma \!\!\displaystyle\sum_{j: \{i, j\}\in E}w_{ij}\left(\hat{\xx}_j^{(t + 1)} - \hat{\xx}_i^{(t + 1)}\right)$
		\ENDFOR
	\end{algorithmic}
\end{algorithm}

Algorithm~\ref{alg:consensus} seems to require each machine to store $deg(i) + 2$ vectors. This is not necessary and the algorithm could be re-written in a way that every node stores only {\it three} vectors: $\xx_i$, $\hat{\xx}_i$ and $\mathbf{s}_i = \sum_{j:\{i, j\} \in E} w_{ij}\hat{\xx}_j$. For simplicity, we omit this technical modification here and refer to Appendix~\ref{sect:efficient_implementation} for the exact form of the memory-efficient algorithm.

\subsection{Convergence Analysis for \algcons%
}
\label{sec:rateconsensus}
We analyze Algorithm \ref{alg:consensus} under the following general quality notion for the compression operator $Q$.

\begin{assumption}[Compression operator]\label{assump:q}
We assume that the compression operator $Q \colon \R^d \to \R^d$ satisfies
	\begin{align}
	\EE{Q}{\norm{Q(\xx) - \xx}}^2 &\leq (1 - \omega) \norm{\xx}^2, & &\forall \xx \in \R^d \,, \label{def:omega}
	\end{align}
	for a parameter $\omega > 0$. Here $\mathbb{E}_Q$ denotes the expectation over the internal randomness of operator $Q$.
\end{assumption}
\noindent\textbf{Example operators} that satisfy~\eqref{def:omega} include\vspace{-2mm}
\begin{itemize}%
\item \emph{sparsification}: Randomly selecting $k$ out of $d$ coordinates ($\operatorname{rand}_k$), or the $k$ coordinates with highest magnitude values ($\operatorname{top}_k$) give $\omega = \frac{k}{d}$ \citep[Lemma A.1]{Stich2018:sparsifiedSGD}.
\item \emph{randomized gossip}: Setting $Q(\xx) = \xx$ with probability $p \in (0,1]$ and $Q(\xx)=\0$ otherwise, gives $\omega = p$.
\item \emph{rescaled unbiased estimators}: suppose $\EE{Q}{Q(\xx)}=\xx$, $\forall \xx \in \R^d$ and $\EE{Q}{\norm{Q(\xx)}^2}\leq \tau \norm{\xx}^2$, then
$Q'(\xx) := \frac{1}{\tau} Q(\xx)$ satisfies~\eqref{def:omega} with $\omega = \frac{1}{\tau}$.
\item \emph{random quantization}: For precision (levels) $s \in \N_+$, and $\tau = (1+\min\{d/s^2,\sqrt{d}/s\})$ the quantization operator
\begin{align*}
\operatorname{qsgd}_s(x) =   \frac{\operatorname{sign}(\xx) \cdot  \norm{\xx}}{s \tau}\cdot \left\lfloor s \frac{ \abs{\xx}}{\norm{\xx}} + \xi \right\rfloor\,,
\end{align*}
for random variable $\xi \sim_{\rm u.a.r.} [0,1]^{d}$ satisfies~\eqref{def:omega} with $\omega = \frac{1}{\tau}$ \citep[Lemma 3.1]{Alistarh2017:qsgd}. %
\end{itemize}

\begin{theorem}\label{th:consensus}
	\algcons (Algorithm~\ref{alg:consensus}) converges 
	linearly for average consensus:%
	\begin{align*}
		e_t \leq \left(1 - \dfrac{\delta^2 \omega}{82}\right)^t e_0\,,
	\end{align*}
	when using the stepsize $\gamma := \frac{\delta^2 \omega}{16\delta + \delta^2 + 4 \beta^2 + 2 \delta\beta^2 - 8 \delta \omega}$,
	where~$\omega$ is the compression factor as in Assumption~\ref{assump:q}, 
	and $e_t = \EE{Q}{\sum_{i = 1}^{n}\left( \bigl\| \xx_i^{(t)} - \overline{\xx}\bigr\|^2 + \bigl\| \xx_i^{(t)} - \hat{\xx}_i^{(t + 1)}\bigr\|^2 \right)}$.
\end{theorem}
For the proof we refer to the appendix, where we used matrix notation to simplify derivations.
For the exact communication case $\omega = 1$ we recover the rate from Theorem~\ref{th:boyd2_rate} for stepsize $\gamma < 1$ up to constant factors (which seems to be a small artifact of our proof technique). The theorem shows convergence for arbitrary $\omega > 0$, showing the superiority of scheme~\eqref{eq:ours} over~\eqref{eq:consensus_first} and~\eqref{eq:consensus_second}.

\section{Decentralized Stochastic Optimization}

In this section we leverage our proposed average consensus Algorithm~\ref{alg:consensus} to achieve consensus among the compute nodes in a decentralized optimization setting with communication restrictions. %

In the decentralized optimization setting \eqref{eq:prob}, not only does every node have a different local objective $f_i$, but we also allow each $f_i$ to have stochastic optimization (or sum) structure, that is %
\begin{align} 
f_i(\xx) := \EE{\xi_i \sim \cD_i}{F_i(\xx,\xi_i)}\,,\label{eq:func}
\end{align}
for a loss function $F_i\colon \R^d \times \Omega \to \R$ and distributions $\cD_1, \dots, \cD_n$ which can be different on every node. %
Our framework therefore covers both stochastic optimization (e.g. when %
all $\cD_i$ are identical) and empirical risk minimization (as in machine learning and deep learning applications) when the $\cD_i$'s are discrete with disjoint support. %

\subsection{Proposed Scheme for Decentralized Optimization}
Our proposed method \algopt---Communication-Compressed Decentralized SGD---is stated in Algorithm~\ref{alg:quantized_decentralized_sgd}. %

\begin{algorithm}[H]
\renewcommand{\algorithmiccomment}[1]{#1}
	\caption{\algopt %
	}\label{alg:quantized_decentralized_sgd}
	\begin{algorithmic}[1]
		\INPUT{:  Initial values $\xx_i^{(0)} \in \R^d$ on each node $i \in [n]$, consensus stepsize $\gamma$, SGD stepsizes $\{\eta_t\}_{t \geq 0}$, communication graph $G = ([n], E)$ and mixing matrix $W$, initialize $\hat{\xx}_i^{(0)} := \0$ $\forall i$}\\
		\FOR[{{\it in parallel for all workers $i \in [n]$}}]{$t$\textbf{ in} $0\dots T-1$}
		\STATE Sample $\xi_i^{(t)}$, compute gradient $\gg_i^{(t)} \!:= \nabla F_i(\xx_i^{(t)}\!, \xi_i^{(t)})$\!
		\STATE $\xx_i^{(t + \frac{1}{2})} := \xx_i^{(t)} - \eta_t \gg_i^{(t)}$  
		\STATE $\qq_i^{(t)} := Q(\xx_i^{(t + \frac{1}{2})} - \hat{\xx}_i^{(t)})$ 
		\FOR{neighbors $j \colon \{i,j\} \in E$ (including $\{i\} \in E$)} 
		\STATE Send $\qq_i^{(t)}$ and receive $\qq_j^{(t)}$ 
		\STATE $\hat{\xx}^{(t+1)}_j := \qq^{(t)}_j + \hat{\xx}_j^{(t)}$
		\ENDFOR
		\vspace{1mm}
		\STATE $\xx_i^{(t + 1)} := \xx_i^{(t + \frac{1}{2})} + \gamma \!\!\displaystyle\sum_{j: \{i, j\}\in E} \!\!w_{ij} \left(\hat{\xx}^{(t+1)}_j \!- \hat{\xx}^{(t+1)}_i\right)$\vspace{-2mm} 
		\ENDFOR 
	\end{algorithmic}
\end{algorithm}
The algorithm consists of four parts. The stochastic gradient step in line\ 3, application of the compression operator in step\ 4, and the \eqref{eq:ours} local communication in lines\ 5--8 followed by the final iterate update in line\ 9.

\begin{remark}\label{rem:exact}
 As a special case without any communication compression, and for consensus stepsize $\gamma=1$ as in exact gossip~\eqref{eq:boyd}, \algopt (Algorithm~\ref{alg:quantized_decentralized_sgd}) recovers the following standard variant of decentralized SGD with gossip (similar e.g. to \cite{Sirb2016:delayedconsensus,Lian2017:decentralizedSGD}), stated for illustration in Algorithm~\ref{alg:all-to-all}.
\end{remark}
\begin{algorithm}[H]
\renewcommand{\algorithmiccomment}[1]{#1}
	\caption{\textsc{Plain Decentralized SGD}}
	\label{alg:all-to-all}
	\begin{algorithmic}[1]
		\FOR[{{\it in parallel for all workers $i \in [n]$}}]{$t$ \textbf{in} $0\dots T-1$}
		\STATE Sample $\xi_i^{(t)}$, compute gradient $\gg_i^{(t)} \!:= \nabla F_i(\xx_i^{(t)}\!, \xi_i^{(t)})$\!
		\STATE $\xx_i^{(t + \frac{1}{2})} := \xx_i^{(t)} - \eta_t \gg_{i}^{(t)}$
		\STATE Send $\xx_i^{(t + \frac{1}{2})}$ to neighbors
		\STATE $\xx_i^{(t + 1)} := \sum_{i = 1}^{n} w_{ij} \xx_j^{(t + \frac{1}{2})}$
		\ENDFOR
	\end{algorithmic}
\end{algorithm}

\subsection{Convergence Analysis for \algopt}
\begin{assumption}\label{assump:f}
	We assume that each function $f_i \colon \R^d \to \R$ for $i \in [n]$ is $L$-smooth and $\mu$-strongly convex and that the variance on each worker is bounded 
	\begin{align*}
    &\EE{\xi_i}{\norm{\nabla F_i(\xx, \xi_i) - \nabla f_i(\xx)}^2}\leq \sigma_i^2\,, & &\forall \xx \in \R^d, i \in [n], \\
	&\EE{\xi_i}{\norm{\nabla F_i(\xx, \xi_i)}}^2 \leq G^2\,, & &\forall \xx \in \R^d, i \in [n],
	\end{align*}
	where $\mathbb{E}_{\xi_i}[\cdot]$ denotes the expectation over $\xi_i \sim \cD_i$.
	It will be also convenient to denote 
	\begin{align*}
	\overline{\sigma}^2 := \frac{1}{n}\sum_{i = 1}^n\sigma_i^2.
	\end{align*}
\end{assumption}
 For the (standard) definitions of smoothness and strong convexity we refer to Appendix~\ref{sec:defsmooth}. These assumptions could be relaxed to only hold for $\xx \in \bigl\{\xx_i^{(t)}\bigr\}_{t=1}^T$, the set of iterates of Algorithm~\ref{alg:quantized_decentralized_sgd}.

\begin{theorem}\label{th:decentralized_sgd}
	Under Assumption~\ref{assump:f}, Algorithm~\ref{alg:quantized_decentralized_sgd} with SGD stepsizes  $\eta_t := \frac{4}{\mu(a + t)}$ for parameter $a \geq \max\left\{ \frac{410}{\delta^2 \omega}, 16\kappa \right\}$ for condition number $\kappa = \frac{L}{\mu}$ and consensus stepsize $\gamma :=\gamma(\delta,\omega)$ chosen as in Theorem~\ref{th:consensus}, converges with the rate 
	\begin{align*}
	\E{\!\Upsilon^{(T)}} \!=\! \cO\left(\dfrac{\overline{\sigma}^2}{\mu n T}\right) \!+\! \cO\left(\dfrac{\kappa G^2}{\mu \omega^2\delta^4 T^2}\right) \!+\! \cO\left(\dfrac{G^2}{\mu\omega^3 \delta^6 T^3}\right),
	\end{align*}
	where $\Upsilon^{(T)} :=  f(\xx_{avg}^{(T)}) - f^\star$ for an averaged iterate $\xx_{avg}^{(T)} = \frac{1}{S_T}\sum_{t = 0}^{T - 1} w_t \overline{\xx}^{(t)}$ with weights $w_t = (a + t)^2$, and $S_T = \sum_{t = 0}^{T - 1}w_t$. 
	As reminder, $\delta$ denotes the eigengap of $W$, and $\omega$ the compression ratio.
\end{theorem}
For the proof we refer to the appendix. When $T$ and $\overline{\sigma}$ are sufficiently large, the second two terms become negligible compared to $\cO\bigl(\frac{\overline{\sigma}^2}{\mu n T}\bigr)$---and we recover the convergence rate of of mini-batch SGD in the centralized setting and with exact communication. This is because topology (parameter $\delta$) and compression (parameter $\omega)$ only affect the higher-order terms in the rate.  We also see that we obtain in this setting a $n \times$ speed up compared to the serial implementation of SGD on only one worker.

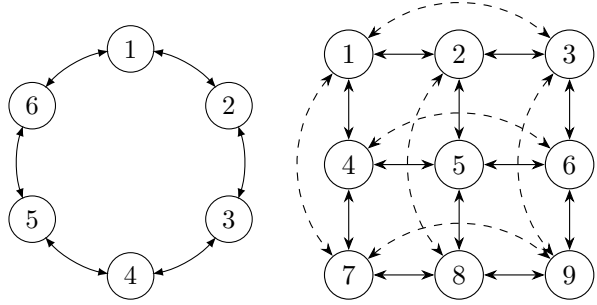
\begin{figure}[t]
	\hfill
	\resizebox{.5\linewidth}{!}{
		\resizebox{.22\linewidth}{!}{
			\begin{tikzpicture}[scale=1]
			\def \n {6}
			\def \radius {1.5cm}
			\def \margin {-11} %
			
			\foreach \s in {1,...,\n}
			{
				\node[draw, circle] at ({-360/\n * (\s - 1) + 90}:\radius) {$\s$};
				\draw[<->, >=latex] ({-360/\n * (\s - 1)+\margin + 90}:\radius)
				arc ({-360/\n * (\s - 1)+\margin + 90}:{-360/\n * (\s)-\margin + 90}:\radius);
			}
			\end{tikzpicture}
		}
		\hfill
		\resizebox{.28\linewidth}{!}{
			\begin{tikzpicture}[scale=0.7]
			\node[shape=circle,draw=black] (1) at (-2,2) {1};
			\node[shape=circle,draw=black] (2) at (0,2) {2};
			\node[shape=circle,draw=black] (3) at (2,2) {3};
			\node[shape=circle,draw=black] (4) at (-2,0) {4};
			\node[shape=circle,draw=black] (5) at (0,0) {5};
			\node[shape=circle,draw=black] (6) at (2,0) {6};
			\node[shape=circle,draw=black] (7) at (-2,-2) {7};
			\node[shape=circle,draw=black] (8) at (0,-2) {8};
			\node[shape=circle,draw=black] (9) at (2,-2) {9};
			\begin{scope}[>={Stealth[black]},
			every edge/.style={draw=black}]
			\path [<->] (1) edge node[left] {} (2);
			\path [<->] (3) edge node[left] {} (2);
			\path [dashed,<->] (1) edge[bend left=40] node[left] {} (3);
			\path [<->] (4) edge node[left] {} (5);
			\path [<->] (5) edge node[left] {} (6);
			\path [dashed,<->] (4) edge[bend left=40] node[left] {} (6);
			\path [<->] (7) edge node[left] {} (8);
			\path [<->] (8) edge node[left] {} (9);
			\path [dashed,<->] (7) edge[bend left=40] node[left] {} (9);
			
			\path [<->] (1) edge node[left] {} (4);
			\path [<->] (4) edge node[left] {} (7);
			\path [dashed,<->] (1) edge[bend right=40] node[left] {} (7);
			
			\path [<->] (2) edge node[left] {} (5);
			\path [<->] (5) edge node[left] {} (8);
			\path [dashed,<->] (2) edge[bend right=40] node[left] {} (8);
			
			\path [<->] (3) edge node[left] {} (6);
			\path [<->] (6) edge node[left] {} (9);
			\path [dashed,<->] (3) edge[bend right=40] node[left] {} (9);
			\end{scope}
			\end{tikzpicture}
		}
	}
	\hfill\null
	\vspace{-3mm}
	\caption{Ring topology (left) and Torus topology (right). }\label{fig:ring_and_torus_topology}
\end{figure}

\section{Experiments}
In this section we first compare \algcons to the gossip baselines from Section~\ref{sec:expgossip} and then compare the \algopt to state of the art decentralized stochastic optimization schemes (that also support compressed communication) in Section~\ref{sec:expSGD}.

\subsection{Shared Experimental Setup}
For our experiments we always report the \emph{number of iterations} of the respective scheme, as well as \emph{the number of transmitted bits}. These quantities are independent of systems architectures and network bandwidth. %

\paragraph{Datasets.}
In the experiments we rely on the $epsilon$ \cite{Sonnenburg:EpsilonDataset} and $rcv1$ \cite{Lewis:2004RCV1} datasets (cf. Table~\ref{tab:datasets}).

\paragraph{Compression operators.}
We use the ($\operatorname{rand}_k$), ($\operatorname{top}_k$) and $(\operatorname{qsgd}_s)$ compression operators as described in Section~\ref{sec:rateconsensus}, where we choose $k$ to be $1\%$ of all coordinates and $s \in \{2^4,2^8\}$, only requiring $4$, respectively $8$ bits to represent a coordinate.

Note that in contrast to \algcons, the earlier schemes \eqref{eq:consensus_first} and \eqref{eq:consensus_second} were both analyzed in~\cite{Carli2010:quantizedconsensus} for unbiased compression operators. In order to reflect this theoretical understood setting we use the rescaled operators ($\frac{d}{k}\cdot \operatorname{rand}_k$) and $(\tau\cdot\operatorname{qsgd}_s)$ in combination with those schemes. 

\subsection{Average Consensus}
\label{sec:expgossip}

\begin{figure}
	\centering
	\begin{minipage}{0.8\linewidth}
		\includegraphics[width=0.48\linewidth]{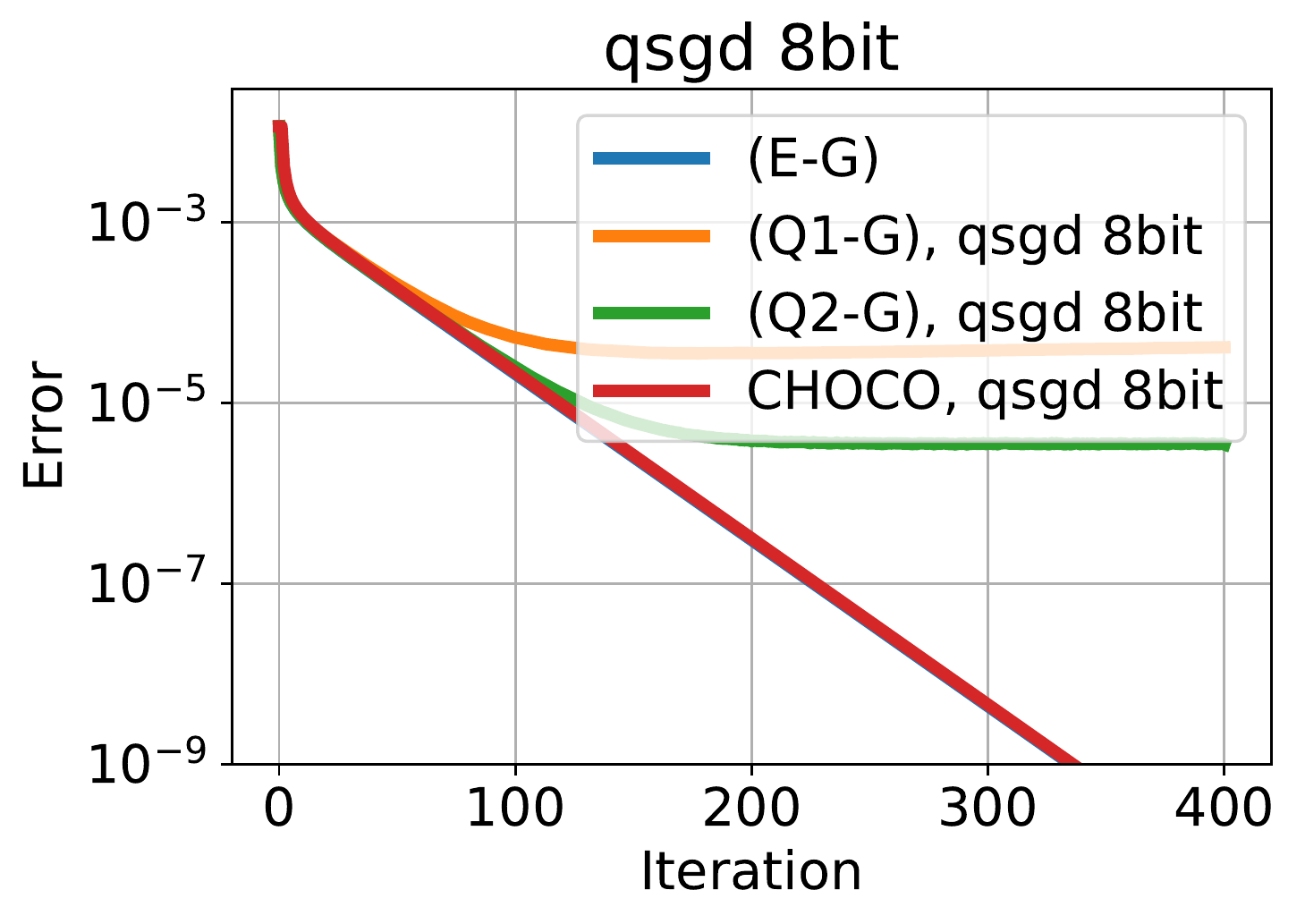}
		\includegraphics[width=0.48\linewidth]{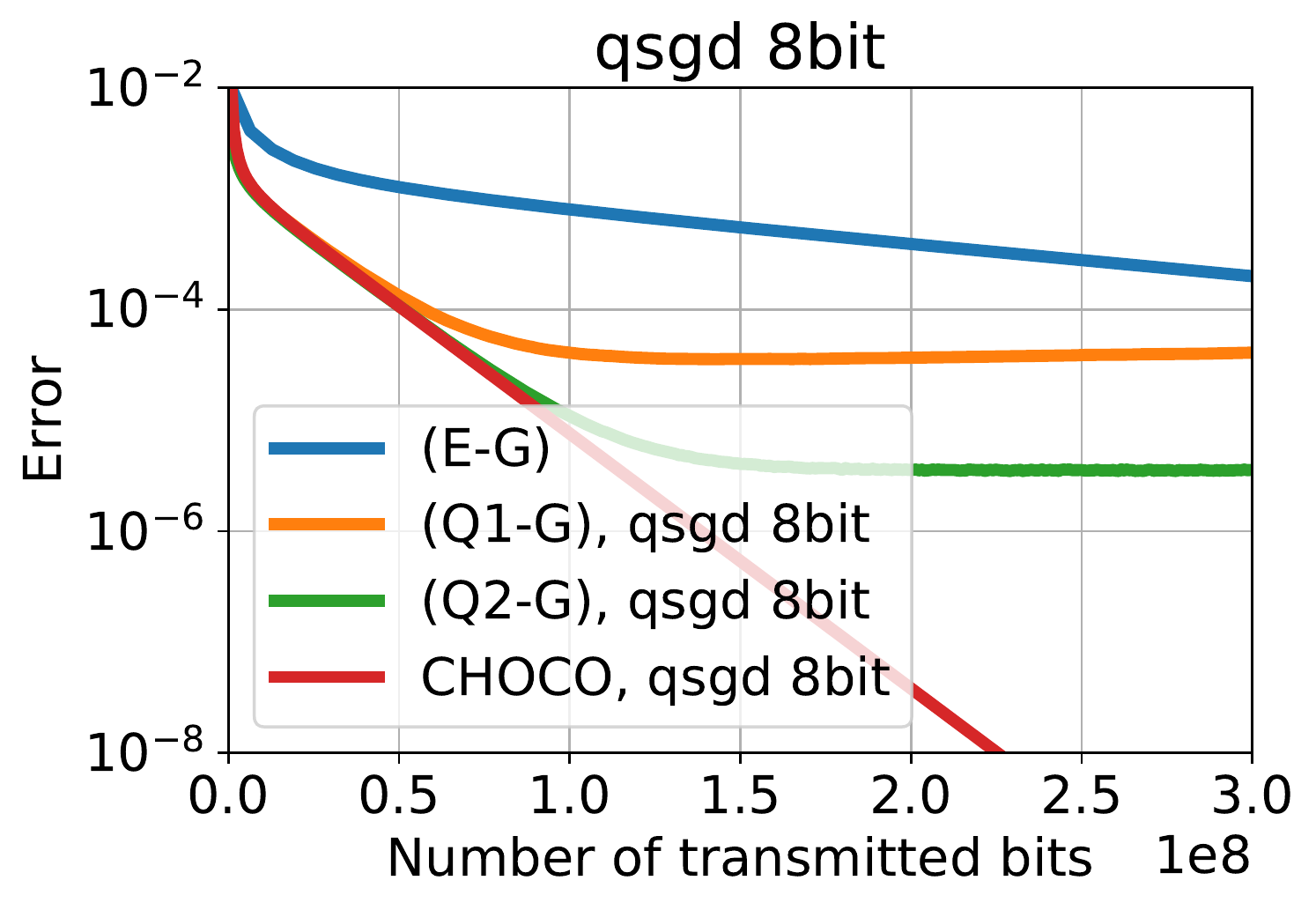}
		\vspace{-2mm}
		\caption{Average consensus on the ring topology with $n=25$ nodes, $d=2000$ coordinates and $(\operatorname{qsgd}_{256})$ compression}\label{fig:average_qsgd_8}
	\end{minipage}
	\hfill
	\centering
	\begin{minipage}{0.8\linewidth}
		\includegraphics[width=0.48\linewidth]{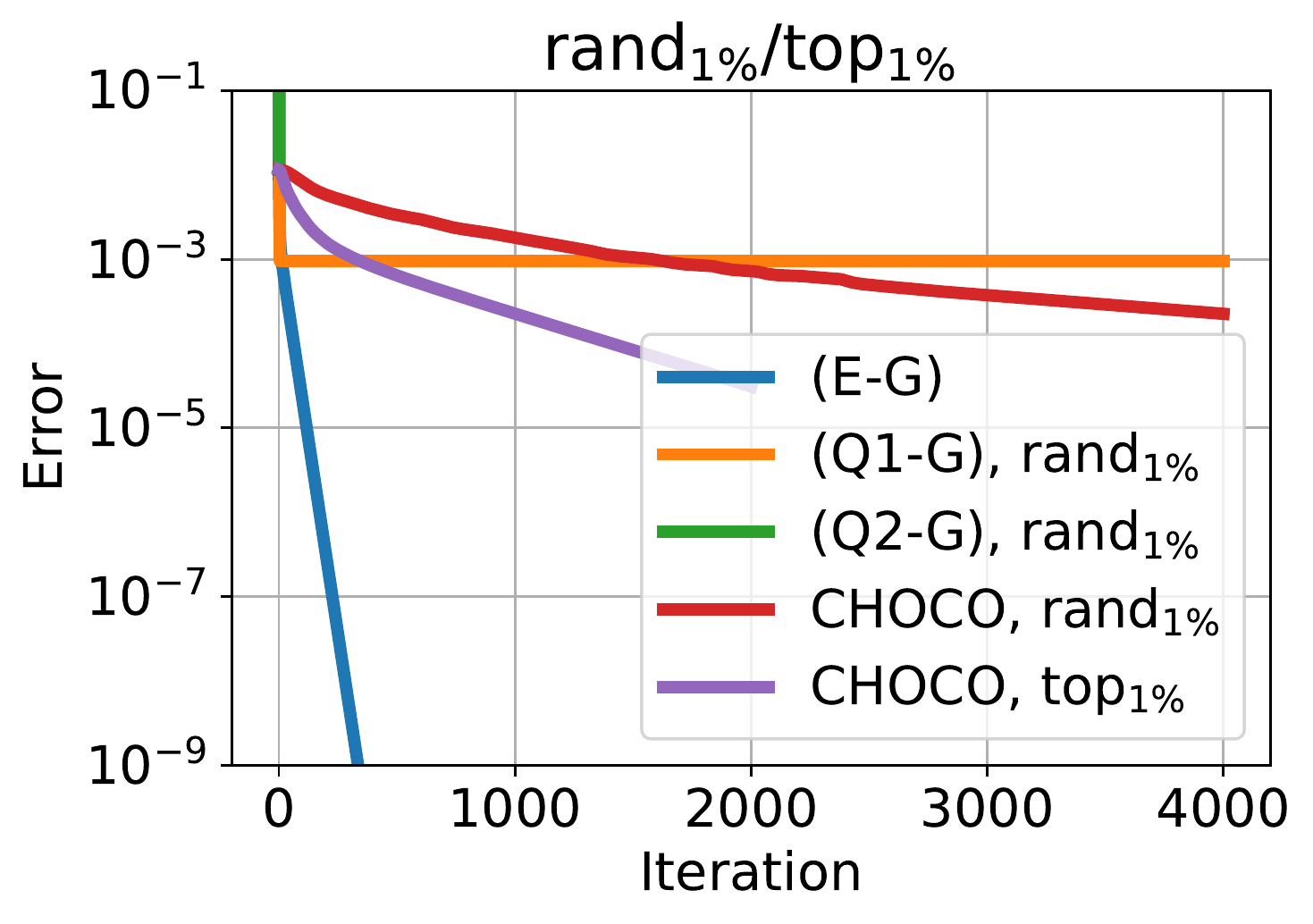}
		\includegraphics[width=0.48\linewidth]{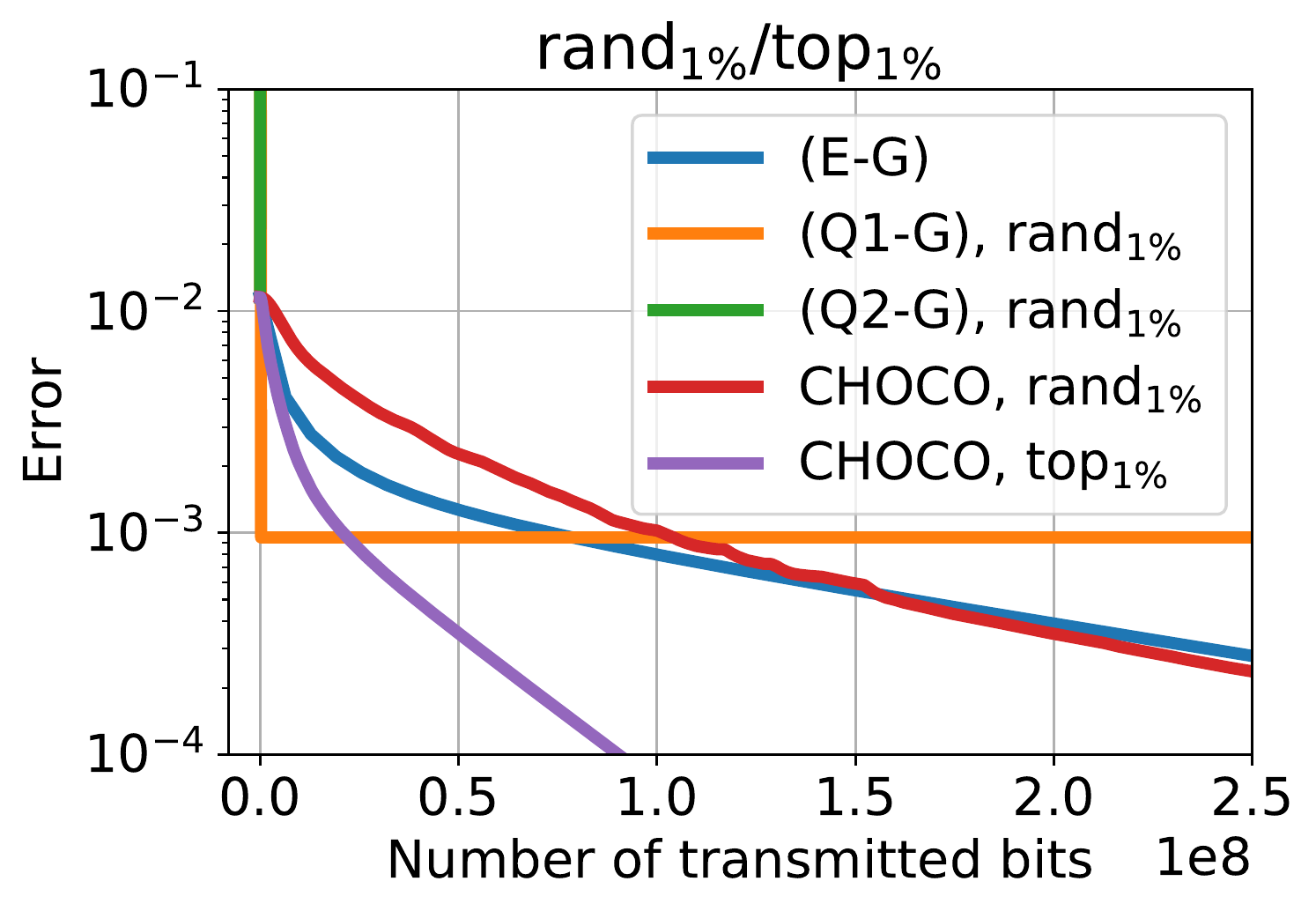}
		\vspace{-2mm}
		\caption{Average consensus on the ring topology with $n=25$ nodes, $d=2000$ coordinates and ($\operatorname{rand}_{1\%}$) and ($\operatorname{top}_{1\%}$) compression}\label{fig:average_random_20}
	\end{minipage}
	
\end{figure}
\begin{figure}

\end{figure}

\begin{table}[t]
	\centering
	\begin{minipage}{.49\linewidth}
	\vspace{4mm}
	\centering
	\begin{tabular}{l|l|l|l}
		dataset & $m$ & $d$  & density \\ \hline
		epsilon & $400000$ & $2000$ & $100\%$ \\
		rcv1 & $20242$ & $47236$ & $0.15\%$
	\end{tabular}
	\caption{Size $(m, d)$ and density of the datasets.}
	\label{tab:datasets}
	\end{minipage}
	\begin{minipage}{.5\linewidth}
	\centering
	\begin{tabular}{l|l}
		experiment & $\gamma$ \\ \hline
		\textsc{Choco}, $(\operatorname{qsgd}_{256})$ & 1 \\
		\textsc{Choco}, ($\operatorname{rand}_{1\%}$)  & 0.011 \\
		\textsc{Choco}, ($\operatorname{top}_{1\%}$)  & 0.046
	\end{tabular}
	\caption{Tuned stepsizes $\gamma$ for averaging in Figs.\ \ref{fig:average_qsgd_8}--\ \ref{fig:average_random_20}.}
	\label{tab:stepsize}
	\end{minipage}
	
	\vspace{2mm}

	\begin{tabular}{l|c|c|c||c|c|c}
		 &  \multicolumn{3}{c||}{$epsilon$} &\multicolumn{3}{c}{$rcv1$}\\\hline
		algorithm & $a$ &$b$ & $\gamma$ & $a$ &$b$ & $\gamma$ \\ \hline\hline
		\textsc{Plain} & 0.1 & $d$ & - &      1 & 1 & - \\
		\textsc{Choco}, $(\operatorname{qsgd}_{16})$ & $0.1$ & $d$ &   0.34 &       1  & 1  &   0.078\\
	\textsc{Choco}, ($\operatorname{rand}_{1\%}$)  & 0.1 & $d$ & 0.01 &    1&    1& 0.016\\
	\textsc{Choco}, ($\operatorname{top}_{1\%}$)  & 0.1 &$d$ &  0.04  &   1 & 1 & 0.04\\
	\textsc{DCD}, ($\operatorname{rand}_{1\%}$)  & $10^{-15}$ & $d$ & - &    $10^{-10}$  & $d$ & - \\
		\textsc{DCD}, $(\operatorname{qsgd}_{16})$ & 0.01 & $d$ & - &  $ 10^{-10}$   & $d$ & - \\
	\textsc{ECD}, ($\operatorname{rand}_{1\%}$)  & $10^{-10}$ & $d$ & - &    $10^{-10} $ & $d$ & - \\
	\textsc{ECD}, ($\operatorname{qsgd}_{16}$)  & $10^{-12}$ & $d$ & - &   $ 10^{-10}$  & $d$ & -
	\end{tabular}
	\caption{Parameters for the SGD learning rate $\eta_t=\frac{ma}{t+b}$ and consensus learning $\gamma$ used in the experiments in Figs.\ \ref{fig:sgd_random_20}--\ref{fig:sgd_qsgd}. Parameters where tuned separately for each algorithm. Tuning details can be found in Appendix~\ref{sect:parameters_search_details}. The \textsc{ECD} and \textsc{DCD} stepsizes are small because the algorithms were observed to diverge for larger choices. }
	\label{tab:learning_rates_sgd}
\end{table}
We compare the performance of the gossip schemes \eqref{eq:boyd} (exact communication), \eqref{eq:consensus_first}, \eqref{eq:consensus_second} (both with unbiased compression), and our scheme~\eqref{eq:ours} in Figure~\ref{fig:average_qsgd_8} for the $(\operatorname{qsgd}_{256})$ compression scheme and in Figure~\ref{fig:average_random_20} for the random ($\operatorname{rand}_{1\%}$) compression scheme. In addition, we also depict the performance of \algcons with biased ($\operatorname{top}_{1\%}$) compression. We use ring topology with uniformly averaging mixing matrix W as in Figure \ref{fig:ring_and_torus_topology}, left. The stepsizes $\gamma$ that were used for \algcons are listed in the Table~\ref{tab:stepsize}.
We consider here the consensus problem~\eqref{def:consensus} with data $\xx_i^{(0)} \in \R^d$ on the $i$-machine was chosen to be the $i$\nobreakdash-th vector in the $epsilon$ dataset. We depict the errors $\tfrac{1}{n}\sum_{i=1}^n \bigl\|\xx_i^{(t)} - \overline{\xx}\bigr\|^2$. %

The proposed scheme~\eqref{eq:ours} with 8 bit quantization $(\operatorname{qsgd}_{256})$   converges with the same rate as~\eqref{eq:boyd} that uses exact communications (Fig.\ \ref{fig:average_qsgd_8}, left), while it requires much less data to be transmitted (Fig.\ \ref{fig:average_qsgd_8}, right). The schemes~\eqref{eq:consensus_first} and~ \eqref{eq:consensus_second} can do not converge and reach only accuracies of $10^{-4}$ --$10^{-5}$. The scheme~\eqref{eq:consensus_first} even starts to diverge, because the quantization error becomes larger than the optimization error. 

With sparsified communication ($\operatorname{rand}_{1\%}$), i.e. transmitting only $1\%$ of all the coordinates, the scheme \eqref{eq:consensus_first} quickly zeros out all the coordinates, and \eqref{eq:consensus_second} diverges because quantization error is too large already from the first step~(Fig.~\ref{fig:average_random_20}).  
\algcons proves to be more robust and converges. The observed rate matches with the theoretical findings, as we expect the scheme with factor $100\times$ compression to be $100\times$ slower than~\eqref{eq:boyd} without compression. In terms of total data transmitted, both schemes converge at the same speed (Fig.\ \ref{fig:average_random_20}, right).
We also see that ($\operatorname{rand}_{1\%}$) sparsification can give additional gains and comes out as the most data-efficient method in these experiments.

\subsection{Decentralized SGD}
\label{sec:expSGD}

\newlength{\smallfigwidth}
\setlength{\smallfigwidth}{0.32\textwidth}
\begin{figure*}[t]
	\centering
	\begin{minipage}{\textwidth}
		\centering
		\includegraphics[width=\smallfigwidth]{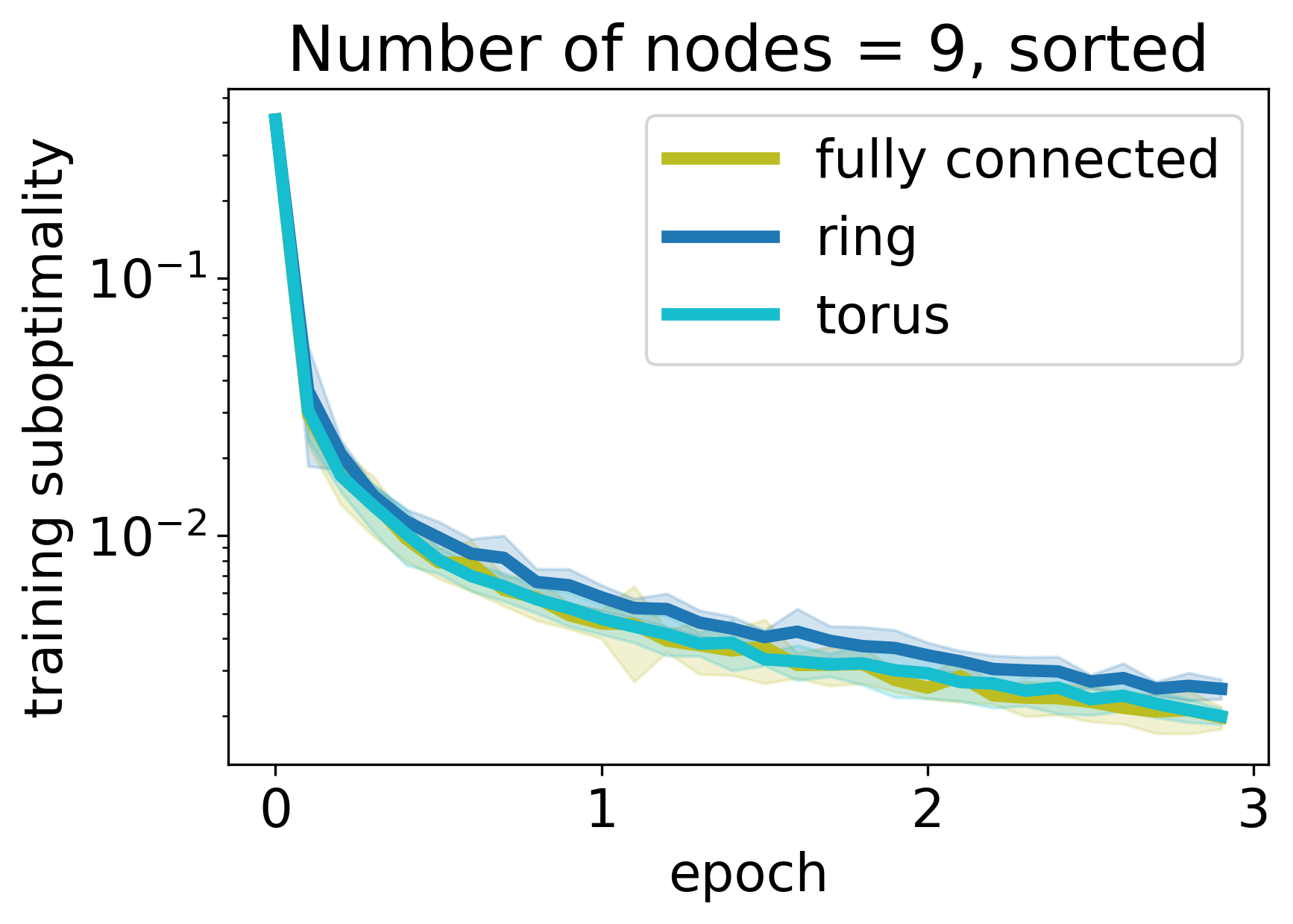}
		\hfill
		\includegraphics[width=\smallfigwidth]{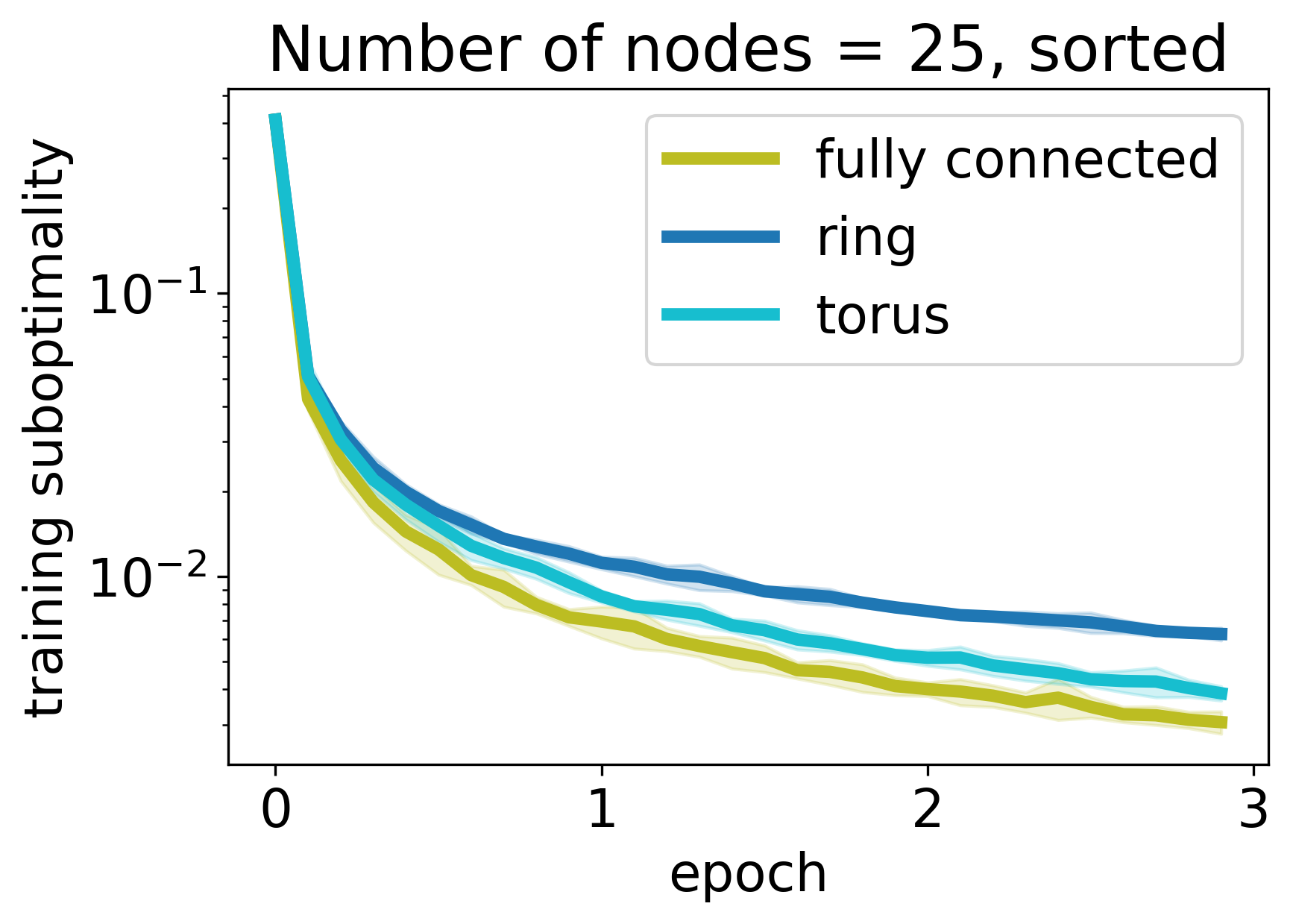}
		\hfill
		\includegraphics[width=\smallfigwidth]{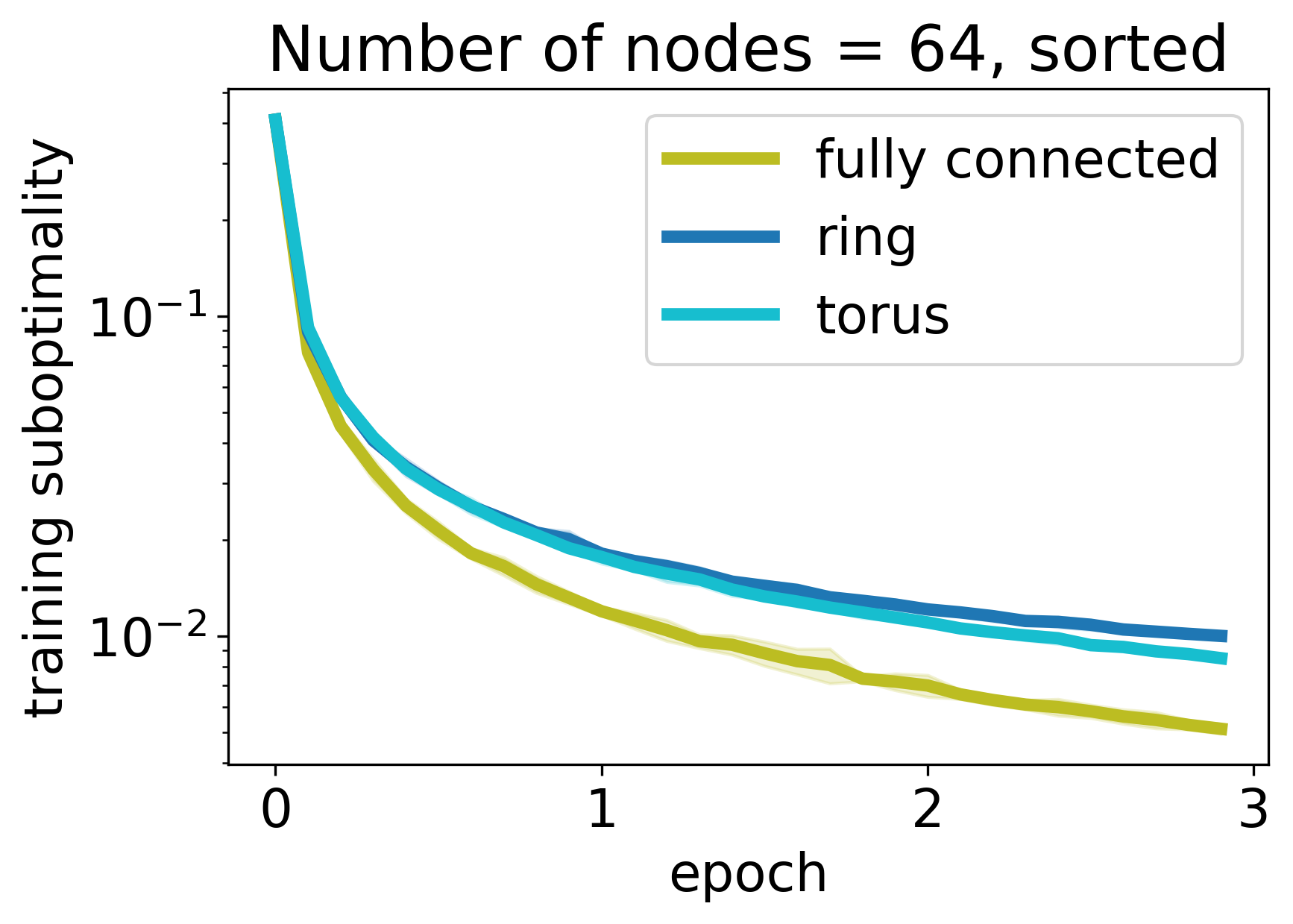}
		\vspace{-2mm}
		\caption{Performance of Algorithm~\ref{alg:all-to-all} on ring, torus and fully connected topologies for $n\in \{9,25,64\}$ nodes. Here we consider the \emph{sorted} setting, whilst the performance for randomly shuffled data is depicted in the Appendix~\ref{sect:additional_experiments}.}%
		\label{fig:topologies_sorted}
	\end{minipage}%
\end{figure*}

\begin{figure}[h!]
	\centering
	\begin{minipage}{0.8\textwidth}
	\vspace{-2mm}
	\includegraphics[width=0.48\linewidth]{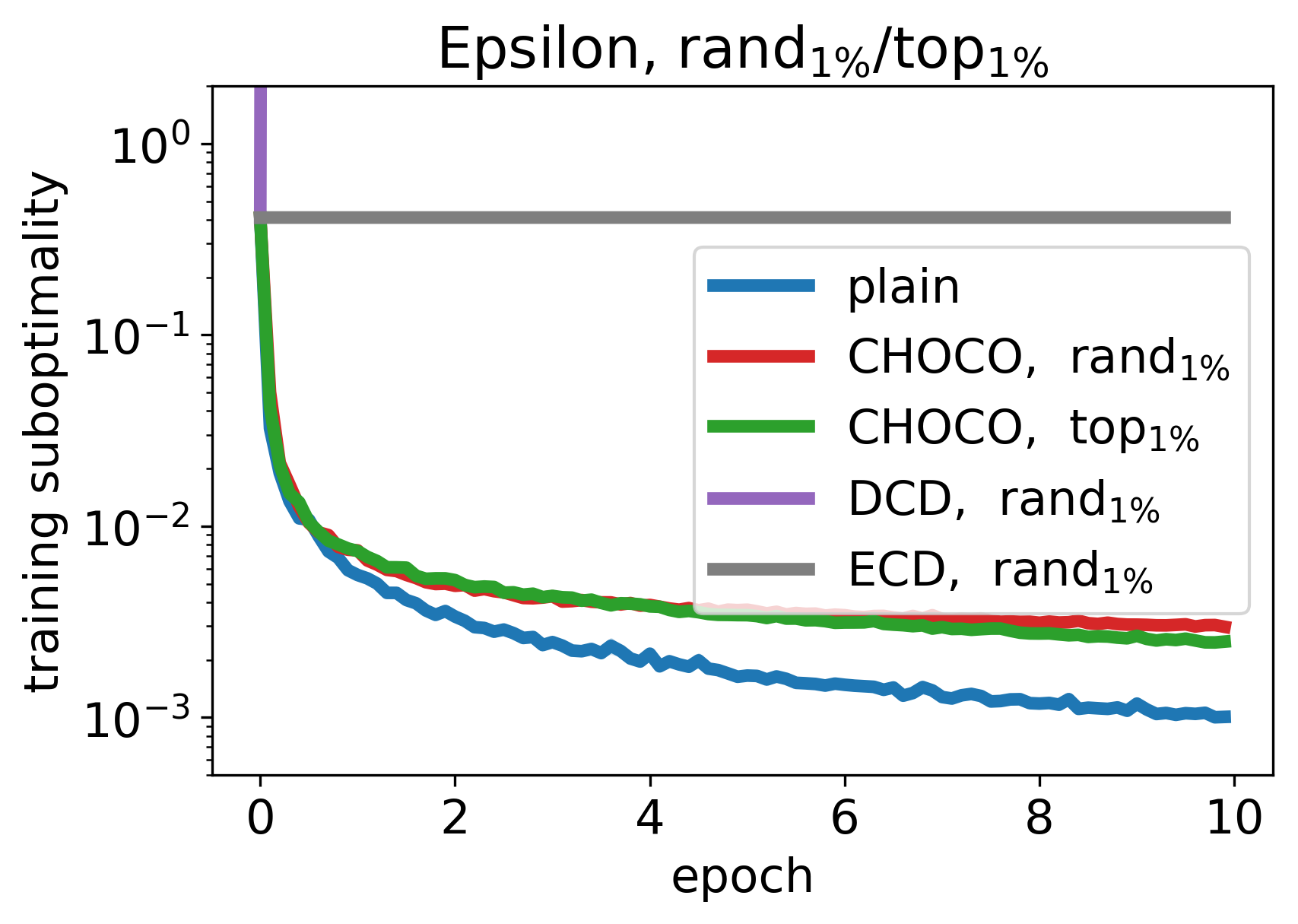}
	\includegraphics[width=0.48\linewidth]{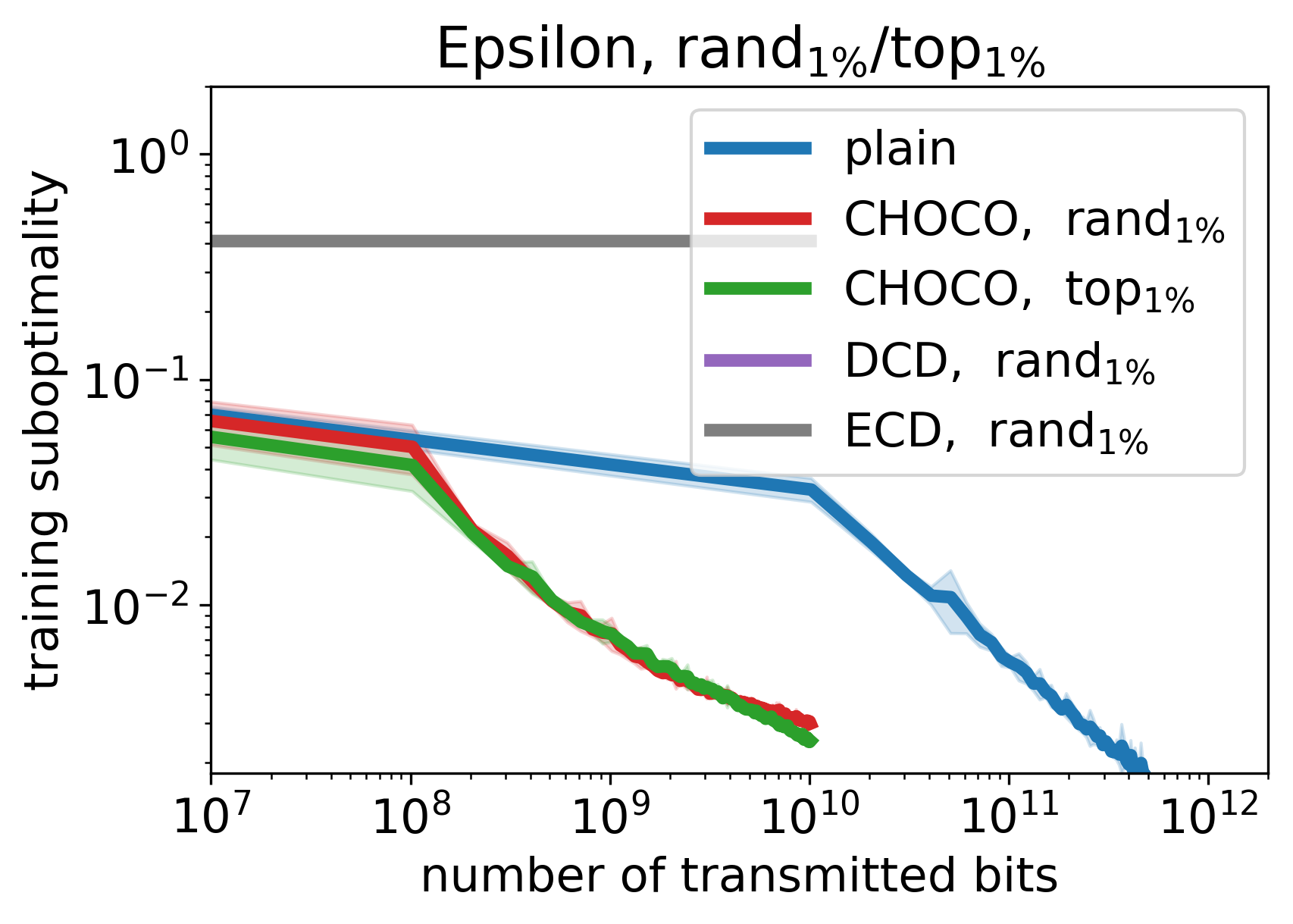}\\
	\includegraphics[width=0.48\linewidth]{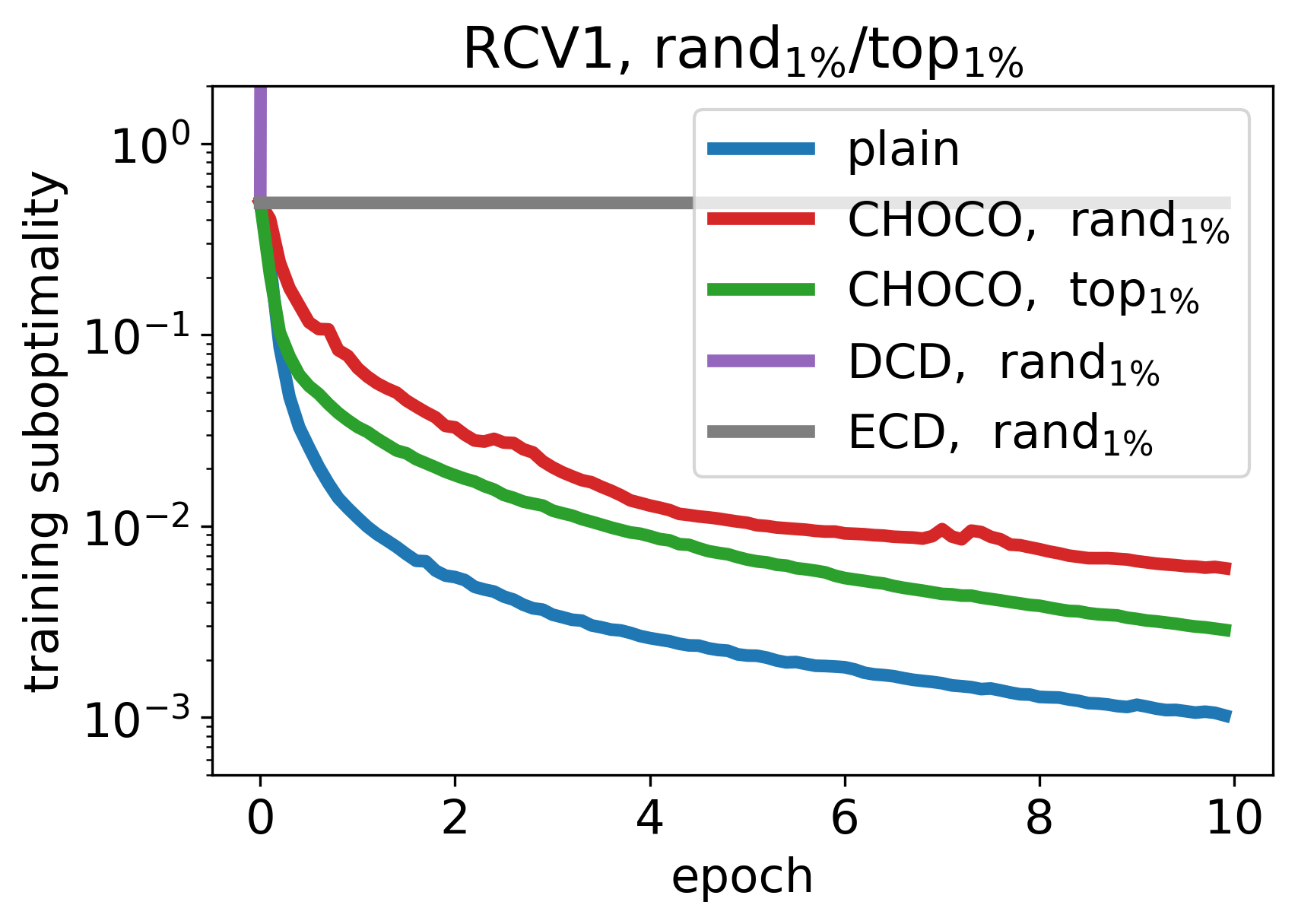}
	\includegraphics[width=0.48\linewidth]{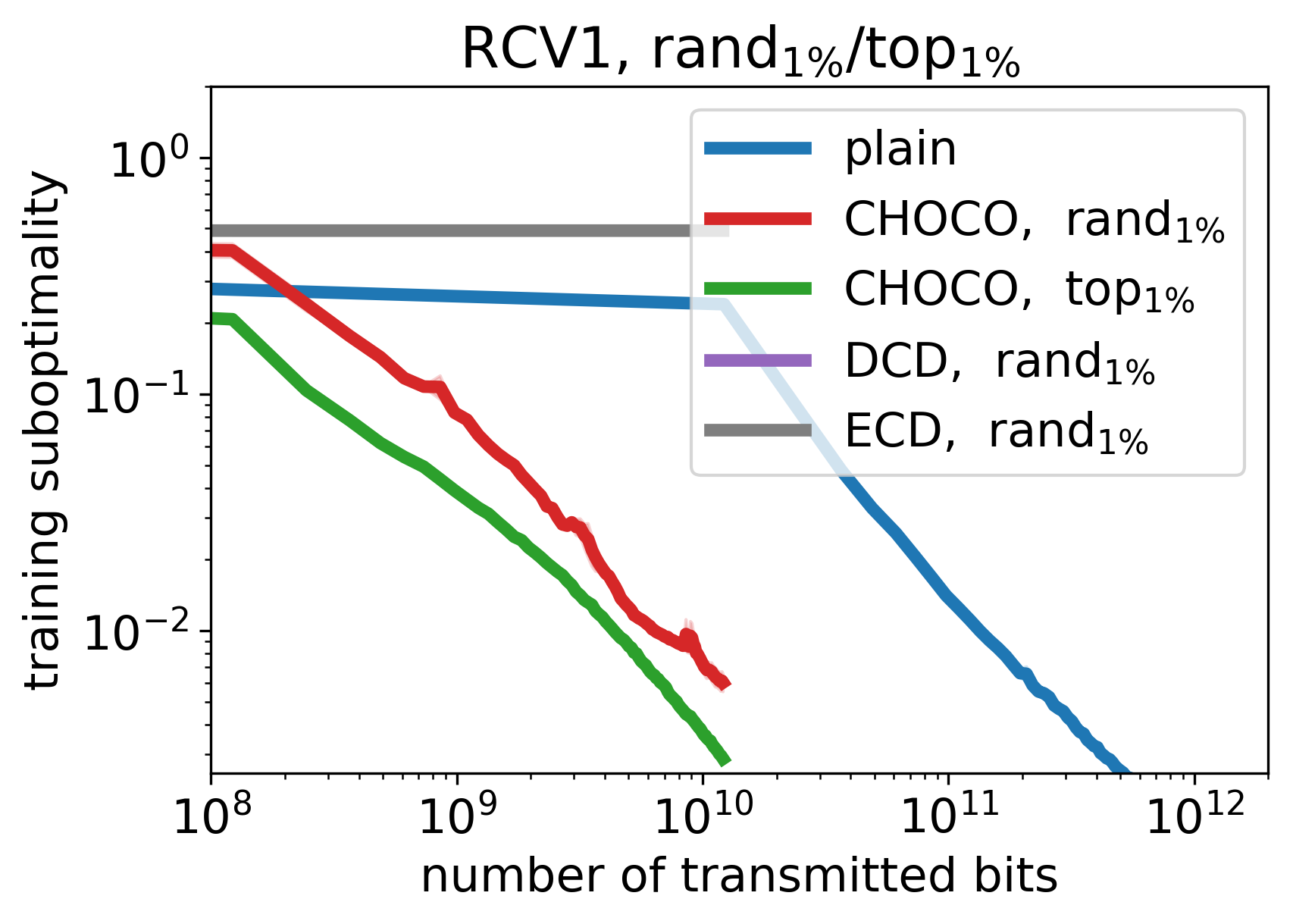}
	\vspace{-2mm}
	\caption{Comparison of Algorithm~\ref{alg:all-to-all} (plain), ECD-SGD, DCD-SGD and \algopt with ($\operatorname{rand}_{1\%}$) sparsification (in addition ($\operatorname{top}_{1\%}$) for \algopt), for $epsilon$ (top) and $rcv1$ (bottom) in terms of iterations (left) and communication cost (right), $n=9$.}
	\label{fig:sgd_random_20}
	\end{minipage}
\end{figure}
\begin{figure}
	\centering
	\begin{minipage}{0.8\textwidth}
		\vspace{-2mm}
			\includegraphics[width=0.48\linewidth]{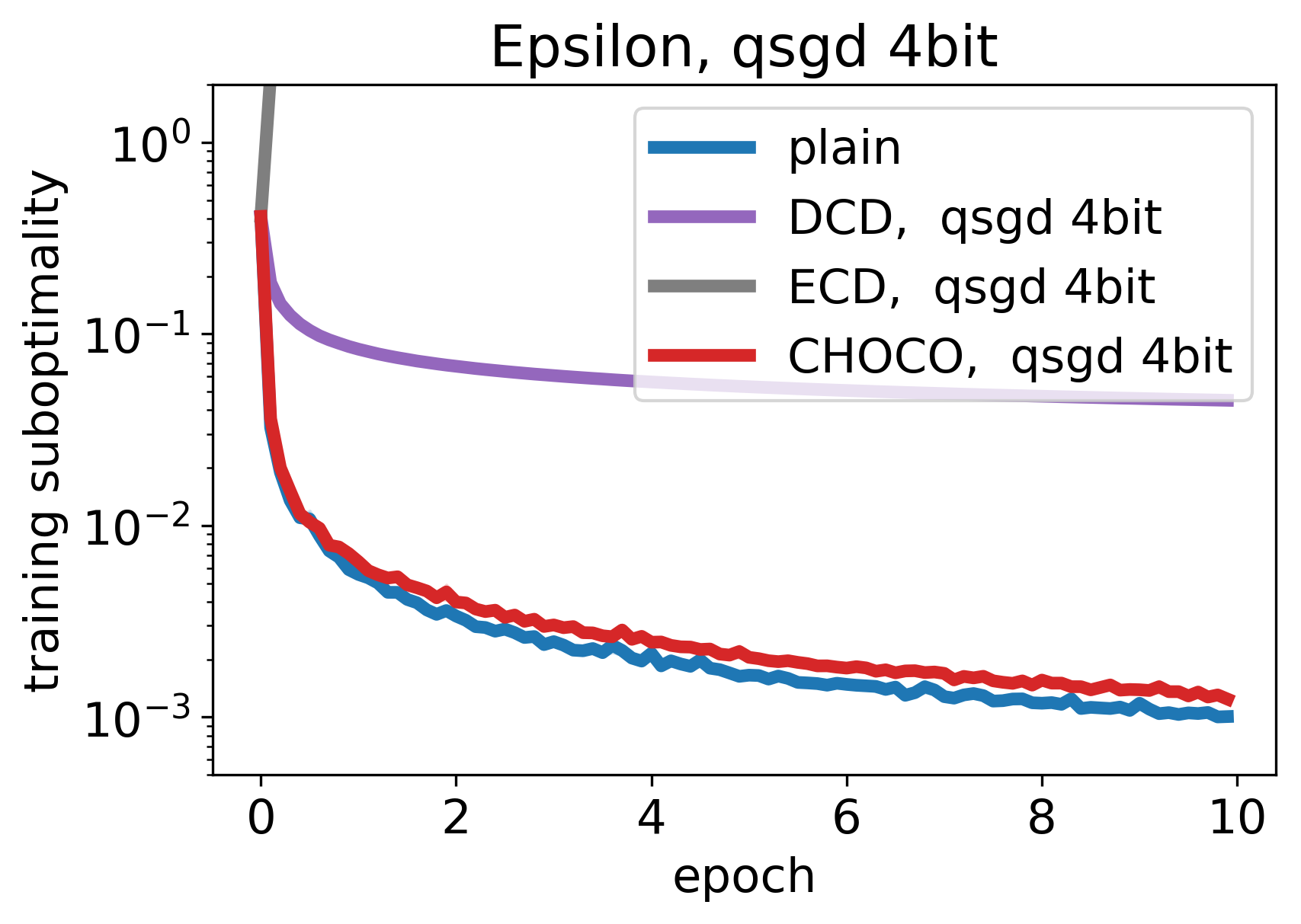}
		\includegraphics[width=0.48\linewidth]{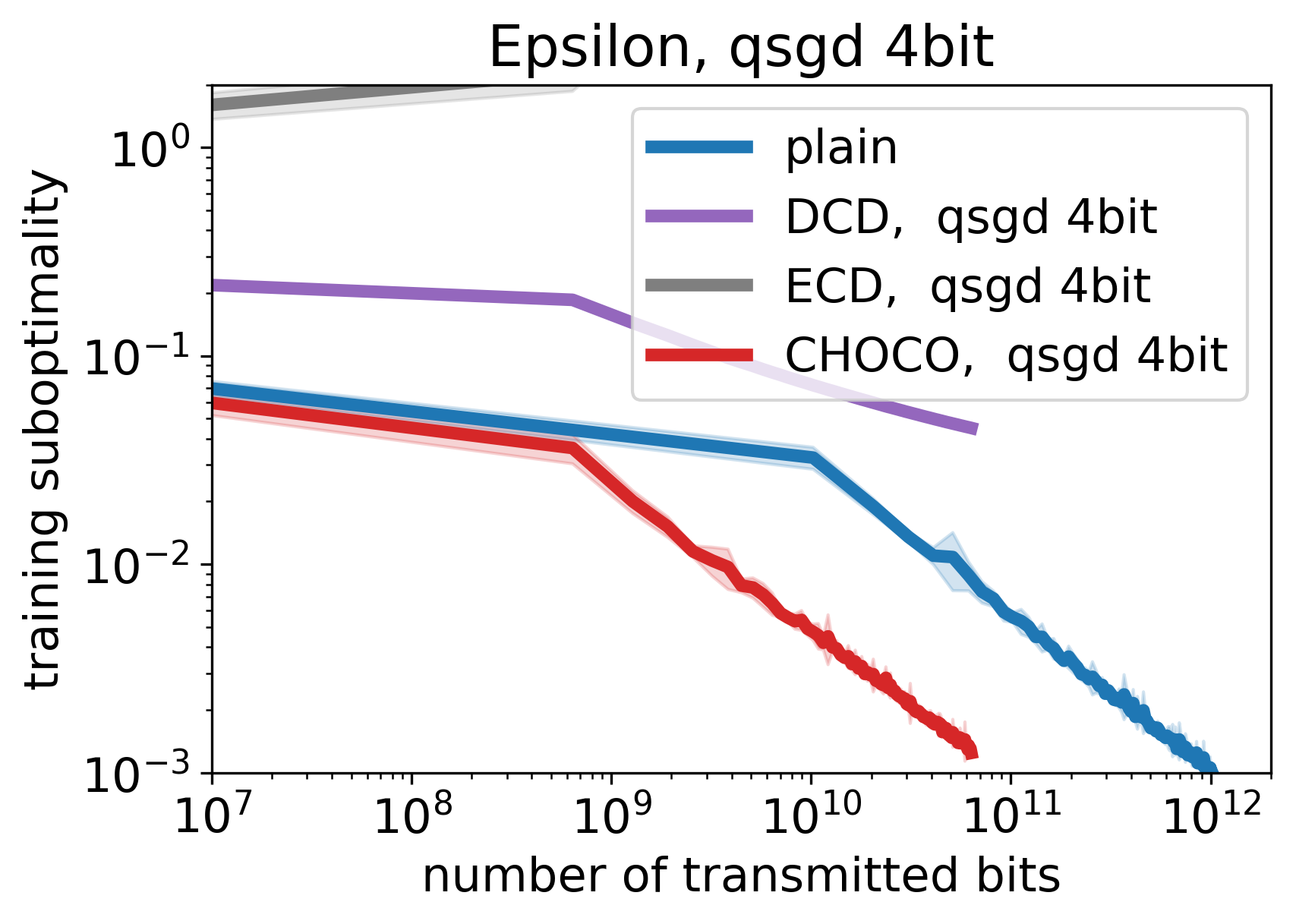}\\
		\includegraphics[width=0.48\linewidth]{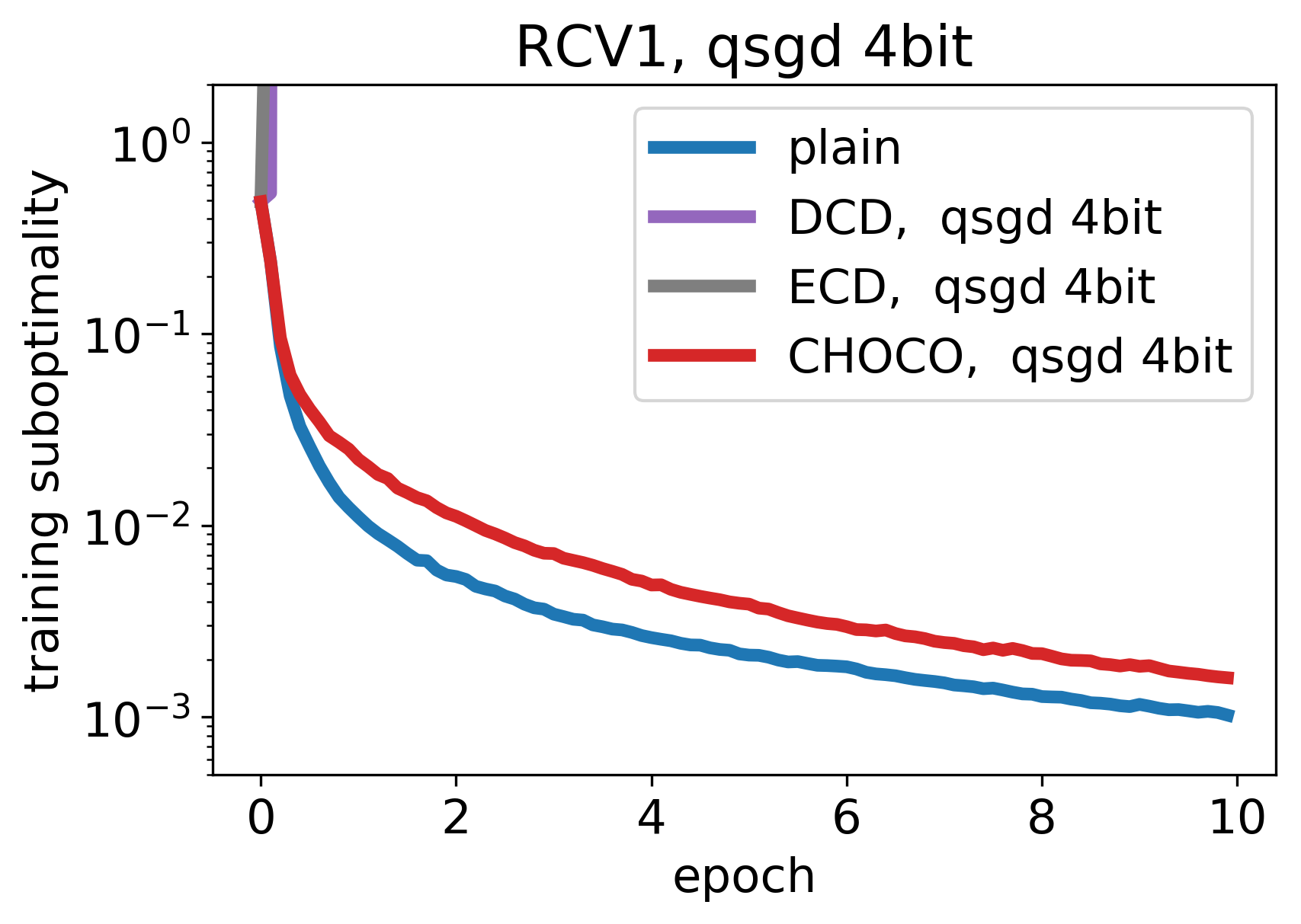}
		\includegraphics[width=0.48\linewidth]{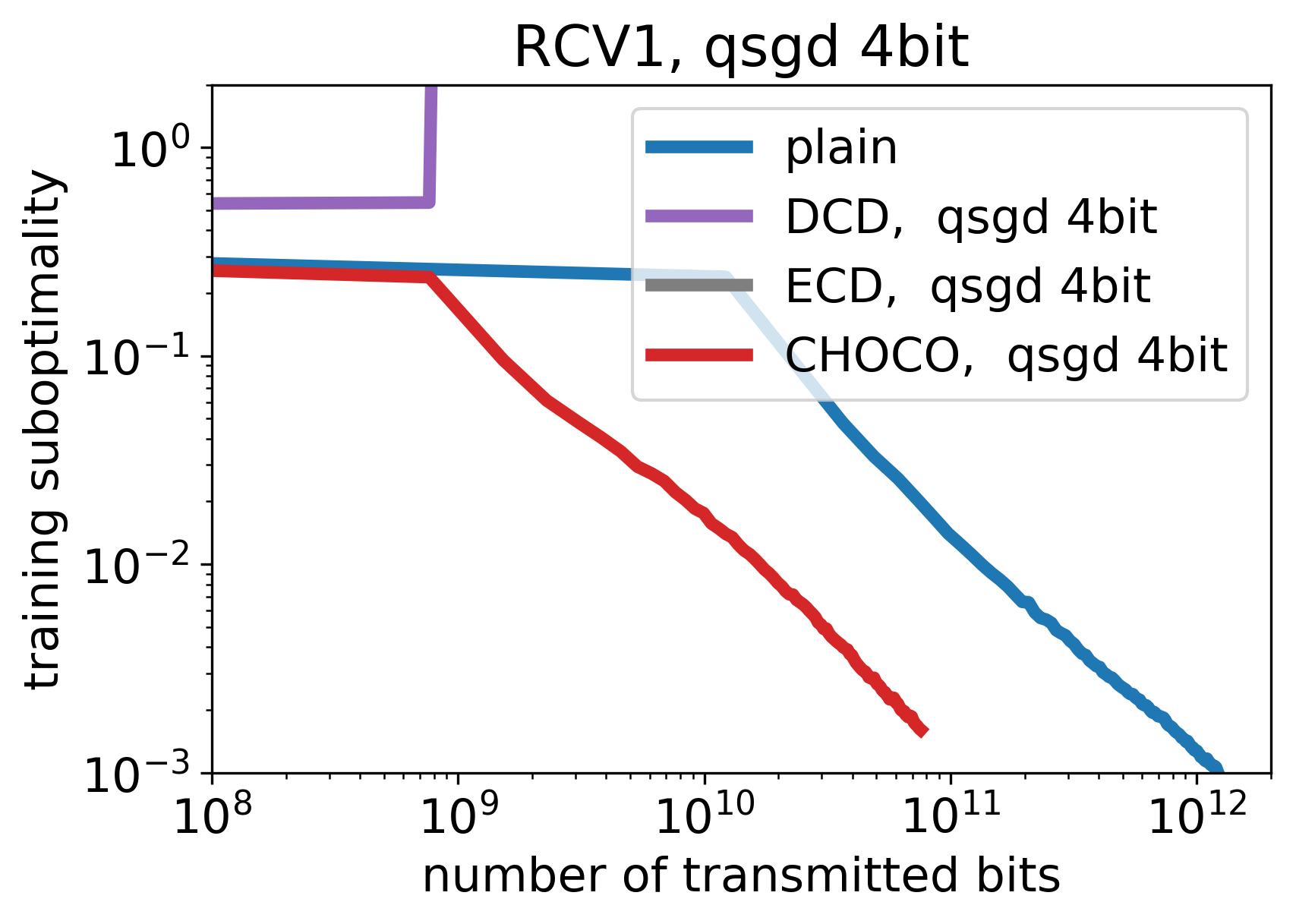}
		\vspace{-2mm}
		\caption{Comparison of Algorithm~\ref{alg:all-to-all} (plain), ECD-SGD, DCD-SGD and \algopt with ($\operatorname{qsgd}_{16}$) quantization, for $epsilon$ (top) and $rcv1$ (bottom) in terms of iterations (left) and communication cost (right), on $n=9$ nodes on a ring topology.}\label{fig:sgd_qsgd}
	\end{minipage}
\end{figure}

We asses the performance of \algopt on logistic regression, defined as $\frac{1}{m} \sum_{j = 1}^m \log(1 + \exp(-b_j \aa_j^\top \xx)) + \frac{1}{2 m}\norm{\xx}^2$, where $\aa_j \in \R^d$ and $b_{j} \in \{-1, 1\}$ are the data samples and $m$ denotes the number of samples in the dataset. We distribute the $m$ data samples evenly among the $n$ workers and consider two settings: (i) \emph{randomly shuffled}, where datapoints are randomly assigned to workers, and the more difficult (ii) \emph{sorted} setting, where each worker only gets data samples just from one class (with the possible exception of one worker that gets two labels assigned). Moreover, we try to make the setting as difficult as possible, meaning that e.g. on the ring topology the machines with the same label form two connected clusters. %
We repeat each experiment three times and depict the mean curve and the area corresponding to one standard deviation. We plot suboptimality, i.e. $f(\overline{\xx}^{(t)}) - f^\star$ (obtained by $\operatorname{LogisticSGD}$ optimizer from scikit-learn \cite{scikit-learn}) versus number of iterations and the number of transmitted bits between workers, which is proportional to the actual running time if communication is a bottleneck. 

\paragraph{Algorithms.} 
As baselines we consider Alg.~\ref{alg:all-to-all} with exact communication (denoted as `plain') and the communication efficient state-of-the-art optimization schemes DCD-SGD and ECD-SGD recently proposed in~\cite{Tang2018:decentralized} (for unbiased quantization operators) and compare them to \algopt. We use decaying stepsize $\eta_t = \frac{ma}{t + b}$ where the parameters $a,b$ are individually tuned for each algorithm and compression scheme, with values given in Table \ref{tab:learning_rates_sgd}.

\paragraph{Impact of Topology.}
In Figure~\ref{fig:topologies_sorted} we depict the performance of the baseline Algorithm~\ref{alg:all-to-all} with exact communication on different topologies (ring, torus and fully-connected; Fig.~\ref{fig:ring_and_torus_topology}) with uniformly averaging mixing matrix $W$. Note that Algorithm \ref{alg:all-to-all} for fully-connected graph corresponds to mini-batch SGD. Increasing the number of workers from $n=9$ to $n=25$ and $n=64$ shows the mild effect of the network topology on the convergence. We observe that the \emph{sorted} setting is more difficult than the \emph{randomly shuffled} setting (see Fig.~\ref{fig:topologies_random_all} in the Appendix~\ref{sect:additional_experiments}), where the convergence behavior remains almost unaffected. In the following we focus on the hardest case, i.e. the ring topology.

\paragraph{Comparison to Baselines.}
In Figures~\ref{fig:sgd_random_20} and~\ref{fig:sgd_qsgd} depict the performance of these algorithms on the ring topology with $n = 9$ nodes for \emph{sorted} data of the $epsilon$ and $rcv1$ datasets. \algopt performs almost as good as the exact Algorithm~\ref{alg:all-to-all} in all situations, but using $100\times$ less communication with ($\operatorname{rand}_{1\%}$) sparsification (Fig.\ \ref{fig:sgd_random_20}, right) and approximately $15\times$ less communication for $(\operatorname{qsgd}_4)$ quantization. The ($\operatorname{top}_{1\%}$) variant performs slightly better than ($\operatorname{rand}_{1\%}$) sparsification. 

\algopt consistently outperforms DCD-SGD in all settings. We also observed that DCD-SGD starts to perform better for larger number of levels $s$ in the $(\operatorname{qsgd}_s)$ in the quantification operator (increasing communication cost). This is consistent with the reporting in~\cite{Tang2018:decentralized} that assumed high precision quantization. %
As a surprise to us, ECD-SGD, which was proposed in~\cite{Tang2018:decentralized} a the preferred alternative over DCD-SGD for less precise quantization operators, always performs worse than DCD-SGD, and often diverges.%

Figures for \emph{randomly shuffled} data and be found in the Appendix~\ref{sect:additional_experiments}. In that case \algopt performs exactly as well as the exact Algorithm~\ref{alg:all-to-all} in all situations.

\paragraph{Conclusion.}
The experiments verify our theoretical findings: \algcons is the first linearly convergent gossip algorithm with quantized communication and \algopt consistently outperforms the baselines for decentralized optimization, reaching almost the same performance as the exact algorithm without communication restrictions while significantly reducing communication cost. 
In view of the striking popularity of SGD as opposed to full-gradient methods for deep-learning, the application of \algopt to decentralized deep learning---an instance of problem \eqref{eq:prob}--- is a promising direction.

\paragraph{Acknowledgments.}
We acknowledge funding from SNSF grant 200021\_175796, as well as a Google Focused Research Award.

\small
\bibliographystyle{icml2019}
\bibliography{../../papers_decentralized}
\normalsize

\appendix
\onecolumn

\section{Basic Identities and Inequalities}
\subsection{Smooth and Strongly Convex Functions}
\label{sec:defsmooth}
\begin{definition}
A differentiable function $f\colon \R^d \to \R$ is $L$-strongly convex for parameter $L \geq 0$ if
\begin{align}
 f(\yy) &\leq f(\xx) + \lin{\nabla f(\xx),\yy-\xx} + \frac{L}{2}\norm{\yy-\xx}^2\,, & &\forall \xx,\yy \in \R^d\,. \label{def:smooth}
\end{align}
\end{definition}

\begin{definition}
A differentiable function $f\colon \R^d \to \R$ is $\mu$-strongly convex for parameter $\mu \geq 0$ if
\begin{align}
 f(\yy) &\geq f(\xx) + \lin{\nabla f(\xx),\yy-\xx} + \frac{\mu}{2}\norm{\yy-\xx}^2\,, & &\forall \xx,\yy \in \R^d\,. \label{def:strongconvex}
\end{align}
\end{definition}

\begin{remark}\label{remark:smoothness}
	If $f$ is $L$-smooth with minimizer $\xx^\star$ s.t $\nabla f(\xx^\star) = \0$, then
	\begin{equation}\label{eq:l-smooth}
	\norm{\nabla f(\xx)}^2 = \norm{\nabla f(\xx) - \nabla f(\xx^\star)}^2 \leq 2L \left(f(\xx) - f(\xx^\star)\right) \,.
	\end{equation}
\end{remark}

\subsection{Vector and Matrix Inequalities}
\begin{remark}\label{rem:frobenious_norm_of_matrix_mult}
	For $A\in \R^{d\times n}$, $B\in \R^{n\times n}$
	\begin{align}\label{eq:frob_norm_of_multiplication}
	\norm{AB}_F \leq \norm{A}_F \norm{B}_2 \,.
	\end{align}
\end{remark}

\begin{remark}\label{remark:norm_of_sum}
	For arbitrary set of $n$ vectors $\{\aa_i\}_{i = 1}^n$, $\aa_i \in \R^d$
	\begin{equation}\label{eq:norm_of_sum}
	\norm{\sum_{i = 1}^n \aa_i}^2 \leq n \sum_{i = 1}^n \norm{\aa_i}^2 \,.
	\end{equation}
\end{remark}
\begin{remark}\label{remark:scal_product}
	For given two vectors $\aa, \bb \in \R^d$
	\begin{align}\label{eq:scal_product}
	&2\lin{\aa, \bb} \leq \gamma \norm{\aa}^2 + \gamma^{-1}\norm{\bb}^2\,, & &\forall \gamma > 0 \,.
	\end{align}
\end{remark}
\begin{remark}\label{remark:norm_of_sum_of_two}
	For given two vectors $\aa, \bb \in \R^d$ %
	\begin{align}\label{eq:norm_of_sum_of_two}
	\norm{\aa + \bb}^2 \leq (1 + \alpha)\norm{\aa}^2 + (1 + \alpha^{-1})\norm{\bb}^2,\,\, & &\forall \alpha > 0\,.
	\end{align}
	This inequality also holds for the sum of two matrices $A,B \in \R^{n \times d}$ in Frobenius norm.
\end{remark}

\subsection{Implications of the bounded gradient and bounded variance assumption}

\begin{remark}\label{rem:expectation_of_gradient_squared}
	If $F_i: \R^d \times \Omega \to \R, i = 1,\dots,n$ are convex functions with $\EE{\xi}{\norm{\nabla F_i(\xx, \xi)}}^2 \leq G^2$, $\partial F(X, \xi) = \left[\nabla F_1(\xx, \xi_1), \dots, \nabla F_n(\xx, \xi_n) \right]$
	\begin{align*}
	\EE{\xi_1, \dots, \xi_n}{\norm{\partial F(X, \xi)}_F^2 \leq n G^2}, \,\, & & \forall X\,.
	\end{align*}
\end{remark}
\begin{remark}[Mini-batch variance]\label{rem:variance} If for functions $f_i$, $F_i$ defined in \eqref{eq:func}  $\EE{\xi}{\norm{\nabla F_i(\xx, \xi) - \nabla f_i(\xx)}^2}\leq \sigma_i^2, i \in [n]$, then
	$$\EE{\xi_1^{(t)}, \dots, \xi_n^{(t)}}{\norm{\frac{1}{n}\sum_{j = 1}^{n} \left(\nabla f_j(\xx_j^{(t)}) - \nabla F_j (\xx_j^{(t)}, \xi_j^{(t)})\right)}}^2\leq \dfrac{\overline{\sigma}^2}{n},$$\\
	where $\overline{\sigma}^2 = \frac{\sum_{i = 1}^{n} \sigma_i^2}{n}$.
\end{remark}
\begin{proof} This follows from
	\begin{align*}
	&\E{\norm{\frac{1}{n}\sum_{j = 1}^{n} Y_j}}^2= \dfrac{1}{n^2} \left(\sum_{j = 1}^n\E{\norm{Y_j}}^2 + \sum_{i \neq j} \E{\lin{Y_i, Y_j}}\right) = \dfrac{1}{n^2} \sum_{j = 1}^n\E{\norm{Y_j}}^2 \leq 
	\frac{1}{n^2}\sum_{j = 1}^{n} \sigma_j^2 = \dfrac{\overline{\sigma}^2}{n}
	\end{align*}
	for $Y_j = f_j(\xx_j^{(t)}) - \nabla F_j (\xx_j^{(t)}, \xi_j^{(t)})$. Expectation of scalar product is equal to zero because $\xi_i$ is independent of $\xi_j$ since $i\neq j$.
\end{proof}

\section{Consensus in Matrix notation}

In the proofs in the next section we will use the matrix notation, as already introduced in the main text. We define
\begin{align}\label{eq:notation}
X^{(t)} &:= \left[ \xx_1^{(t)},\dots, \xx_n^{(t)}\right] \in \R^{d\times n}, & &Q^{(t)} := \left[\qq_1^{(t)}, \dots, \qq_n^{(t)}\right] \in \R^{d\times n}, & &\hat{X}^{(t)} := \left[\hat{\xx}_1^{(t)}, \dots, \hat{\xx}_n^{(t)}\right] \in \R^{d\times n}\,.
\end{align}
Then using matrix notation we can rewrite Algorithm~\ref{alg:consensus} as

\setcounter{algorithm}{0}
\begin{algorithm}[H]
	\caption{\textsc{\algcons in matrix notation}}\label{alg:consensus_matrix}
	\begin{algorithmic}[1]
		\INPUT{: $X^{(0)}$, $\gamma$, $W$.}
		\STATE Initialize: $\hat{X}^{(0)} = 0$
		\FOR{$t$ \textbf{in} $0\dots T-1$}
		\STATE $Q^{(t)} = Q(X^{(t)} - \hat{X}^{(t)})$
		\STATE $\hat{X}^{(t + 1)} = \hat{X}^{(t)} + Q^{(t)}$
		\STATE $X^{(t + 1)} = X^{(t)} + \gamma \hat{X}^{(t + 1)}\left(W - I\right)$
		\ENDFOR
	\end{algorithmic}
\end{algorithm}
\begin{remark}
	Note that since every worker $i$ for each neighbor $j: \{i, j\} \in E$ stores $\hat{\xx}_j$, the proper notation for $\hat{\xx}$ would be to use $\hat{\xx}_{ij}$ instead. We simplified it using the property that if $\hat{\xx}_{ij}^{(0)} = \hat{\xx}_{kj}^{(0)}$, $\forall i, k: \{i,j\} \in E$ and $\{k, j\} \in E$, then they are equal at all timesteps $\hat{\xx}_{ij}^{(t)} = \hat{\xx}_{kj}^{(t)}$, $\forall t \geq 0$.
\end{remark}
\begin{remark}
	The results of Theorem~\ref{th:decentralized_sgd} and~\ref{th:sgd_general} also hold for arbitrary initialized $\hat{X}^{(0)}$ with the constraint that $\forall j$ all the neighbors of the node $j$ initialized with the same $\hat{\xx}_i$, i.e. using extended notation $\hat{\xx}_{ij}^{(0)} = \hat{\xx}_{kj}^{(0)}$, $\forall i, k: \{i,j\} \in E$ and $\{k, j\} \in E$.
\end{remark}

\subsection{Useful Facts}
\begin{remark}\label{rem:average}
	Let $X^{(t)} = \left[\xx_1^{(t)}, \dots, \xx_n^{(t)}\right] \in \R^{d\times n}$ and $\overline{X}^{(t)} = \left[\overline{\xx}^{(t)}, \dots, \overline{\xx}^{(t)}\right]\in \R^{d\times n}$, for  $\overline{\xx}^{(t)} = \frac{1}{n}\sum_{i = 1}^n \xx_i^{(t)}$, then because $W$ is doubly stochastic
	\begin{align}\label{eq:average}
	\overline{X}^{(t)} = X^{(t)}\frac{1}{n}\1\1^\top, && 	\overline{X}^{(t)}W = \overline{X}^{(t)}\,.
	\end{align}
\end{remark}
\begin{remark}
	The average $\overline{X}^{(t)} = \left[\overline{\xx}^{(t)}, \dots, \overline{\xx}^{(t)}\right]\in \R^{d\times n}$ during iterates of the Algorithm~\ref{alg:consensus} is preserved, i.e. 
	\begin{align}
	\overline{X}^{(t)} = \overline{X}^{(0)}, \quad \forall t,
	\end{align} 
	where $\overline{X}^{(t)} = \left[\overline{\xx}^{(t)} \dots, \overline{\xx}^{(t)}\right] \in \R^{d\times n}$.
\end{remark}
\begin{proof}
	\begin{align*}
	\overline{X}^{(t + 1)} =  \overline{X}^{(t)} + \gamma \hat{X}^{(t)} \left(W - I\right) \frac{\1\1^\top}{n} = \overline{X}^{(t)},
	\end{align*}
	because $W \frac{\1\1^\top}{n} = I$ since $W$ is doubly stochastic. 
\end{proof}
\begin{lemma}\label{lem:mixing_matrix}
	For $W$ satisfying Definition~\ref{def:W}, i.e. W is symmetric doubly stochastic matrix with second largest eigenvalue $1-\delta = |\lambda_2(W)| < 1$
	\begin{align}\label{eq:mixing_matrix}
	\norm{W^k - \frac{1}{n}\1\1^\top}_2 \leq (1-\delta)^{k}\,.
	\end{align}
\end{lemma}
\begin{proof}
	Let $U\Lambda U^\top$ be SVD-decomposition of $W$, then $W^k = U \Lambda^k U^\top$
	Because of the stochastic property of $W$ its first eigenvector is $u_1 = \frac{1}{\sqrt{n}}\1$.
	\begin{align*}
	U \begin{pmatrix}
	1 & 0 &\dots & 0\\
	0 & 0 & \dots & 0\\
	\dots\\
	0 & 0 & \dots & 0\\
	\end{pmatrix} U^\top = u_1 u_1^\top = \dfrac{1}{n} \1\1^\top
	\end{align*}
	Hence,
	\begin{align*}
	\norm{W^k - \frac{1}{n}\1\1^\top}_2 &= \norm{U\Lambda^kU^\top - U \begin{pmatrix}
		1 & 0 &\dots & 0\\
		0 & 0 & \dots & 0\\
		\dots\\
		0 & 0 & \dots & 0\\
		\end{pmatrix} U^\top}_2 = \norm{\Lambda^k- \begin{pmatrix}
		1 & 0 &\dots & 0\\
		0 & 0 & \dots & 0\\
		\dots\\
		0 & 0 & \dots & 0\\
		\end{pmatrix}}_2 = (1-\delta)^{k}. \qedhere
	\end{align*}
\end{proof}

\section{Proof of Theorem~\ref{th:consensus}---Convergence of \algcons}

\begin{lemma}\label{lem:avg_gen_first_part}
	Let $X^{(t)}, \hat{X}^{(t)} \in \R^{d \times n}$, $\overline{X} = [\overline{\xx}, \dots, \overline{\xx}]$ for average $\overline{\xx} = \frac{1}{n} X^{(t)}\1 \in \R^d$ and let $X^{(t+1)} = X^{(t)} + \gamma \hat{X}^{(t)}(W-I) \in \R^{d \times n}$ be defined as in Algorithm~\ref{alg:consensus} with stepsize $\gamma \geq 0$ and mixing matrix $W \in \R^{n \times n}$ as in Definition~\ref{def:W}. Then 
	\begin{align*}
	\norm{X^{(t + 1)} - \overline{X}}^2_F \leq (1 - \delta\gamma)^2 (1 + \alpha_1) \norm{X^{(t)} - \overline{X}}_F^2 + \gamma^2 (1 + \alpha_1^{-1})\beta^2 \norm{\hat{X}^{(t + 1)} - X^{(t)}}_F^2,\,\, \forall \alpha_1 > 0\,.
	\end{align*}
	Here $\alpha_1 >0$ is a parameter whose value will be chosen later, $\delta = 1 - |\lambda_2(W)|$ and $\beta = \max_i \{1 - \lambda_i(W)\}$ as defined above.
\end{lemma}
\begin{proof}
By the definition of $X^{(t+1)}$ and the observation $\overline{X}(W-I)=0$, we can write
	\begin{align*}
	\norm{X^{(t + 1)} - \overline{X}}_F^2 &= \norm{X^{(t)}  - \overline{X}+ \gamma \hat{X}^{(t + 1)}(W - I)}^2_F \\
	&= \norm{X^{(t)}  - \overline{X} + \gamma\left(X^{(t)} - \overline{X}\right)(W - I) + \gamma \left(\hat{X}^{(t + 1)} - X^{(t)}\right)(W - I)}_F^2 \\
	&= \norm{\left(X^{(t)} - \overline{X}\right)((1 - \gamma)I + \gamma W) + \gamma\left(\hat{X}^{(t + 1)} - X^{(t)}\right)(W - I)}^2_F \\
	&\stackrel{\eqref{eq:norm_of_sum_of_two}}{\leq} (1 + \alpha_1)\norm{\left(X^{(t)} - \overline{X}\right)((1 - \gamma)I + \gamma W)}_F^2 +(1 + \alpha_1^{-1}) \norm{\gamma\left(\hat{X}^{(t + 1)} - X^{(t)}\right)(W - I)}^2_F \\
	&\stackrel{\text{\eqref{eq:frob_norm_of_multiplication}}}{\leq} (1 + \alpha_1)\norm{\left(X^{(t)} - \overline{X}\right)((1 - \gamma)I + \gamma W)}^2_F + (1 + \alpha_1^{-1}) \gamma^2 \norm{W-I}_2^2 \cdot \norm{\hat{X}^{(t + 1)} - X^{(t)}}^2_F \,.
	\end{align*}
	Let's estimate the first term
	\begin{align*}
	\norm{\left(X^{(t)} - \overline{X}\right)((1 - \gamma)I + \gamma W)}_F&\leq (1 - \gamma)\norm{X^{(t)} - \overline{X}}_F + \gamma \norm{\left(X^{(t)} - \overline{X}\right)W}_F \\
	&\stackrel{\text{\eqref{eq:average}}}{=} (1 - \gamma)\norm{X^{(t)} - \overline{X}}_F + \gamma \norm{\left(X^{(t)} - \overline{X}\right)\left(W - \1\1^\top/n\right)}_F \\%\iseb{\text{show that average is preserved}} \\
	&\stackrel{\text{\eqref{eq:mixing_matrix}, \eqref{eq:frob_norm_of_multiplication} }}{\leq} (1 - \gamma\delta)\norm{X^{(t)} - \overline{X}}_F
	\end{align*}
	where we used $(X^{(t)}-\overline{X}) \1 \1^\top /n = 0$, by definition of $\overline{X}$, in the second line.
	Putting this together gives us the statement of the lemma. %
\end{proof}

\begin{lemma}\label{lem:avg_gen_second_part}
   Let $X^{(t)}, \hat{X}^{(t)} \in \R^{d \times n}$, $\overline{X} = [\overline{\xx}, \dots, \overline{\xx}]$ for average $\overline{\xx} = \frac{1}{n} X^{(t)} \1  \in \R^d$ and let $X^{(t+1)} \in \R^{d \times n}$
and $\hat{X}^{(t+2)} \in \R^{d \times n}$  be defined as in Algorithm~\ref{alg:consensus} with stepsize $\gamma \geq 0$, mixing matrix $W \in \R^{n \times n}$ as in Definition~\ref{def:W} and quantization as in Assumption~\ref{assump:q}. Then 
	\begin{align*}
	\EE{Q}{\norm{X^{(t + 1)} - \hat{X}^{(t + 2)}}^2_F} &\leq (1 - \omega)(1 + \gamma \beta)^2(1 + \alpha_2)\norm{X^{(t)} - \hat{X}^{(t + 1)}}^2_F \\
	&+ (1 - \omega) \gamma^2 \beta^2 (1 + \alpha_2^{-1})\norm{X^{(t)} - \overline{X}}_F^2,\,\, &\forall \alpha_2 > 0 \,.
	\end{align*}
	Here $\alpha_2 >0$ is a parameter whose value will be chosen later, $\beta = \max_i \{1 - \lambda_i(W)\}$ as defined above and compression ratio $\omega > 0$.
\end{lemma}
\begin{proof} By the definition of $X^{(t+1)}$ and $\hat{X}^{(t+2)}$ we can write
	\begin{align*}
	\EE{Q}{\norm{X^{(t+ 1)} - \hat{X}^{(t+2)}}^2_F} &= \EE{Q}{\norm{X^{(t + 1)} - \hat{X}^{(t + 1)} - Q(X^{(t + 1)} - \hat{X}^{(t + 1)})}_F^2} \stackrel{\text{\eqref{def:omega}}}{\leq} (1 - \omega) \norm{X^{(t + 1)} - \hat{X}^{(t + 1)}}_F^2\\
	&= (1 - \omega) \norm{X^{(t)} + \gamma \hat{X}^{(t + 1)}(W - I) - \hat{X}^{(t + 1)}}_F^2 \\%\iseb{\text{put \& left of $=$, otherwise there is no space}}\\
	&\stackrel{\text{\eqref{eq:average}}}{=} (1 - \omega) \norm{\left(X^{(t)} - \hat{X}^{(t + 1)}\right)\left( (1 + \gamma) I  - \gamma W\right) + \gamma (W - I) \left({X}^{(t)} - \overline{X}\right)}_F^2\\
	&\stackrel{\text{\eqref{eq:norm_of_sum_of_two}}}{\leq} (1 - \omega) (1 + \alpha_2)\norm{\left(X^{(t)} - \hat{X}^{(t + 1)}\right)\left((1+\gamma)I-\gamma W\right)}_F^2 \\
	&\hspace{10mm}+ (1 - \omega)(1 + \alpha_2^{-1})\norm{\gamma (W - I) \left({X}^{(t)} - \overline{X}\right)}_F^2\\
	&\stackrel{\text{\eqref{eq:frob_norm_of_multiplication}}}{\leq} (1 - \omega) (1 + \gamma \beta)^2(1 + \alpha_2)\norm{X^{(t)} - \hat{X}^{(t + 1)}}_F^2 + (1 - \omega) \gamma^2\beta^2 (1 + \alpha_2^{-1})\norm{X^{(t)} - \overline{X}}_F^2\,,
	\end{align*}
where we used $\norm{I + \gamma(I-W)}_2 = 1 + \gamma \norm{I-W}_2 = 1+ \gamma \beta$ because eigenvalues of $\gamma(I-W)$ are positive.
\end{proof}

\begin{proof}[Proof of Theorem~\ref{th:consensus}]
    As observed in Remark~\ref{rem:average} the averages of the iterates is preserved, i.e. $\overline{X} \equiv X^{(t)}\frac{1}{n}\1\1^\top$ for all $t \geq 0$. By applying the Lemmas~\ref{lem:avg_gen_first_part} and~\ref{lem:avg_gen_second_part} from above we obtain
	\begin{align*}
\EE{Q}{e_{t+1}} &\leq \eta_1 (\gamma) \norm{X^{(t)} - \overline{X}}_F^2
	 + \xi_1(\gamma) \norm{\hat{X}^{(t + 1)} - X^{(t)}}_F^2 \leq \max\{\eta_1(\gamma),\xi_1(\gamma)\} \cdot e_t\,,
	\end{align*}
	where
	\begin{align*}
	 \eta_1(\gamma) &:= (1 - \delta\gamma)^2 (1 + \alpha_1) + (1 - \omega) \gamma^2 \beta^2 (1 + \alpha_2^{-1})\,, \\
	 \xi_1(\gamma) &:=\gamma^2 \beta^2(1 + \alpha_1^{-1}) + (1 - \omega) (1 + \gamma\beta)^2(1 + \alpha_2) \,.
	\end{align*}
	Now, we need to choose the parameters $\alpha_1, \alpha_2$ and stepsize $\gamma$ such as to minimize the factor $\max\{\eta_1(\gamma),\xi_1(\gamma)\}$. Whilst the optimal parameter settings can for instance be obtained using specialized optimization software, we here proceed by showing that for the (suboptimal) choice
	\begin{align}
	\begin{split}
	  \alpha_1 &:= \frac{\gamma \delta}{2},\\% \text{with } p=(1-\sqrt{r}) \\ p = \delta
	  \alpha_2 &:= \frac{\omega}{2}\\%, \text{with } \omega = 1- \sigma \\   \omega = \sigma
	  \gamma^\star &:=  \frac{\delta\omega}{16\delta + \delta^2 + 4\beta^2 + 2\delta\beta^2 - 8 \delta\omega} 
	\end{split} \label{eq:params}
    \end{align}	 
it holds
\begin{align}
   \max\{\eta_1(\gamma^\star),\xi_1(\gamma^\star)\} &\leq  1 - \frac{\delta^2 \omega}{2(16\delta + \delta^2 + 4 \beta^2 + 2 \delta \beta^2 - 8 \delta \omega)}  \,. \label{eq:3435} 
   \end{align}
 The claim of the theorem then follows by observing
 \begin{align}
  1 - \frac{\delta^2 \omega}{2(16\delta + \delta^2 + 4 \beta^2 + 2 \delta  \beta^2 - 8 \delta \omega)}  
  \leq 1 - \frac{\delta^2 \omega}{82} \,,
\end{align}
using the crude estimates $0 \leq \delta \leq 1$, $\beta \leq 2$, $\omega \geq 0$.

We now proceed to show that~\eqref{eq:3435} holds. Observe that for $\alpha_1, \alpha_2$ as in~\eqref{eq:params},
\begin{align*}
 \eta_1(\gamma) &\leq  \left( 1 - \gamma \delta \right) \left( 1 - \gamma \delta\right) \left( 1 + \frac{\gamma \delta}{2} \right)+  \gamma^2 \beta^2 (1-\omega) \left(1 + \frac{2}{\omega} \right) \\
 & \leq \left( 1 - \frac{\gamma \delta}{2} \right)^2 + \frac{2}{\omega} \gamma^2 \beta^2 =: \eta_2(\gamma) \,,
\end{align*}
where we used the inequality $(1-x)(1+\frac{x}{2}) \leq (1-\frac{x}{2})$ and $(1-\omega)(1+2/\omega) \leq \frac{2}{\omega}$ for $\omega > 0$. The quadratic function $\eta_2(\gamma)$ is minimized for $\gamma' = \frac{2\delta \omega}{8 \beta^2 + \delta^2 \omega}$ with value $\eta_2(\gamma')=\frac{8 \beta^2}{8 \beta^2 + \delta^2 \omega} < 1$. Thus by Jensen's inequality
\begin{align}
 \eta_2(\lambda \gamma') \leq (1-\lambda) \eta_2(0) + \lambda \eta_2(\gamma') = 1- \lambda \frac{\delta^2 \omega}{8 \beta^2 + \delta^2 \omega} \label{eq:3455} 
\end{align}
for $0 \leq \lambda \leq 1$, and especially for the choice $\lambda' = \frac{8 \beta^2 + \delta^2 \omega}{2(16\delta + \delta^2 + 4 \beta^2 + 2 \delta \beta^2 - 8 \delta \omega)}$ we have
\begin{align}
\eta_1(\gamma^\star) \leq  \eta_2(\lambda' \gamma') \stackrel{\eqref{eq:3455}}{\leq}  1 - \frac{\delta^2 \omega}{2(16\delta + \delta^2 + 4 \beta^2 + 2 \delta  \beta^2 - 8 \delta \omega)} \,,
\end{align}
as $\gamma^\star = \lambda' \gamma'$. Now we proceed to estimate $\xi_1(\gamma^\star)$. Observe
\begin{align}
 \xi_1(\gamma) &\leq \gamma^2 \beta^2 \left(1 + \frac{2}{\gamma \delta} \right) + (1+\gamma \beta)^2 (1-\omega)\left(1+\frac{\omega}{2}\right)  \leq \gamma^2 \beta^2 \left(1 + \frac{2}{\gamma \delta} \right) + (1+\gamma \beta)^2 \left(1-\frac{\omega}{2}\right)\,,
\end{align}
again from $(1-x)(1+\frac{x}{2}) \leq (1-\frac{x}{2})$ for $x > 0$. As $\beta \leq 2$ we can estimate $(1+\gamma \beta)^2 \leq 1+8 \gamma$ for any $0 \leq \gamma \leq 1$. Furthermore $\gamma^2 \leq \gamma$ for $0 \leq \gamma \leq 1$. Thus
\begin{align}
 \xi_1(\gamma^\star) \leq \beta^2 \left(\gamma^\star + \frac{2\gamma^\star}{\delta}\right) + \left(1-\frac{\omega}{2}\right) (1+8 \gamma^\star) = 1 - \frac{\delta^2 \omega}{2(16\delta + \delta^2 + 4 \beta^2 + 2 \delta  \beta^2 - 8 \delta \omega)}\,,
\end{align}
as a quick calculation shows.
\end{proof}

\section{Proof of Theorem~\ref{th:decentralized_sgd}---Convergence of \algopt}
Recall, that $\bigl\{\xx_i^{(t)}\}_{t=0}^T$ denote the iterates of Algorithm~\ref{alg:quantized_decentralized_sgd} on worker $i \in [n]$. We define
\begin{align}
\overline{\xx}^{(t)} := \frac{1}{n} \sum_{i=1}^n \xx_i^{(t)}\,,
\end{align}
the average over all workers. Note that this quantity is not available to the workers at any given time, but it will be conveniently to use for the proofs.
In this section we use both vector and matrix notation whenever it is more convenient, and define
\begin{align}
\begin{split}
X^{(t)} := \left[ \xx_1^{(t)},\dots, \xx_n^{(t)}\right] \in \R^{d\times n}, \qquad \overline{X}^{(t)} := \left[ \overline{\xx}^{(t)},\dots, \overline{\xx}^{(t)}\right]  \in \R^{d\times n}, \\  \partial F(X^{(t)}, \xi^{(t)}) := \left[\nabla F_1(\xx_{1}^{(t)}, \xi_1^{(t)}), \dots,  \nabla F_n(\xx_{n}^{(t)}, \xi_n^{(t)})\right]  \in \R^{d\times n}.
\end{split} \label{eq:notation_sgd}
\end{align}

Instead of proving Theorem~\ref{th:decentralized_sgd} directly, we prove a slightly more general statement in this section. Algorithm~\ref{alg:quantized_decentralized_sgd} relies on the (compressed) consensus Algorithm~\ref{alg:consensus}. However, we can also show convergence of Algorithm~\ref{alg:quantized_decentralized_sgd} for more general averaging schemes.
In Algorithm~\ref{alg:blackbox_sgd_matrix} below, the function $h : \R^{d\times n}\times \R^{d\times n} \to \R^{d\times n} \times \R^{d\times n}$ denotes a \emph{blackbox averaging scheme}. Note that $h$ could be random. 
\setcounter{algorithm}{3}
\begin{algorithm}[H]
	\caption{\textsc{decentralized SGD with arbitrary averaging scheme}}\label{alg:blackbox_sgd_matrix}
	\begin{algorithmic}[1]
		\INPUT{: $X^{(0)}$, stepsizes $\{\eta_t\}_{t=0}^{T-1}$, averaging function $h : \R^{d\times n}\times \R^{d\times n} \to \R^{d\times n} \times \R^{d\times n}$}
		\STATE {\it In parallel (task for worker $i, i \in [n]$)}
		\FOR{$t$\textbf{ in} $0\dots T-1$}
		\STATE $X^{(t + \frac{1}{2})} = X^{(t)} - \eta_t\partial F_i(X^{(t)}, \xi^{(t)})$ \hfill $\triangleright$ stochastic gradient updates
		\STATE $(X^{(t + 1)}, Y^{(t + 1)}) = h(X^{(t + \frac{1}{2})}, Y^{(t)})$ \hfill  $\triangleright$ blackbox averaging/gossip
		\ENDFOR
	\end{algorithmic}
\end{algorithm}
In this work we in particular focus on two choices of $h$, the averaging operator $ h(X^{(t)}, Y^{(t)}) \mapsto (X^{(t + 1)}, Y^{(t + 1)})$:
\begin{itemize}
 \item  Setting $X^{(t + 1)} = X^{(t)}W$ and $Y^{(t + 1)} = X^{(t + 1)}$ corresponds to standard (exact) averaging with mixing matrix $W$, as in algorithm~\eqref{eq:boyd}.
 \item Setting $X^{(t + 1)} = X^{(t)} + \gamma Y^{(t)}\left(W - I\right)$and  $Y^{(t + 1)} = Y^{(t)} + Q(X^{(t + 1)} - Y^{(t)})$ for $Y^{(t)} = \hat{X}^{(t + 1)}$, we get the compressed consensus algorithm~\eqref{eq:ours}, leading to Algorithm~\ref{alg:quantized_decentralized_sgd}, as introduced in the main text.
\end{itemize}

\begin{assumption}\label{assump:avg}
    For an averaging scheme $h \colon \R^{d \times n} \times \R^{d \times n} \to \R^{d \times n} \times \R^{d \times n}$ let $(X^+,Y^+):=h(X,Y)$ for $X,Y \in \R^{d \times n}$. Assume that $h$
     preserves the average of the first iterate over all iterations:
	\begin{align*}
	X^+ \frac{\1\1^\top}{n} &= X \frac{\1\1^\top}{n} \,, &\forall X,Y \in \R^{d \times n}\,,
\intertext{	and that it converges with linear rate for a parameter $0 < p \leq 1$ }
	\EE{h}{\Psi(X^+, Y^+)} &\leq (1 - p) {\Psi(X, Y)}\,, &\forall X,Y \in \R^{d \times n}\,,
	\end{align*}
	and Laypunov function $\Psi(X, Y) := \|X - \overline{X}\|_F^2 + \|X - Y\|_F^2$ with $\overline{X} := \tfrac{1}{n} X \1\1^\top$, where $\mathbb E_{h}$ denotes the expectation over internal randomness of averaging scheme $h$.
\end{assumption}

This assumption holds for exact averaging as in~\eqref{eq:boyd} with parameter $p = \gamma \delta$ (as shown in Theorem~\ref{th:boyd2_rate}). %
For the proposed compressed consensus algorithm~\eqref{eq:ours} the assumption holds for parameter $p = \frac{\omega \delta^2}{82}$ (as show in Theorem~\ref{th:consensus}). Here $\omega$ denotes the compression ratio and $\delta$ the eigengap of mixing matrx $W$. We can now state the more general Theorem (that generalizes Theorem~\ref{th:decentralized_sgd}):

\begin{theorem}\label{th:sgd_general}
	Under Assumption~\ref{assump:avg} for $p > 0$, Algorithm~\ref{alg:blackbox_sgd_matrix} with stepsize $\eta_t = \frac{4}{\mu(a + t)}$, for parameter $a \geq \max \left\{ \frac{5}{p}, 16 \kappa\right\}$, $\kappa = \frac{L}{\mu}$ converges at the rate
	\begin{align*}
	f(\xx_{avg}^{(T)}) - f^\star \leq \dfrac{\mu a^3}{8S_T} \norm{\overline{\xx}^{(0)} - \xx^\star}^2  + \dfrac{4T(T + 2a)}{\mu S_T}\frac{\overline{\sigma}^2}{n} + \dfrac{64T}{\mu^2 S_T} (2L + \mu)  \dfrac{40}{p^2} G^2,
	\end{align*}
	where $\xx_{avg}^{(T)}  = \frac{1}{S_T}\sum_{t = 0}^{T - 1} w_t \overline{\xx}^{(t)}$ for weights $w_t = (a + t)^2$, and $S_T = \sum_{t = 0}^{T - 1}w_t\geq\frac{1}{3} T^3$.
\end{theorem}
\begin{proof}[Proof of Theorem~\ref{th:decentralized_sgd}]
	The proof follows from Theorem~\ref{th:sgd_general} using the consensus averaging algorithm~\ref{alg:consensus} (giving $p=\frac{\delta^2 \omega}{82}$ by Theorem~\ref{th:consensus}) and the inequality $\E \mu \norm{\xx_0 - \xx^\star} \leq 2G$
	derived in \citep[Lemma 2]{Rakhlin2012:bound_for_a_0} to upper bound the first term.
\end{proof}

\subsection{Proof of Theorem~\ref{th:sgd_general}}
The proof below uses techniques from both \cite{Stich2018:sparsifiedSGD}  and \cite{Stich2018:LocalSGD}. %

\begin{lemma}\label{lem:main_recursion_for_decentralized_sgd}
	The averages $\overline{\xx}^{(t)}$ of the iterates of the Algorithm \ref{alg:blackbox_sgd_matrix} satisfy the following
	\begin{multline*}
	\EE{\xi_1^{(t)},\dots,\xi_n^{(t)}}{\|\overline{\xx}^{(t + 1)} - \xx^\star\|}^2 \leq \left(1 - \dfrac{\eta_t\mu}{2}\right) {\norm{\overline{\xx}^{(t)} - \xx^\star}}^2 + \dfrac{\eta_t^2\overline{\sigma}^2}{n} - 2\eta_t \left(1 - 2L\eta_t\right)\left(f(\overline{\xx}^{(t)}) - f^\star\right) + \\+ \eta_t \dfrac{2\eta_t L^2 + L + \mu}{n} \sum_{i = 1}^{n}\norm{\overline{\xx}^{(t)} - \xx_i^{(t)}}^2,
	\end{multline*}
	where %
	$\overline{\sigma}^2 = \frac{1}{n}\sum_{i = 1}^{n} \sigma_i^2$.
\end{lemma}
\begin{proof} Because the blackbox averaging function $h$ preserves the average (Assumption~\ref{assump:avg}), we have
	\begin{align*}
	\norm{\overline{\xx}^{(t + 1)} - \xx^\star}^2 &= \norm{\overline{\xx}^{(t)} - \frac{\eta_t}{n}\sum_{j = 1}^n  \nabla F_j(\xx_j^{(t)}, \xi_j^{(t)}) - \xx^\star}^2 \\
	&= \norm{\overline{\xx}^{(t)} - \xx^\star - \frac{\eta_t}{n} \sum_{i = 1}^{n}\nabla f_i(\xx_i^{(t)}) + \frac{\eta_t}{n} \sum_{i = 1}^{n}\nabla f_i(\xx_i^{(t)}) - \frac{\eta_t}{n}\sum_{j = 1}^n \nabla F_j(\xx_j^{(t)}, \xi_j^{(t)})}^2 = \\
	&= \norm{\overline{\xx}^{(t)} - \xx^\star - \frac{\eta_t}{n} \sum_{i = 1}^{n}\nabla f_i(\xx_i^{(t)})}^2
	+\eta_t^2 \norm{\frac{1}{n} \sum_{i = 1}^{n}\nabla f_i(\xx_i^{(t)}) - \frac{1}{n}\sum_{j = 1}^n \nabla F_j(\xx_j^{(t)}, \xi_j^{(t)})}^2+\\
	&\qquad {}+ \frac{2\eta_t}{n}\left\langle  \overline{\xx}^{(t)} - \xx^\star - \frac{\eta_t}{n} \sum_{i = 1}^{n}\nabla f_i(\xx_i^{(t)}), \sum_{i = 1}^{n}\nabla f_i(\xx_i^{(t)}) - \sum_{j = 1}^n \nabla F_j(\xx_j^{(t)}, \xi_j^{(t)}) \right\rangle \,.
	\end{align*}
	The last term is zero in expectation, as $\EE{\xi_i^{(t)}}{\nabla F_i(\xx_i^{(t)}, \xi_i^{(t)}) } = \nabla f_i(\xx_i^{(t)})$. The second term is less than $\frac{\eta_t^2 \overline{\sigma}^2}{n}$ (Remark~\ref{rem:variance}). The first term can be written as:
	\begin{align*}
	\norm{\overline{\xx}^{(t)} - \xx^\star - \frac{\eta_t}{n} \sum_{i = 1}^{n}\nabla f_i(\xx_i^{(t)})}^2 = \norm{\overline{\xx}^{(t)} - \xx^\star}^2 + \eta_t^2 \underbrace{\norm{\frac{1}{n}\sum_{i = 1}^{n}\nabla f_i(\xx_i^{(t)})}^2}_{=: T_1} - \underbrace{2\eta_t \lin{\overline{\xx}^{(t)} - \xx^\star, \frac{1}{n}\sum_{i = 1}^{n}\nabla f_i(\xx_i^{(t)})}}_{=: T_2} \,.
	\end{align*}
	We can estimate %
	\begin{align*}
	T_1 &=
	\left\| \frac{1}{n}\sum_{i = 1}^{n}(\nabla f_i(\xx_i^{(t)}) -\nabla f_i(\overline{\xx}^{(t)}) + \nabla f_i(\overline{\xx}^{(t)}) -  \nabla f_i(\xx^\star))\right\|^2 \\
	&\stackrel{\eqref{eq:norm_of_sum}}{\leq}\frac{2}{n} \sum_{i = 1}^{n} \norm{\nabla f_i(\xx_i^{(t)}) - \nabla f_i(\overline{\xx}^{(t)})}^2  + 2\norm{\frac{1}{n} \sum_{i = 1}^{n} \nabla f_i(\overline{\xx}^{(t)}) - \frac{1}{n} \sum_{i = 1}^{n} \nabla f_i(\xx^\star)}^2 \\
	&\stackrel{\eqref{def:smooth},\eqref{eq:l-smooth}}{\leq} \dfrac{2L^2}{n} \sum_{i = 1}^{n} \norm{\xx_i^{(t)} - \overline{\xx}^{(t)}}^2 + \dfrac{4L}{n}\sum_{i = 1}^{n} \left(f_i(\overline{\xx}^{(t)}) - f_i(\xx^\star)\right)
	\\& = \dfrac{2L^2}{n} \sum_{i = 1}^{n} \norm{\xx_i^{(t)} - \overline{\xx}^{(t)}}^2 + 4L \left(f(\overline{\xx}^{(t)}) - f^\star\right)\,.
	\end{align*}
	
	And for the remaining $T_2$ term: 
	\begin{align*}
	- \frac{1}{\eta_t} T_2
	&= - \frac{2}{n} \sum_{i = 1}^{n}\left[\lin{\overline{\xx}^{(t)}  - \xx_i^{(t)}, \nabla f_i(\xx_i^{(t)})} + \lin{\xx_i^{(t)} - \xx^\star, \nabla f_i(\xx_i^{(t)})}\right]
	\\ &\stackrel{\eqref{def:smooth},\eqref{def:strongconvex}}{\leq} - \dfrac{2}{n}\sum_{i = 1}^{n} \left[ f_i(\overline{\xx}^{(t)}) - f_i(\xx_i^{(t)}) - \dfrac{L}{2} \norm{\overline{\xx}^{(t)} - \xx_i^{(t)}}^2 + f_i(\xx_i^{(t)}) - f_i(\xx^\star) + \dfrac{\mu}{2} \norm{\xx_i^{(t)} - \xx^\star}^2\right] \\
	&{}\stackrel{\eqref{eq:norm_of_sum}}{\leq} - 2 \left(f(\overline{\xx}^{(t)}) - f(\xx^\star)\right) + \dfrac{L + \mu}{n}\sum_{i = 1}^{n} \norm{\overline{\xx}^{(t)} - \xx_i^{(t)}}^2 - \dfrac{\mu}{2} \norm{\overline{\xx}^{(t)} - \xx^\star}^2\,.
	\end{align*}
	Putting everything together we are getting statement of the lemma.%
\end{proof}

\begin{lemma}\label{lem:last_term_bound}
	The iterates $\{X^{(t)}\}_{t \geq 0}$ of Algorithm~\ref{alg:blackbox_sgd_matrix} with stepsizes $\eta_t = \frac{b}{t + a}$, for parameters $a \geq \frac{5}{p}$,   $b > 0$ satisfy 
	\begin{align*}
	\norm{X^{(t + 1)} - \overline{X}^{(t + 1)}}_F^2\leq 40 \eta_t^2 \dfrac{1}{p^2} nG^2\,.
	\end{align*}
	Here $0< p \leq 1$ denotes the a convergence rate of the blackbox averaging algorithm as in Assumption~\ref{assump:avg}.
\end{lemma}
\begin{proof}
	Using linear convergence of the blackbox averaging algorithm as given in Assumption~\ref{assump:avg} we can write for $\Xi := \E{\norm{X^{(t + 1)} - \overline{X}^{(t + 1)}}_F^2} + \E{\norm{X^{(t + 1)} -\hat{X}^{(t + 2)} }_F^2}$,
	\begin{align*}
	\Xi &\leq
	(1-p)\E{\norm{\overline{X}^{(t + \frac{1}{2})} - X^{(t + \frac{1}{2})}}_F^2}  + (1-p)  \E{\norm{\hat{X}^{(t + 1)} - {X}^{(t + \frac{1}{2})}}_F^2} \\
	&= (1-p) \E{\norm{\overline{X}^{(t)} - X^{(t)} + \eta_t\partial F(X^{(t)}, \xi^{(t)})\left(\frac{\1\1^\top}{n} - I\right)}_F^2 }
	\\ &\qquad + (1-p)\E{\norm{\hat{X}^{(t + 1)} - {X}^{(t)} + \eta_t \partial F(X^{(t)}, \xi^{(t)})}_F^2} \\
	& \stackrel{\eqref{eq:norm_of_sum_of_two}}{\leq} (1-p) (1 + \alpha_3^{-1}) \E{\left( \norm{\overline{X}^{(t)} - X^{(t)}}_F^2  + \norm{\hat{X}^{(t + 1)} - {X}^{(t)}}_F^2 \right)}\\
	&\qquad  + (1-p)(1 + \alpha_3)\eta_t^2 \E{\left(\norm{\partial F(X^{(t)}, \xi^{(t)})\left(\frac{\1\1^\top}{n} - I\right)}_F^2 +  \norm{\partial F(X^{(t)}, \xi^{(t)})}_F^2\right)}\\
	& \leq (1-p)\left((1 + \alpha_3^{-1})\E{\left( \norm{\overline{X}^{(t)} - X^{(t)}}_F^2  + \norm{\hat{X}^{(t + 1)} - {X}^{(t)}}_F^2\right)} + 2n(1 + \alpha_3)\eta_t^2G^2\right) \\
	&\stackrel{\alpha_3 = \frac{2}{p}}{\leq} \left(1-\frac{p}{2}\right)\E{\left(\norm{\overline{X}^{(t)} - X^{(t)}}_F^2  + \norm{\hat{X}^{(t + 1)} - {X}^{(t)}}_F^2\right)} + \frac{4n}{p}\eta_t^2G^2 \,.
	\end{align*}
	The statement now follows from Lemma~\ref{lem:recursion} and the inequality
	\begin{align*}
	\E{\norm{X^{(t + 1)} - \overline{X}^{(t + 1)}}_F^2} &\leq \Xi :=\E{\norm{X^{(t + 1)} -\hat{X}^{(t + 2)} }_F^2} + \E{\norm{X^{(t + 1)} - \overline{X}^{(t + 1)}}_F^2}\,. \qedhere
	\end{align*}
\end{proof}

\begin{lemma}\label{lem:recursion} %
	Let $\{r_t\}_{t \geq 0}$ denote a sequence of positive real values satisfying $r_0 = 0$ and
	\begin{align*}
	r_{t + 1} &\leq \left(1- \frac{p}{2}\right) e_t + \frac{2}{p} \eta_t^2 A\,, & &\forall t \geq 0\,,
	\intertext{for a parameter $p > 0$, stepsize $\eta_t = \frac{b}{t + a}$, for parameters $a \geq \frac{5}{p}$ and with arbitrary $b > 0$. Then 
	$r_{t}$ is bounded as }
	r_t &\leq 20 \eta_t^2 \dfrac{1}{p^2}A\,, & &\forall t \geq 0\,.
	\end{align*}
\end{lemma}
\begin{proof}
	We will proceed the proof by induction. For $t = 0$ the statement is true by assumption on $r_0 = 0$.	
	Suppose that for timestep $t$ the statement is also true, then for timestep $t + 1$	
	\begin{align*}
	e_{t + 1} &\leq \left(1- \frac{p}{2}\right) e_t + \frac{2}{p} \eta_t^2 A  \leq \left(1- \frac{p}{2}\right) 20 \eta_t^2 \dfrac{1}{p^2}A + \frac{2}{p} \eta_t^2 A = A \eta_t^2 \frac{1}{p^2} \left(-8p + 20\right) \,.
	\end{align*}
	Now we show $\eta_t^2 \left(- 8p+ 20 \right) \leq 20 \eta_{t+1}^2$ which proves the claim. By assumption $p \geq \frac{5}{a}$, hence
	\begin{align*}
	\eta_t^2 \left(- 8p+ 20 \right) &\leq 20 \eta_t^2 \left(1-\frac{2}{a}\right) \leq 20 \eta_{t+1}^2\,,
	\end{align*}
	where the second inequality follows from %
	\begin{align*}
	(a+t+1)^2  \left(1-\frac{2}{a}\right) &= (a+t)^2 + 2(a+t)+1 - \left(2\frac{(a+t)^2}{a} + 4\frac{(a+t)}{a} + \frac{2}{a}\right) \\
	&\leq (a+t)^2 + 2(a+t)+1 - \left(2(a+t) + 4 \right) \leq (a+t)^2\,. \qedhere
	\end{align*}
\end{proof}

\begin{lemma}[\citet{Stich2018:LocalSGD}]\label{lem:from_local_sgd}
	Let $\{a_t\}_{t \geq 0}$, $a_t \geq 0$, $\{e_t\}_{t \geq 0}$, $e_t \geq 0$ be sequences satisfying 
	\begin{align*}
	a_{t + 1} \leq (1 - \mu \eta_t) a_t- \eta_t e_t A+ \eta_t^2B + \eta_t^3C\,,
	\end{align*}
	for stepsizes $\eta_t = \frac{4}{\mu(a + t)}$ and constants $A > 0, B, C\geq 0$, $\mu > 0$, $a > 1$. Then 
	\begin{align*}
	\dfrac{A}{S_T} \sum_{t = 0}^{T - 1} w_t e_t \leq \dfrac{\mu a^3}{4S_T} a_0 + \dfrac{2T(T + 2a)}{\mu S_T} B + \dfrac{16T}{\mu^2 S_T} C\,,
	\end{align*}
	for $w_t = (a + t)^2$ and $S_T := \sum_{t = 0}^{T  - 1} w_t = \frac{T}{6} (2T^2 + 6aT -3T + 6a^2 -6a + 1) \geq \frac{1}{3}T^3$.
\end{lemma}

\begin{proof}[Proof of Theorem~\ref{th:sgd_general}] Substituting the result of Lemma~\ref{lem:last_term_bound} into the bound provided in Lemma~\ref{lem:main_recursion_for_decentralized_sgd} (here we use $a \geq \frac{5}{p}$) we get that 
\begin{align*}
\E{\|\overline{\xx}^{(t + 1)} - \xx^\star\|}^2 \leq \left(1 - \dfrac{\eta_t\mu}{2}\right) \E{\norm{\overline{\xx}^{(t)} - \xx^\star}}^2  - 2\eta_t \left(1 - 2L\eta_t\right)e_t + \eta_t^2\dfrac{\overline{\sigma}^2}{n} (2\eta_t L^2 + L + \mu)40 \eta_t^3 \dfrac{1}{p^2} G^2,
\end{align*}
For $\eta_t \leq \frac{1}{4L}$ (this holds, as $a \leq 16\kappa$) it holds $2L\eta_t - 1 \leq -\frac{1}{2}$ and $(2\eta_t L^2 + L + \mu) < (2L + \mu)$, hence 

\begin{align*}
\E{\|\overline{\xx}^{(t + 1)} - \xx^\star\|}^2 \leq \left(1 - \dfrac{\eta_t\mu}{2}\right) \E{\norm{\overline{\xx}^{(t)} - \xx^\star}}^2 + \dfrac{\eta_t^2\overline{\sigma}^2}{n} - \eta_t e_t +  (2L + \mu)40 \eta_t^3 \dfrac{1}{p^2} G^2\,.
\end{align*}
	From Lemma~\ref{lem:from_local_sgd}  we get
	\begin{align*}
	\dfrac{1}{S_T} \sum_{t = 0}^{T - 1} w_t e_t \leq \dfrac{\mu a^3}{8S_T} {\norm{\overline{\xx}^{(0)} - \xx^\star}}^2 + \dfrac{4T(T + 2a)}{\mu S_T}\frac{\overline{\sigma}^2}{n} 
	+ \dfrac{64T}{\mu^2 S_T} (2L + \mu) 40 \dfrac{1}{p^2} G^2,
	\end{align*}
	for weights $w_t = (a + t)^2$ and $S_T := \sum_{t = 0}^{T  - 1} w_t = \frac{T}{6} (2T^2 + 6aT -3T + 6a^2 -6a + 1) \geq \frac{1}{3}T^3$, where $p$ is convergence rate of the averaging scheme. The theorem statement follows from convexity of $f$.
\end{proof}

\section{Efficient Implementation of the Algorithms}\label{sect:efficient_implementation}

In this section we present memory-efficient implementations of \algcons and \algopt algorithms, which require each node to store only three vectors: $\xx$, $\hat{\xx}_i$ and $\mathbf{s}_i = \sum_{i = 1}^n w_{ij} \hat{\xx}_j$. 

\begin{algorithm}[h]
	\renewcommand{\algorithmiccomment}[1]{#1}
	\caption{Memory-efficient \algcons %
	}\label{alg:consensus_mem}
	\begin{algorithmic}[1]
		\INPUT{:  Initial values $\xx_i^{(0)} \in \R^d$ on each node $i \in [n]$, stepsize $\gamma$, communication graph $G = ([n], E)$ and mixing matrix $W$, initialize $\hat{\xx}_i^{(0)} := \0$ , $\mathbf{s}_i^{(0)} = 0$, $\forall i$}\\
		\FOR[{{\it in parallel for all workers $i \in [n]$}}]{$t$ \textbf{in} $0\dots T-1$}
		\STATE $\qq_i^{(t)} := Q(\xx_i^{(t + 1)} - \hat{\xx}_i^{(t)})$
		\FOR{neighbors $j \colon \{i,j\} \in E$ (including $\{i\} \in E$)}
		\STATE Send $\qq_i^{(t)}$ and receive $\qq_j^{(t)}$ 
		\ENDFOR
		\vspace{1mm}
		\STATE $\hat{\xx}_i^{(t + 1)} = \hat{\xx}_i^{(t)} + \qq_i^{(t)}$
		\STATE $\mathbf{s}_i^{(t + 1)} := \mathbf{s}_i^{(t)} + \sum_{i = 1}^n w_{ij} \qq_j^{(t)}$%
		\STATE $\xx_i^{(t + 1)} := \xx_i^{(t)} + \gamma\left(\mathbf{s}_i^{(t + 1)}  - \hat{\xx}_i^{(t + 1)}\right)$
		\ENDFOR
	\end{algorithmic}
\end{algorithm}

\begin{algorithm}[H]
	\renewcommand{\algorithmiccomment}[1]{#1}
	\caption{Memory-efficient \algopt %
	}\label{alg:quantized_decentralized_sgd_mem}
	\begin{algorithmic}[1]
		\INPUT{:  Initial values $\xx_i^{(0)} \in \R^d$ on each node $i \in [n]$, consensus stepsize $\gamma$, communication graph $G = ([n], E)$ and mixing matrix $W$, initialize $\hat{\xx}_i^{(0)} := \0$ $\forall i$}\\
		\FOR[{{\it in parallel for all workers $i \in [n]$}}]{$t$\textbf{ in} $0\dots T-1$}
		\STATE Sample $\xi_i^{(t)}$, compute gradient $\gg_i^{(t)} \!:= \nabla F_i(\xx_i^{(t)}\!, \xi_i^{(t)})$\!
		\STATE $\xx_i^{(t + \frac{1}{2})} := \xx_i^{(t)} - \eta_t \gg_i^{(t)}$  
		\STATE $\qq_i^{(t)} := Q(\xx_i^{(t + \frac{1}{2})} - \hat{\xx}_i^{(t)})$ 
		\FOR{neighbors $j \colon \{i,j\} \in E$ (including $\{i\} \in E$)} 
		\STATE Send $\qq_i^{(t)}$ and receive $\qq_j^{(t)}$ 
		\ENDFOR
		\vspace{1mm}
		\STATE $\hat{\xx}^{(t+1)}_i := \qq^{(t)}_i + \hat{\xx}_i^{(t)}$
		\STATE $\mathbf{s}_i^{(t + 1)} := \mathbf{s}_i^{(t)} + \sum_{i = 1}^n w_{ij} \qq_j^{(t)}$%
		\STATE $\xx_i^{(t + 1)} := \xx_i^{(t + \frac{1}{2})} + \gamma \left(\mathbf{s}_i^{(t + 1)} - \hat{\xx}^{(t+1)}_i\right)$
		\ENDFOR 
	\end{algorithmic}
\end{algorithm}

\section{Parameters Search Details of SGD Experiments}\label{sect:parameters_search_details}

For each optimization problem, we first tuned $\gamma$ on a separate average consensus problem with the same configuration (topology, number of nodes, quantization, dimension). Parameters $a$, $b$ where later tuned separately for each algorithm by running the algorithm for 10 epochs. To find $a$ and $b$ we performed grid search independently for each algorithm and each quantization function. For values of $a$ we used logarithmic grid of powers of $10$. We searched values of $b$ in the set $\{1, 0.1 d, d ,10 d, 100 d\}$.

\section{Additional Experiments}\label{sect:additional_experiments}

\begin{table}[h!]
	\centering
		\begin{tabular}{l|c|c|c||c|c|c}
			&  \multicolumn{3}{c||}{Epsilon} &\multicolumn{3}{c}{RCV1}\\\hline
			experiment & $a$ &$\tau$ & $\gamma$ & $a$ &$\tau$ & $\gamma$ \\ \hline\hline
			\textsc{Plain} & 0.1 & $d$ & - &      1 & 1 & - \\
			\textsc{Choco}, $(\operatorname{qsgd}_{16})$ & 0.1 & $d$ &   0.34 &       1  & 1  &   0.078\\
			\textsc{Choco}, ($\operatorname{rand}_{1\%}$)  & 0.1 & $d$ & 0.01 &    1&    0.1$d$& 0.016\\
			\textsc{Choco}, ($\operatorname{top}_{1\%}$)  & 0.1 &$d$ &  0.04  &   1 & 1 & 0.04\\
			\textsc{DCD}, ($\operatorname{rand}_{1\%}$)  & $10^{-15}$ & $d$ & - &    $10^{-15}$  & $d$ & - \\
			\textsc{DCD}, $(\operatorname{qsgd}_{16})$ & 0.01 & $d$ & - &  $ 10^{-15}$   & $d$& - \\
			\textsc{ECD}, ($\operatorname{rand}_{1\%}$)  & $10^{-6}$ & $d$ & - &    $10^{-4} $ & 10$d$& - \\
			\textsc{ECD}, ($\operatorname{qsgd}_{16}$)  & $10^{-6}$ & $d$ & - &   $ 10^{-15}$  & $d$& -
		\end{tabular}
	\caption{Values for initial learning rate and consensus learning rate used in SGD experiments Fig. \ref{fig:sgd_random_20}, \ref{fig:sgd_qsgd}. Parameter $\gamma$ found separately by tuning average consensus with the same configuration (topology, number of nodes, quantization, dimension). Parameters $a$, $\tau$ found by tuning. \textsc{ECD}, \textsc{DCD} stepsizes are small because it diverge for larger choices. }
	\label{tab:learning_rates_sgd_random}
\end{table}

\setlength{\smallfigwidth}{0.32\textwidth}
\begin{figure*}[h!]
	\centering
		\centering
		\includegraphics[width=\smallfigwidth]{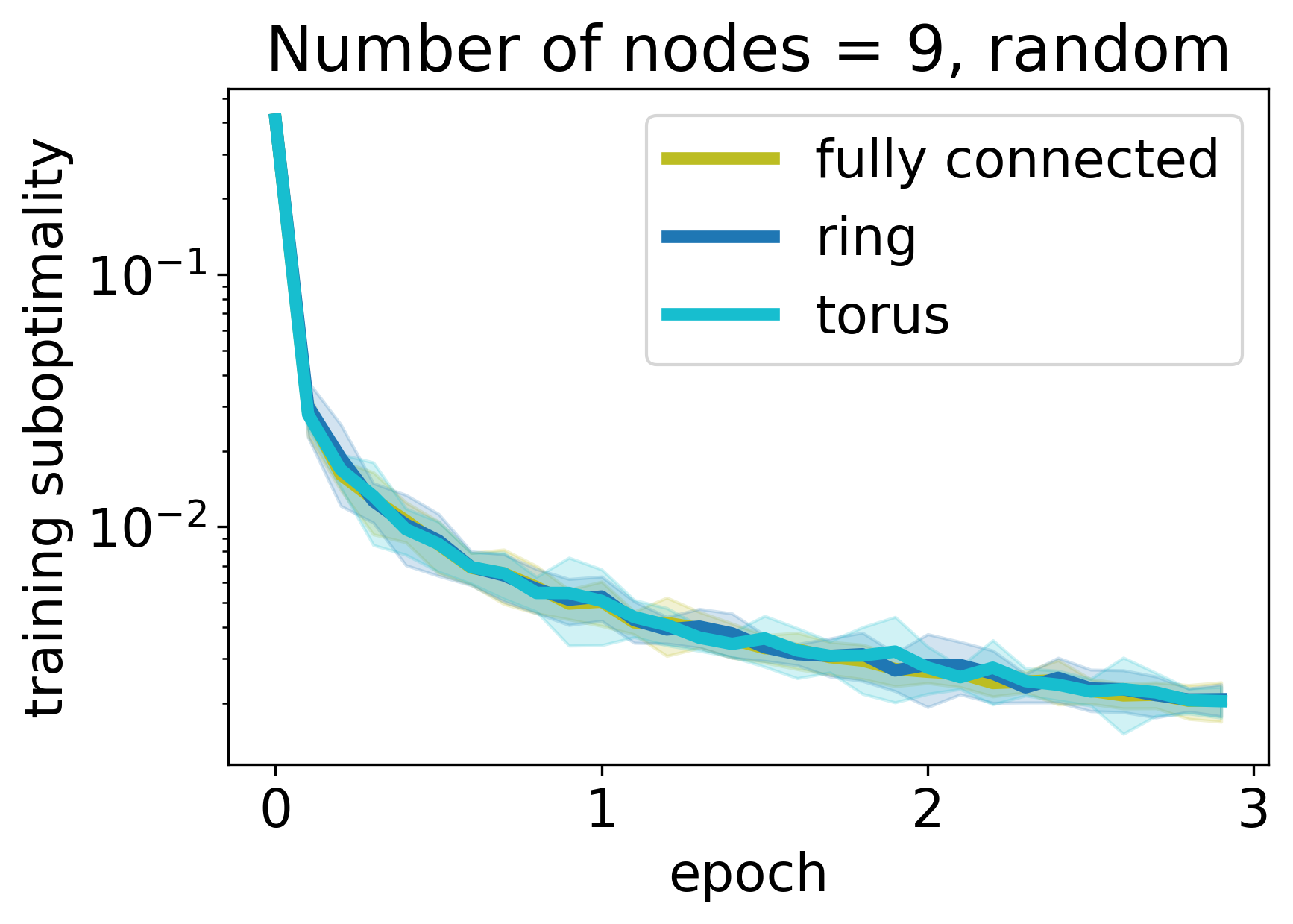}
		\hfill
		\includegraphics[width=\smallfigwidth]{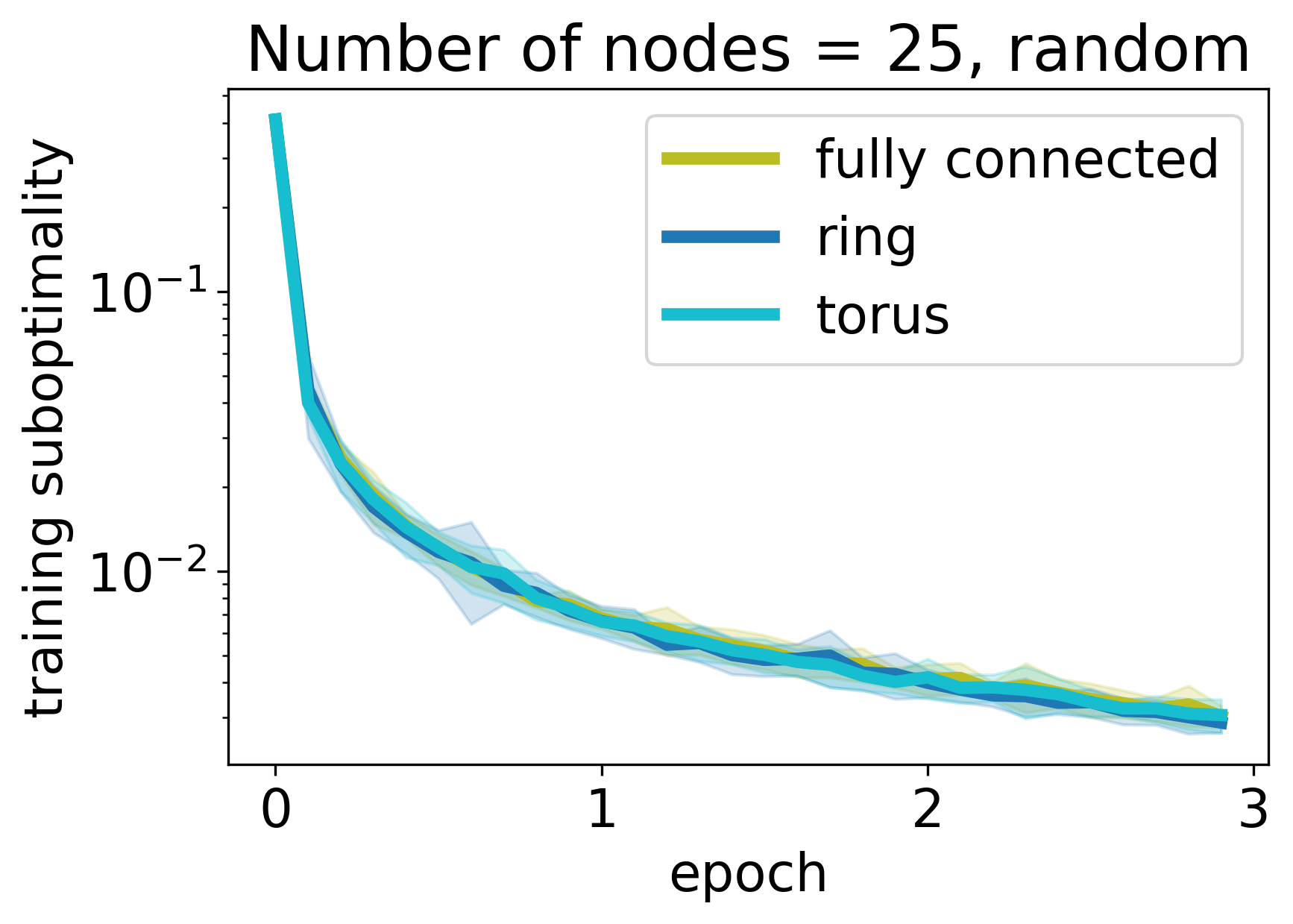}
		\hfill
		\includegraphics[width=\smallfigwidth]{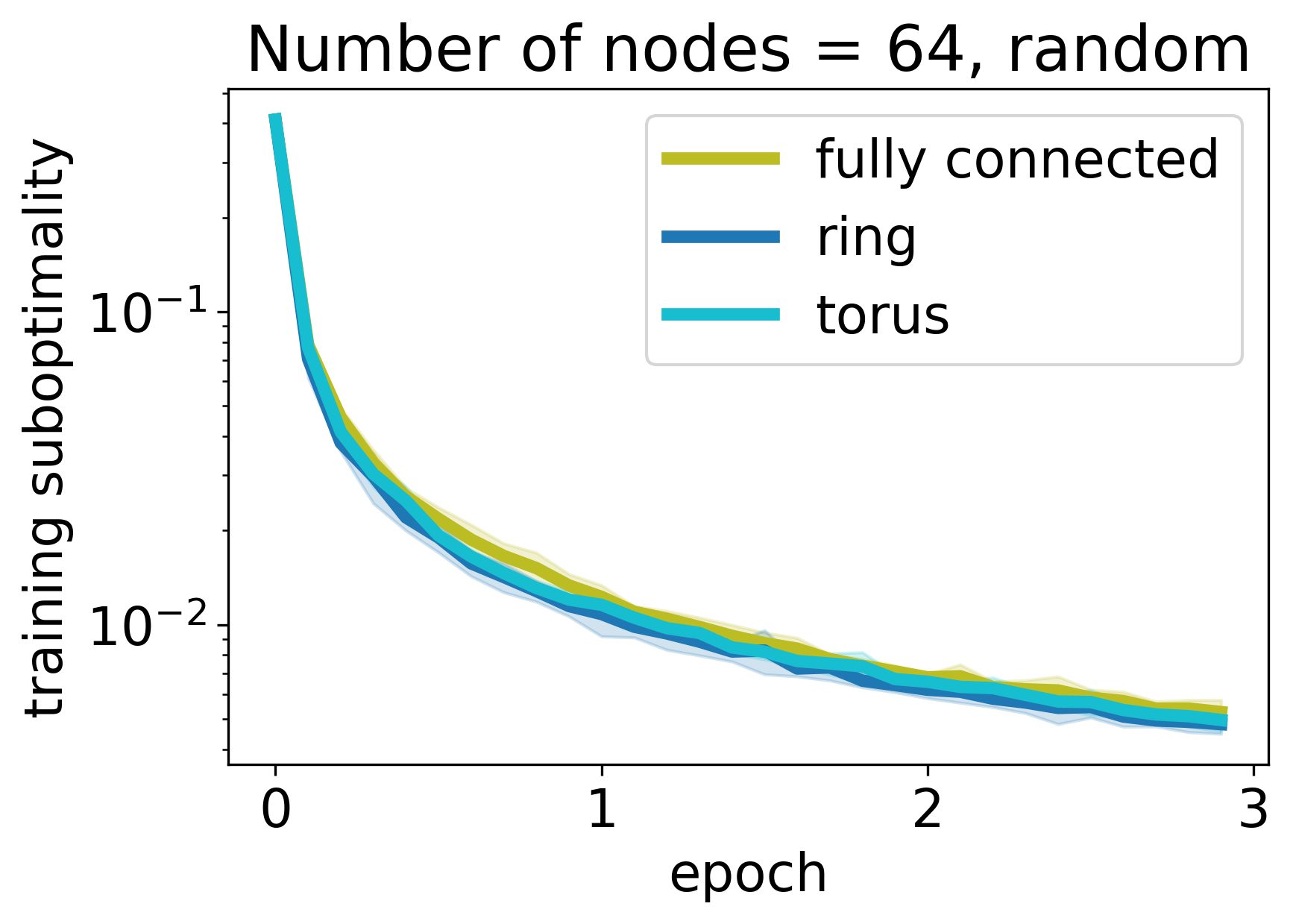}
		\caption{Performance of Algorithm~\ref{alg:all-to-all} on ring, torus and fully connected topologies for $n\in \{9,25,64\}$ nodes. \emph{Randomly shuffled} data between workers}%
		\label{fig:topologies_random_all}
\end{figure*}

\begin{figure}[h!]
	\centering
	\begin{minipage}{0.8\textwidth}
		\setlength{\smallfigwidth}{0.48\linewidth}
		\vspace{4mm}
		\includegraphics[width=\smallfigwidth]{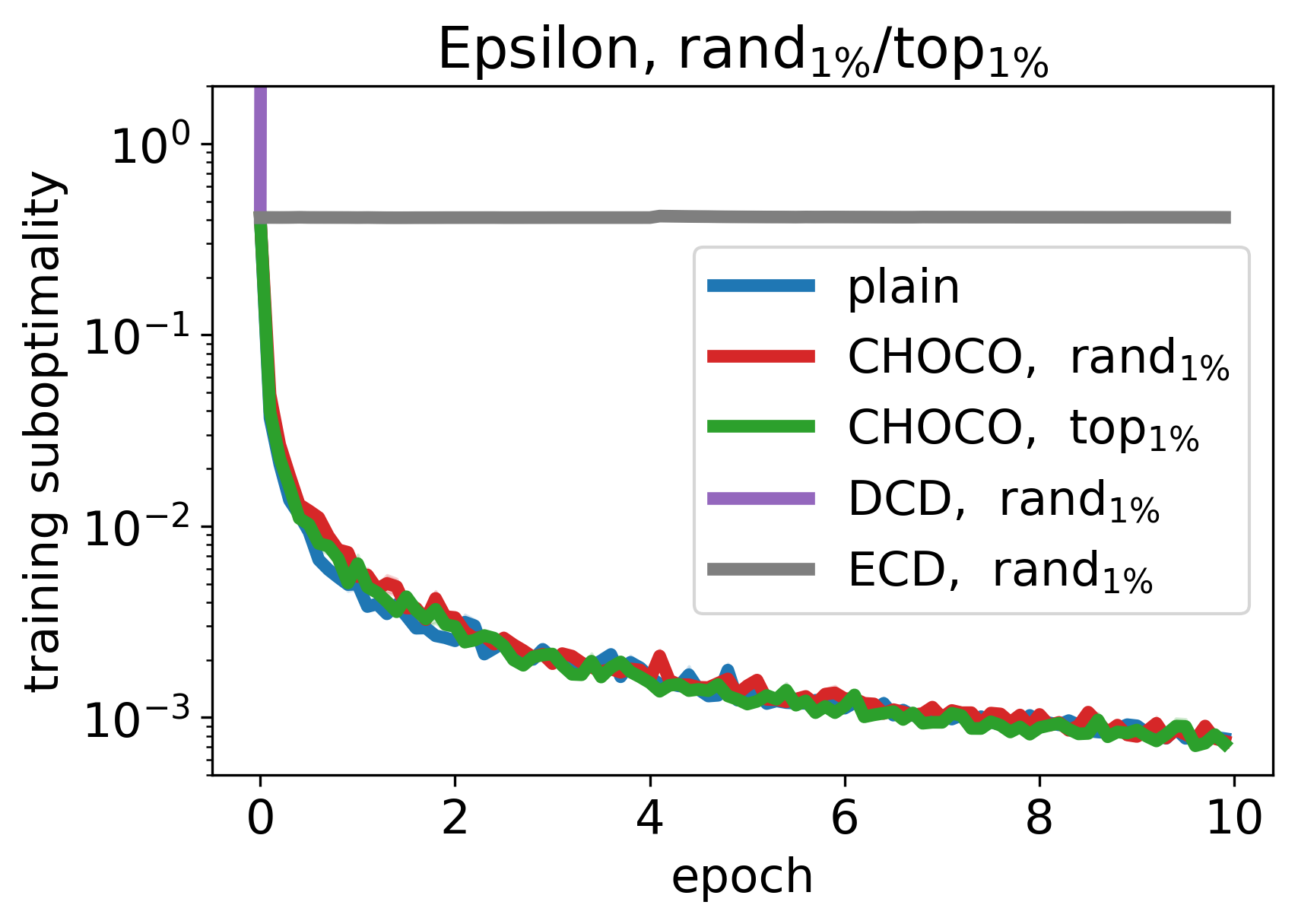}
		\hfill
		\includegraphics[width=\smallfigwidth]{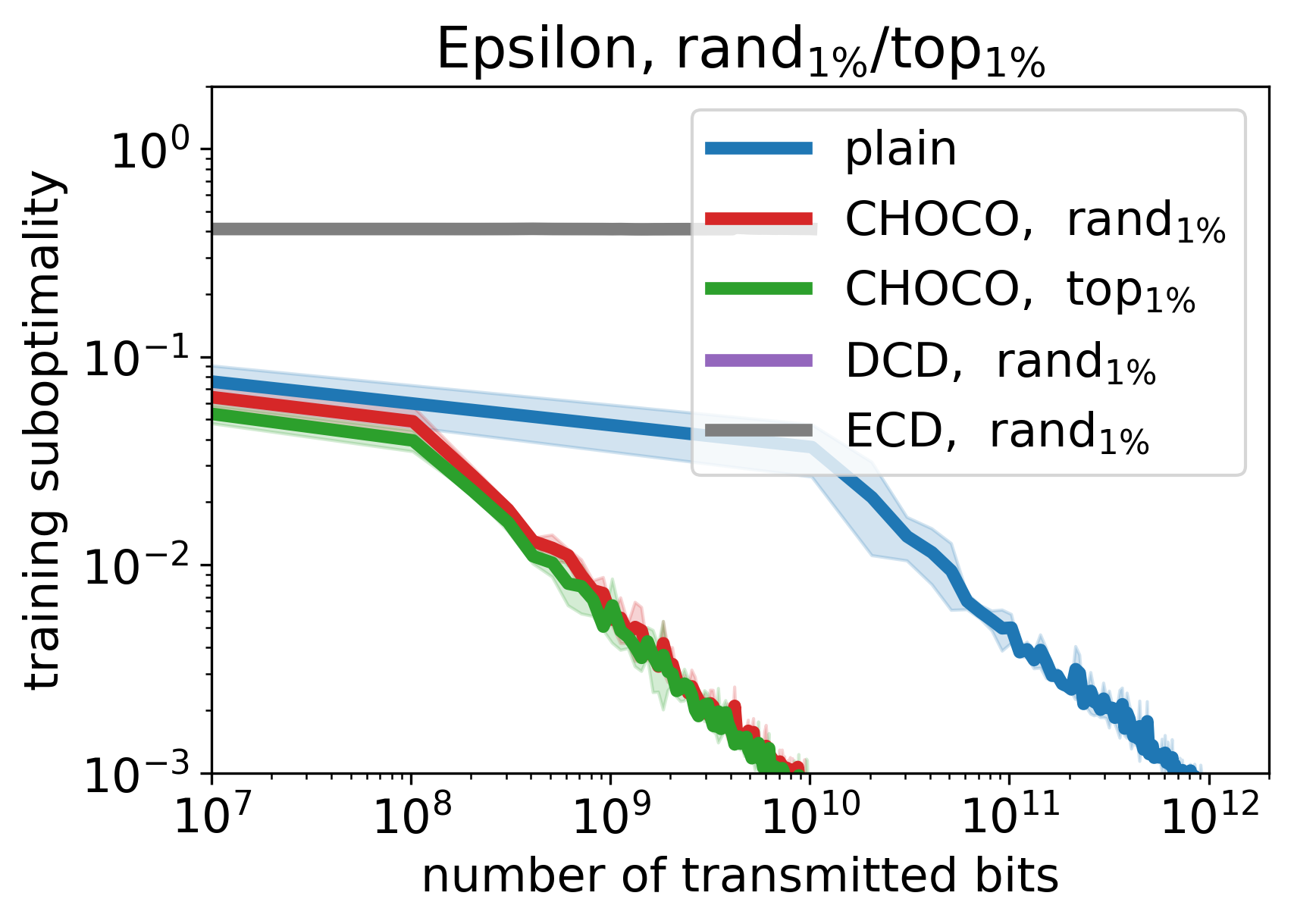}\\
		\includegraphics[width=\smallfigwidth]{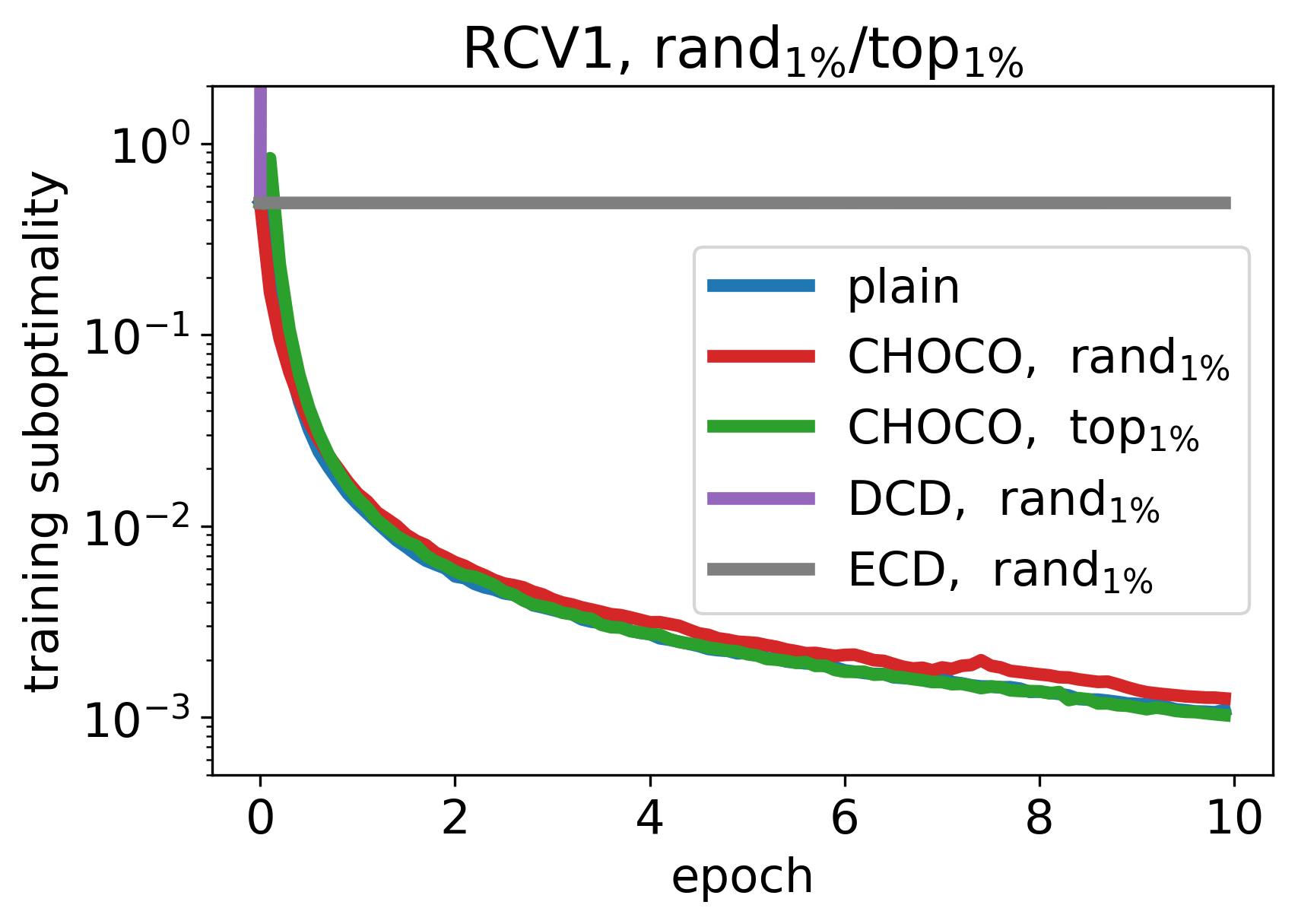}
		\hfill
		\includegraphics[width=\smallfigwidth]{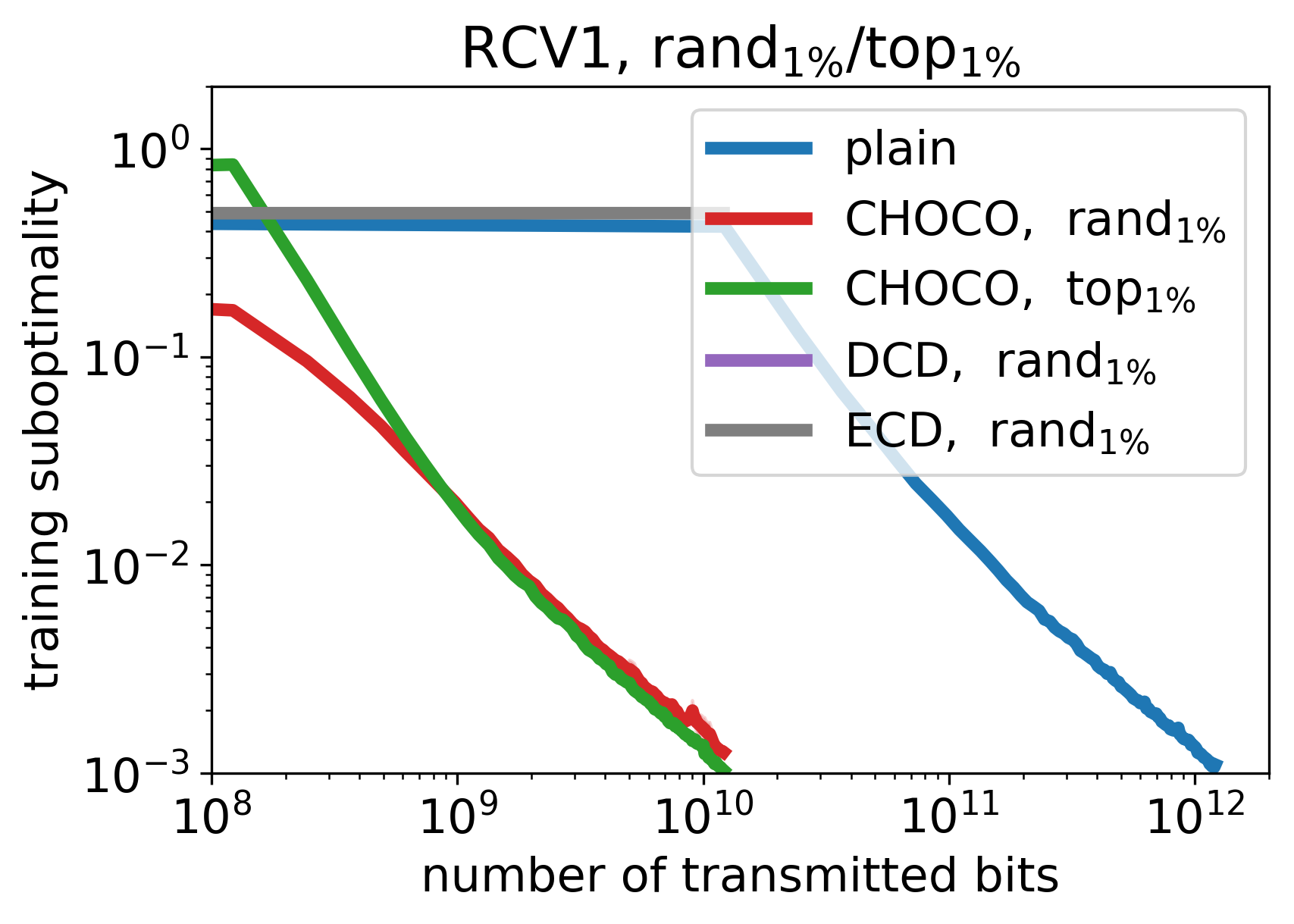}
		\caption{Comparison of Algorithm~\ref{alg:all-to-all} (plain), ECD-SGD, DCD-SGD and \algopt with ($\operatorname{rand}_{1\%}$) sparsification (in addition ($\operatorname{top}_{1\%}$) for \algopt), for $epsilon$ (top) and $rcv1$ (bottom) in terms of iterations (left) and communication cost (right). \emph{Randomly shuffled} data between workers.}
		\label{fig:sgd_random_20_random}
	\end{minipage}
\end{figure}
\begin{figure}
	\centering
	\begin{minipage}{0.8\textwidth}
		\setlength{\smallfigwidth}{0.48\linewidth}
		\includegraphics[width=\smallfigwidth]{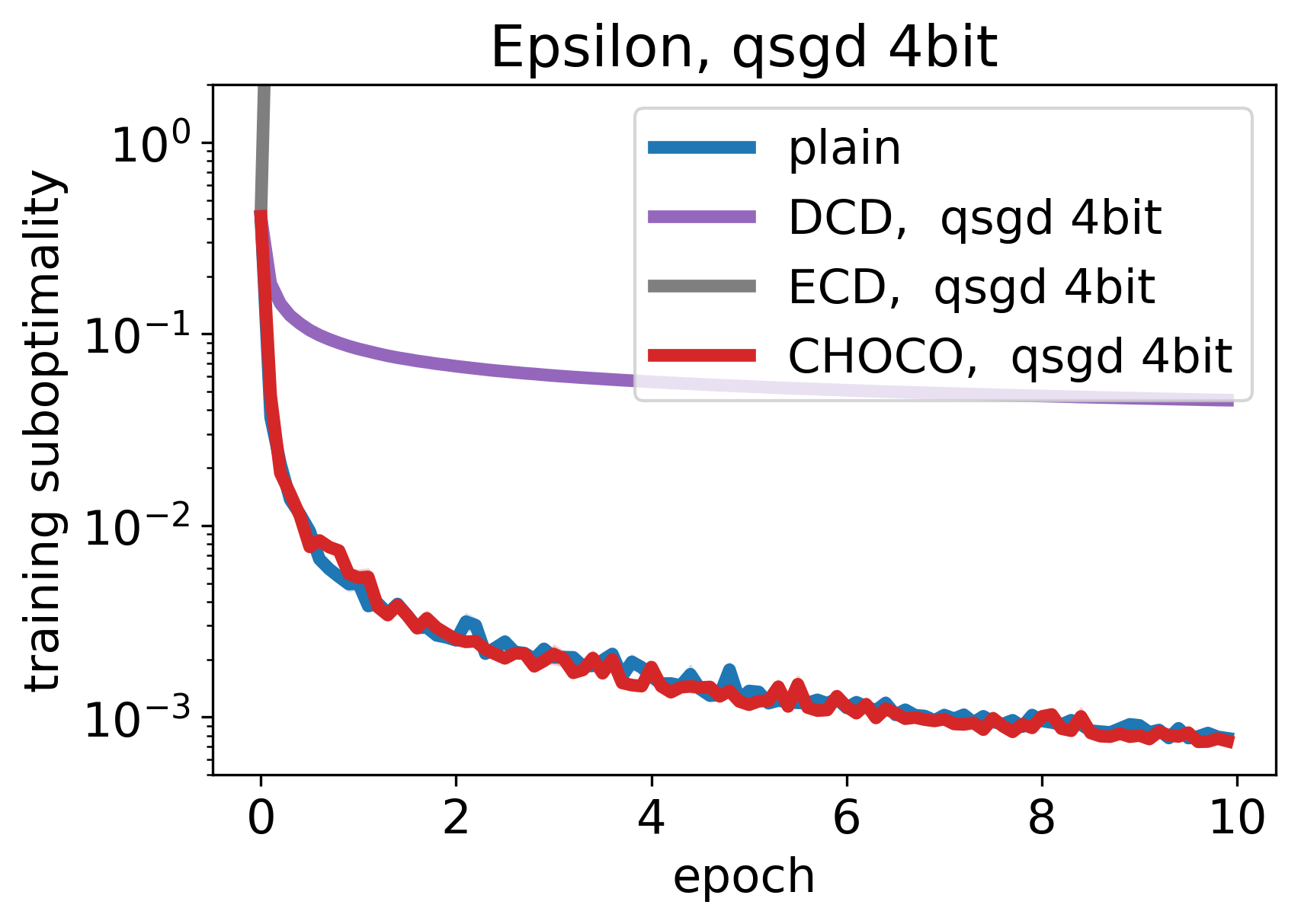}
		\hfill
		\includegraphics[width=\smallfigwidth]{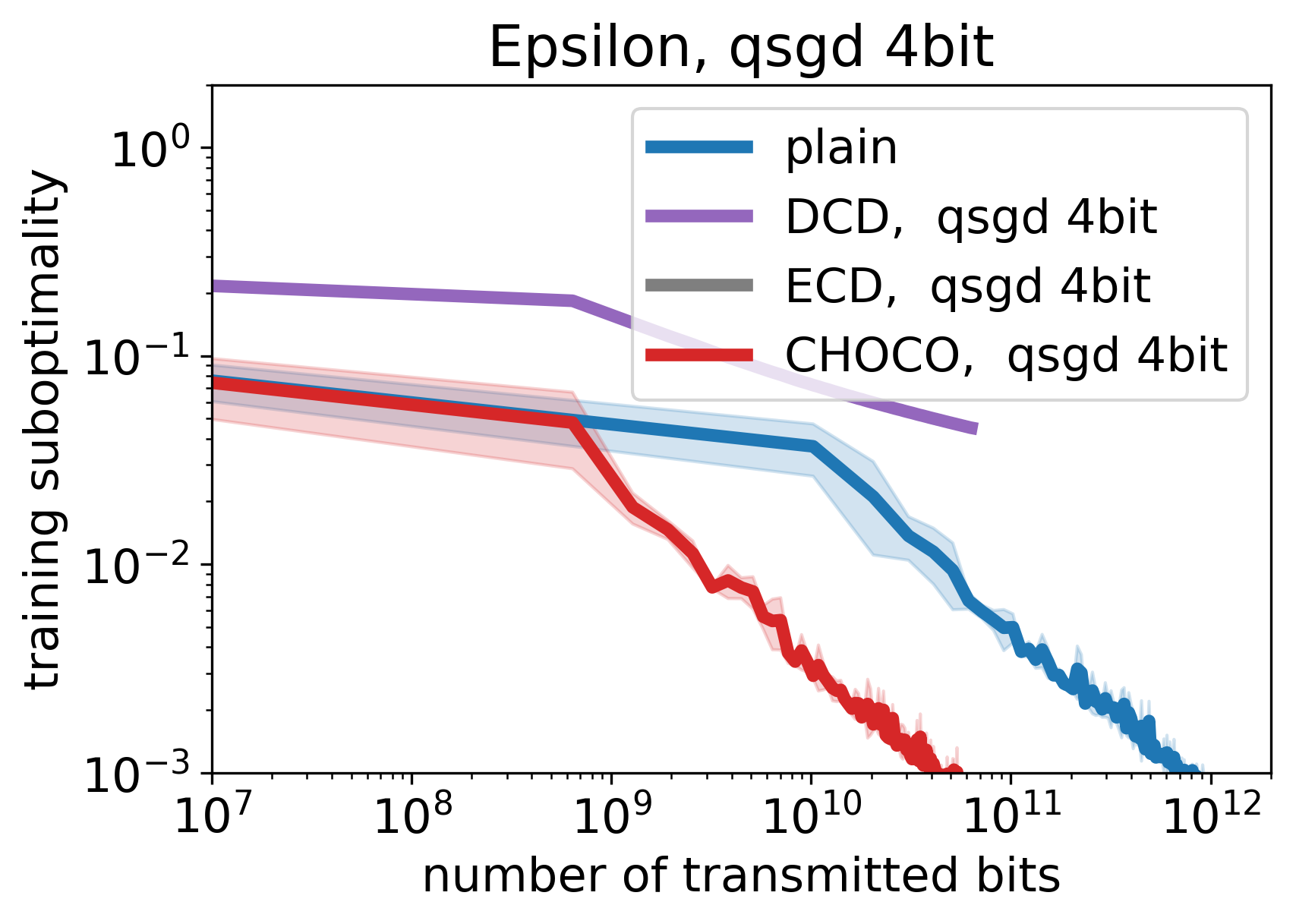}\\
		\includegraphics[width=\smallfigwidth]{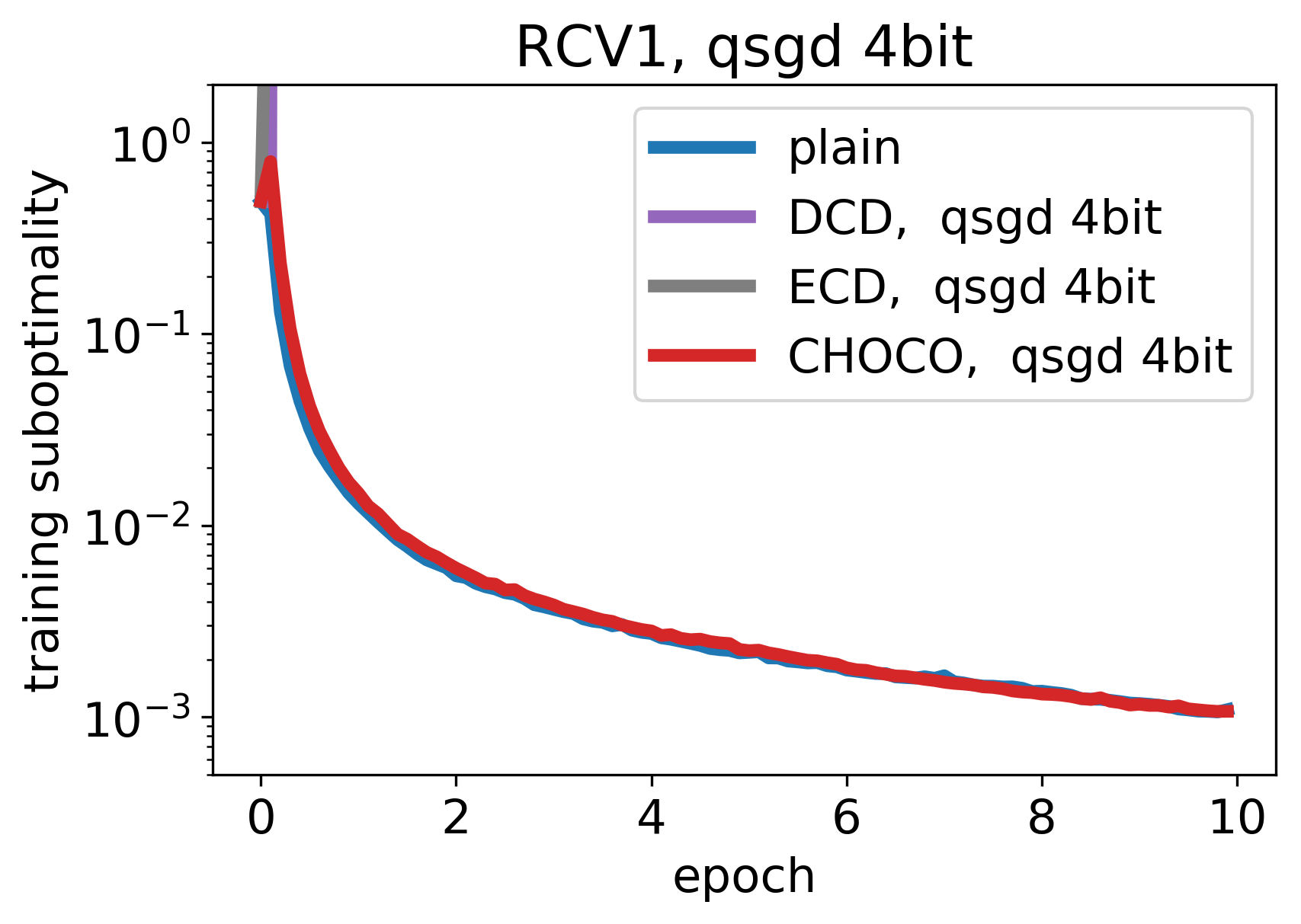}
		\hfill
		\includegraphics[width=\smallfigwidth]{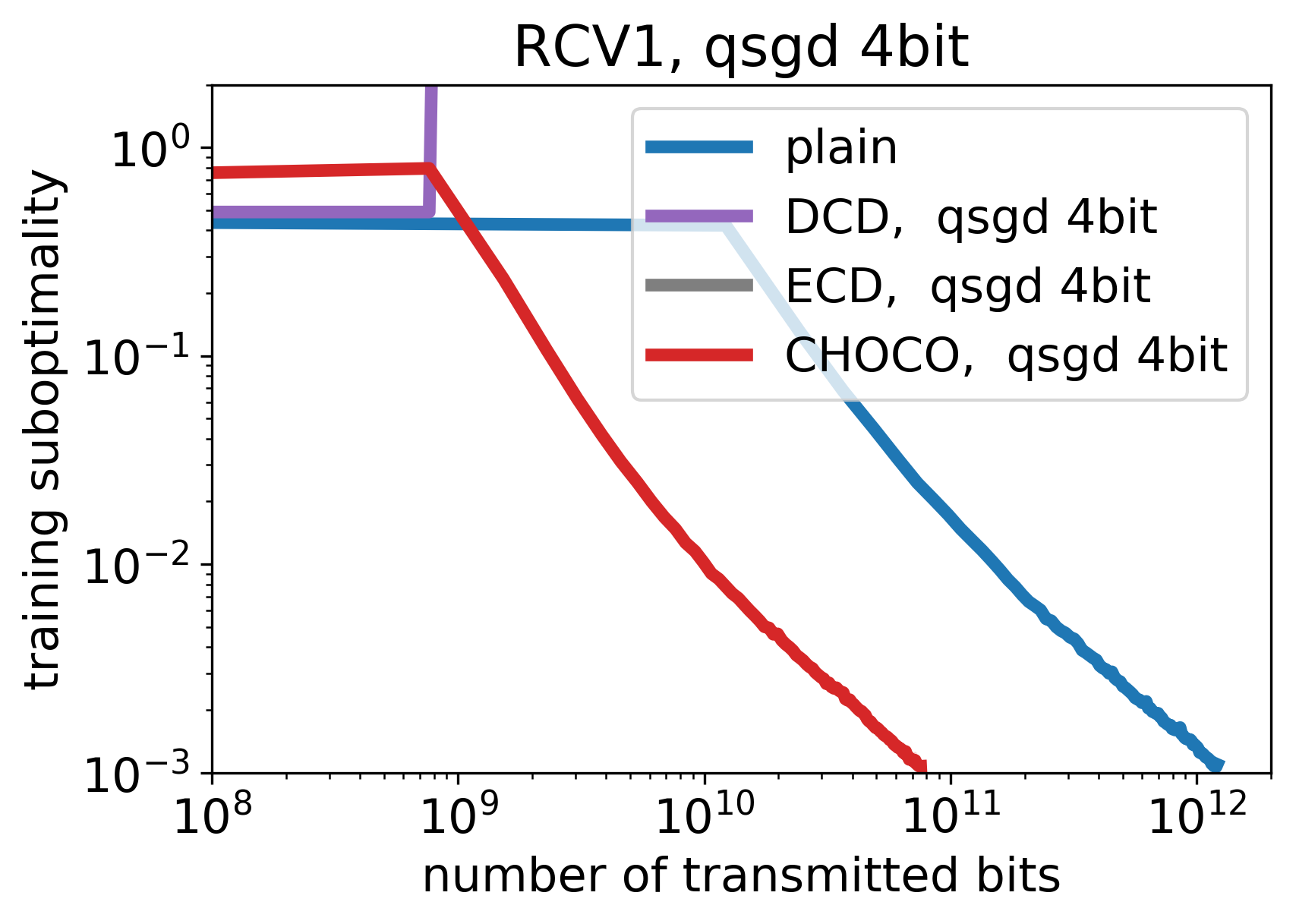}
		\caption{Comparison of Algorithm~\ref{alg:all-to-all} (plain), ECD-SGD, DCD-SGD and \algopt with ($\operatorname{qsgd}_{16}$) quantization, for $epsilon$ (top) and $rcv1$ (bottom) in terms of iterations (left) and communication cost (right). \emph{Randomly shuffled} data between workers.}\label{fig:sgd_qsgd_random}
	\end{minipage}
\end{figure}
\end{document}